\documentclass[runningheads]{llncs}
\usepackage[T1]{fontenc}

\usepackage{latexsym}
\usepackage{amssymb}
\usepackage{array}
\newcolumntype{C}[1]{>{\centering\arraybackslash}p{#1}}
\usepackage{adjustbox}
\usepackage{amsmath}
\usepackage{booktabs}
\usepackage{enumitem}
\usepackage{graphicx}
\usepackage{color}
\usepackage{xcolor,colortbl}
\usepackage{multirow}

\usepackage{tabularx}
\usepackage{hyperref}

\usepackage{algorithmicx}
\usepackage{algpseudocode}
\usepackage{algorithm}
\usepackage{subcaption,booktabs,makecell}
\usepackage{wrapfig}

\newcommand{\Var}{{\rm Var}\,}
\newcommand{\Exp}{{\rm I\hspace{-0.8mm}E}\,}
\newcommand{\Prob}{{\mathbb{P}}\,}
\newcommand{\calN}{{\mathcal{N}}\,}

\newcommand{\F}{{\mathcal{F}}}
\newcommand{\x}{{\mathbf{x}}}
\newcommand{\X}{{\mathbf{X}}}

\newcommand{\examplefirst}{{1}}
\newcommand{\examplesecond}{{A1}}
\newcommand{\examplemultit}{{2}}
\newcommand{\examplesdf}{{3}}
\newcommand{\examplereal}{{4}}
\newcommand{\examplelarged}{{5}}
\newcommand{\examplelargestd}{{6}}

\begin{document}
\title{Reducing Estimation Uncertainty Using Normalizing Flows and Stratification}
\author{Pawe{\l} Lorek\inst{1,6}\orcidID{0000-0003-2894-2799} \and
Rafał Nowak\inst{2,6} \and Rafał Topolnicki\inst{3,6} \and Tomasz Trzciński\inst{4,6,7,8} \and Maciej Zięba\inst{5,6} \and Aleksandra Krystecka\inst{1}}
\authorrunning{P. Lorek, R. Nowak, R. Topolnicki, T. Trzciński, M. Zięba, A. Krystecka}

%\titlerunning{FlowStrat: Reducing Estimation Variance with Flows and Stratification}
\titlerunning{Reducing Estimation Uncertainty Using Flows and Stratification}
\institute{University of Wrocław, Mathematical Institute
\and University of Wrocław, Institute of Computer Science
\and Dioscuri Center in Topological Data Analysis,\\ Institute of Mathematics,  Polish Academy of Sciences
\and Warsaw University of Technology
\and Wrocław University of Science and Technology
\and Tooploox
\and Jagiellonian University of Cracow
\and IDEAS NCBR}
\maketitle

\makeatletter
\setlength{\skip\footins}{2.5em}
\makeatother

\renewcommand{\thefootnote}{}
\footnotetext{\footnotesize
\hspace{-0.1cm}This  is the extended version of a paper accepted for publication at ACIIDS 2026.
}
\addtocounter{footnote}{-1}

\begin{abstract}
 Estimating the expectation of a real-valued function of a random variable from sample data is a critical aspect of statistical analysis, with far-reaching implications in various applications. Current methodologies typically assume (semi-)parametric distributions such as Gaussian or mixed Gaussian, leading to significant estimation uncertainty if these assumptions do not hold. We propose a flow-based model, integrated with stratified sampling, that leverages a parametrized neural network to offer greater flexibility in modeling unknown data distributions, thereby mitigating this limitation. Our model shows a marked reduction in estimation uncertainty across multiple datasets, including high-dimensional (30 and 128) ones, outperforming crude Monte Carlo estimators and Gaussian mixture models. Reproducible code is available at \url{https://github.com/rnoxy/flowstrat}.
\end{abstract}

\section{Introduction}
In modern machine learning, estimating the expectation of a real-valued function $f$ under a random variable $\X$,
\begin{equation}\label{eq:I_ExpfX}
I=\Exp f(\X),
\end{equation}
is fundamental. While $I$ is often intractable, independent samples from $\X$ are typically available. High dimensionality, multi-modality, or costly evaluations of $f$ can render sample averages inefficient, motivating variance reduction methods.

Here we consider the case where the distribution of $\X$ is unknown but only $n$ samples $\x_1,\ldots,\x_n$ are given. A naive estimator is the sample mean, $\hat{Y}^{\rm obs}={1\over n}\sum_{i=1}^n f(\x_i)$. Alternatively, one can estimate the distribution of $\X$ and resample, often assuming a parametric family (e.g., Gaussians or mixtures), which enables stratified sampling. However, such assumptions are restrictive when the true distribution is unknown.

The \textsl{main contribution of our paper} is a model for estimating $I=\Exp f(\X)$ when the distribution of $\X$ is unknown, relying solely on observed samples. The model uses stratified sampling (with proportional or optimal allocations) to improve accuracy, yielding more precise estimations of $I$. We propose two methods for stratified sampling of Gaussian distributions (cartesian and spherical) and present an effective approach for stratifying high-dimensional Gaussian distributions.
To address this, we employ flow-based models \cite{rezende2015variational}, which map a Gaussian base distribution to complex ones via invertible neural transformations. These models provide flexible distribution estimates, from which stratified sampling improves the accuracy of $I$’s estimation.

The main application arises when the number of observations is \textsl{insufficient} to estimate $I$ directly, but \textsl{sufficient} to train the flow model, enabling more accurate estimation through model-based sampling. Estimation uncertainty is further reduced via stratified sampling. For example, in estimating $I=\Prob(X_1>1.2, X_2>1.2)$ from synthetic data, direct observation-based estimation is poor, whereas flow-based stratified sampling (with optimal allocation and 16 strata) yields more accurate results with substantially narrower 95\% confidence intervals (see Fig.~\ref{fig:ExNO2d_cartesian_lines_short} and Table~\ref{tab:Example2d_NO_cart}, row $j_{1.2}^+$). For high-dimensional data, we propose two stratification methods: \textsl{cartesian}, applied to selected coordinates, and \textsl{spherical}, which stratifies only the radius and scales to high dimensions (applied in Examples \examplelarged{} and \examplelargestd{} with 30$d$ and 128$d$, respectively).

\begin{figure}[t]
\centering
\setlength{\tabcolsep}{4pt}

\begin{tabular}{c c c c c}
\raisebox{0.7cm}{\small $\hat{Y}_R^{\F,\rm CMC}$} &
\includegraphics[width=0.34\textwidth]{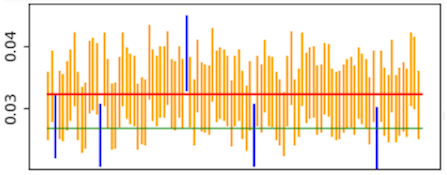} &
\hspace{-1mm}
&
\raisebox{0.7cm}{\small $\hat{Y}^{\F,\rm opt, M1}_{R,m=16}$} &
\includegraphics[width=0.34\textwidth]{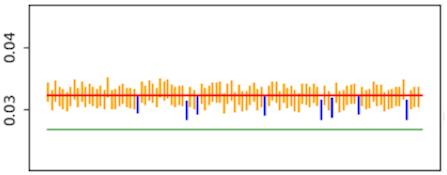}
\\[2mm]
\end{tabular}
\caption{Results for Example \examplefirst{}: 100   estimations of $I=\Prob(X_1>1.2, X_2>1.2)$, each (vertical line) resulted from $R=2^{12}$  simulations.
95$\%$ confidence intervals (vertical lines) depicted:
\textcolor{orange}{orange} lines: intervals containing true $I$ (\textcolor{red}{red} line), \textcolor{blue}{blue} lines: those not containing $I$.
A \textcolor{green}{green} line -- estimation of $I$ from observations.
Here $\F$ means that samples $\x_i$ were sampled from trained flow model, CMC stands for Crude Monte Carlo (i.e., a mean of $f(\x_i)$, no stratification) and ($\rm{opt}$,M1) denotes specific stratification.
}
\label{fig:ExNO2d_cartesian_lines_short}
\end{figure}

Normalizing flow models are known to outperform Gaussians and Gaussian mixture models (GMMs) when data come from a different distribution. This is illustrated in Example \examplefirst{}, where samples from this distribution are shown in Fig.~\ref{fig:Example2d_norm_flow_16str}. We compare GMM estimates with those from the trained flow model (both without stratification), as outlined in Table~\ref{tab:Example2d_NO_gmm}. Example \examplefirst{} further shows that relatively small sample sizes suffice to train the flow model accurately,   see Section \ref{sec:sample_perform}.

\section{Background}
\subsection{Normalizing Flows}

Generative models known as normalizing flows \cite{rezende2015variational,tabak2010density} can be effectively trained through direct likelihood estimation by utilizing the change-of-variable formula. Continuous Normalizing Flows (CNFs) \cite{grathwohl2018ffjord} propose to model transformation between data and base distributions with a dynamic system. The objective of CNFs is to solve a differential equation of the form $\frac{d\mathbf{z}}{dt}=\mathbf{g}_{\boldsymbol{\beta}}(\mathbf{z}(t), t)$, where $\mathbf{g}_{\boldsymbol{\beta}}(\mathbf{z}(t), t)$ represents the dynamics function described by parameters $\boldsymbol{\beta}$. The goal is to find a solution to the equation at time $t_1$, $\mathbf{y}:=\mathbf{z}(t_1)$, given an initial state $\mathbf{z}:=\mathbf{z}(t_0)$ with a known prior. The transformation function $\mathbf{h}_{\boldsymbol{\beta}}$
and the log-density of $\mathbf{x}$ are as follows:
% \begin{equation}\label{eq:H_beta}
% \mathbf{x} = \mathbf{h}_{\boldsymbol{\beta}}( \mathbf{z} ) =  \mathbf{z} + \int^{t_1}_{t_0} \mathbf{g}_{\boldsymbol{\beta}}(\mathbf{z}(t), t) \mathrm{dt},
% \quad \log p_{\boldsymbol{\beta}}(\mathbf{x}) = \log \calN(\mathbf{z};\mathbf{0},\mathbf{I})
%     - \int^{t_1}_{t_0} \frac{\mathrm{d} \mathbf{h}_{\boldsymbol{\beta}}(\mathbf{z}(t),t)}{\mathrm{d}\mathbf{z}(t)}\mathrm{dt}.
% \end{equation}
\begin{eqnarray}
\mathbf{x} & =&  \mathbf{h}_{\boldsymbol{\beta}}( \mathbf{z} ) =  \mathbf{z} + \int^{t_1}_{t_0} \mathbf{g}_{\boldsymbol{\beta}}(\mathbf{z}(t), t) \mathrm{dt}, \nonumber  \\
\log p_{\boldsymbol{\beta}}(\mathbf{x}) &=& \log \calN(\mathbf{z};\mathbf{0},\mathbf{I})
    - \int^{t_1}_{t_0} \frac{\mathrm{d} \mathbf{h}_{\boldsymbol{\beta}}(\mathbf{z}(t),t)}{\mathrm{d}\mathbf{z}(t)}\mathrm{dt}.\label{eq:H_beta}
\end{eqnarray}
The decision to use CNFs is driven by successful applications in models such as NGGP \cite{sendera2021non}, PointFlow \cite{yang2019pointflow}, and StyleFlow \cite{abdal2021styleflow}, which are characterized by similar dimensionalities.

\subsection{Monte Carlo estimation}
Here, we assume we aim to estimate the expectation of a random variable $Y\in\mathbb{R}$, namely
$I=\Exp Y,$
given that we can sample $Y$ and know that $\Exp Y^2<\infty$.
(Later we consider $Y=f(\X), f:\mathbb{R}^d\to\mathbb{R}$).
Suppose we have $R\geq 1$ samples $Y_1,\ldots,Y_R$, and $\hat{Y}_R$ is an unbiased estimator of $I$
(i.e., $\Exp\hat{Y}_R=I$) based on these samples.
Then, the central limit theorem (CLT) states that for any $\alpha\in(0,1)$,
we have
(in our applications $R$ is large, making the CLT approximation accurate)
\begin{equation}\label{eq:gen_conf_int}
 \Prob\left(I\in\left[\hat{Y}_R \pm z_{1-\alpha/2} \sqrt{\Var(\hat{Y}_R)}\right]\right)
\approx 1-\alpha.
\end{equation}
Thus, the smaller the estimator’s variance, the narrower the above confidence interval.
\paragraph{Crude Monte Carlo estimator.} The simplest estimator, usually called
\textsl{Crude Monte Carlo estimator}, is a~sample mean of iid $Y_1,\ldots,Y_R$,
namely
$\hat{Y}_R^{\rm CMC}=\sum_{j=1}^R Y_j/R.$
In such a case, we have $\Var(\hat{Y}_R^{\rm CMC})=\Var Y/R$.
In practice we often do not know $\Var Y$, we may however replace it with
a sample variance $\hat{S}_R^2=\sum_{j=1}^R (Y_j-\hat{Y}_R^{\rm CMC})^2/(R-1)$,
since a version  of CLT states  that $\sum_{j=1}^R(Y_j-I)/(\hat{S}_R\sqrt{R})$
converges to a standard normal $\mathcal{N}(0,1)$ as $R\to\infty$.
In other words, we have the confidence interval given by Eq.
\eqref{eq:gen_conf_int} with $\Var(\hat{Y}_R)$ replaced
with $\hat{S}^2_R/R$.
\subsection{Stratified sampling}\label{sec:strat_est}
Fix $m\geq 1$ (number of strata) and let $A^1,\ldots,A^m$ be disjoint sets such that
$\Prob\left(Y\in\cup_{j=1}^m A^j\right)=1$. This split defines \textsl{strata}
$S^j=\{\omega: Y(\omega)\in A^j\}$. Slightly abusing the notation, we will call $A^j$ also a stratum.
We assume that $m$ and strata $A^j,j=1,\ldots,m$ are fixed.
Let $p_j=\Prob(Y\in A^j)$ be the probability of $j$-th stratum and let
$I^j=\Exp[Y|Y\in A^j]$ be the expectation of $Y$ on $j$-th stratum.
The law of total expectation  yields
\begin{equation}\label{eq:I_total_law}
 I=\Exp Y = p_1 I^1+\ldots+p_m I^m.
\end{equation}
The idea of stratified sampling can be thus explained as follows. Split the total budget of
simulations (\textsl{general split}) $R=R_1+\ldots +R_m$, then estimate $I^j$ using $R_j$ simulations and finally
estimate $I$ computing the weighted average (\ref{eq:I_total_law}).
Note that to apply the procedure, we need to be able to sample
$Y$ from $j$-th stratum, i.e., $Y^j\stackrel{\mathcal{D}}{=}(Y|Y\in A^j)$ (equality in distribution).
Let $Y^j_1,\ldots, Y_{R_j}^j$ be iid replications of $Y^j$. We estimate
$I^j$ via $\hat{Y}_{R_j}^j={1\over R_j}\sum_{i=1}^{R_j} Y_i^j$ (which is
a CMC estimator of $I^j$). Finally, denoting, $\sigma_j^2=\Var Y^j$, the \textsl{stratified estimator} and
its variance are:
\begin{equation*}\label{eq:estim_str}
 \textstyle \hat{Y}_R^{\rm str}=p_1\hat{Y}_{R_1}^1+\ldots+p_m\hat{Y}_{R_m}^m,
 \Var(\hat{Y}_R^{\rm str})=\sum_{j=1}^m {p_j^2\over R_j}\sigma^2_j.
\end{equation*}
Again, in practice, we do not know the variances $\sigma^2_j$,
we estimate them via sample variances $\hat{s}^2_j$ and we estimate the variance
of $\hat{Y}_R^{\rm str}$ via
 $ \widehat{\Var}(\hat{Y}_R^{\rm str})=\sum_{j=1}^m {p_j^2\over R_j}\hat{s}^2_j.$

\paragraph{Proportional allocation.} We split $R$ proportionally to
probabilities of strata,
 i.e., $R_j=Rp_j$.
Denote the corresponding estimator as $\hat{Y}_R^{\rm pa}$.
Then we have
$ \Var(\hat{Y}_R^{\rm pa})={1\over R}\sum_{j=1}^m p_j \sigma_j^2$ and $
 \widehat{\Var}(\hat{Y}_R^{\rm pa})={1\over R}\sum_{j=1}^m p_j \hat{s}_j^2.$
 A variance of CMC estimator may be decomposed
 as $\Var(\hat{Y}_R^{\rm CMC})= \Var(\hat{Y}_R^{\rm pa})+\sum_{j=1}^m p_j(I^j-I)^2$,
 thus   $\Var(\hat{Y}_R^{\rm pa})\leq \Var(\hat{Y}_R^{\rm CMC}).$
 
 \paragraph{Optimal allocation.}
 It turns out that the  split:
$\displaystyle R_j={p_j\sigma_j\over \sum_{i=1}^m p_i \sigma_i} R, \quad j=1,\ldots,m,$
 is best possible, in the sense that  $\displaystyle \Var(\hat{Y}^{\rm opt}_R)\leq \Var(\hat{Y}^{\rm str}_R)$,
 where $\hat{Y}^{\rm opt}_R$ is the stratified estimator with above split and $\hat{Y}^{\rm str}_R$ is a general stratified estimator.
 The proof of this fact (formally in Theorem \ref{thm:optimal} in appendix) an the decomposition of
  CMC variance  may be found
 e.g., in \cite{madras} (Theorem 3.3), in Appendix \ref{sec:more_on_str}  we provide
  different proofs.

\section{Methodology of our method}

\subsection{Stratifying multivariate Gaussian}
The two methods described below are presented for equally probable
strata (i.e., $p_i=1/m$), as we will need only this form (it generalizes easily).
\subsubsection{Method M1: Cartesian stratification}\label{sec:method1}
Let us start with a univariate standard normal $Z\sim\calN(0,1)$, denoting its
distribution function by $\Phi$.
To perform stratified sampling with $m_0$ strata, we split $\mathbb{R}$ into $m_0$ intervals:
\begin{equation}\label{eq:strata_A}
 A^1=(a_0,a_1],  \ldots, A^{m_0}=(a_{m_0-1},a_{m_0}],
\end{equation}
where $a_0=-\infty$, $a_{m_0}=+\infty$, and $a_j$ are chosen so that $\Prob(Z\in(a_{j-1},a_j])=p_j$,
achieved by setting
$a_1=\Phi^{-1}(p_1), \quad a_2=\Phi^{-1} (p_1+p_2),\ldots$
(inverses $\Phi^{-1}(\cdot)$ are evaluated numerically).
To sample $Z^j$ from $A^j$, we take $U\sim\mathcal{U}(0,1)$ and set
$Y^j=\Phi^{-1}(V^j)$, where $V^j=a_{j-1}+(a_j-a_{j-1})U.$
Now, let $(Z_1,\ldots,Z_d)$ be a standard normal $d$-dimensional random variable.
We split each $Z_i$ into $m_0$ strata as above, obtaining $m=m_0^d$ strata. If $p_j=1/m_0$ for each $Z_i$,
we obtain $m_0^d$ equally probable strata.

\begin{figure}[t]
\centering

% -------------------- LEFT SUBFIGURE --------------------
\begin{subfigure}{0.48\textwidth}
\centering
\includegraphics[width=0.48\textwidth]{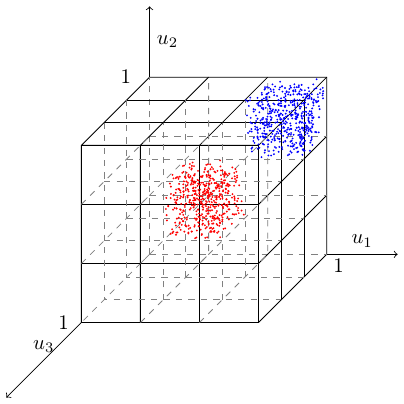}
\includegraphics[width=0.48\textwidth]{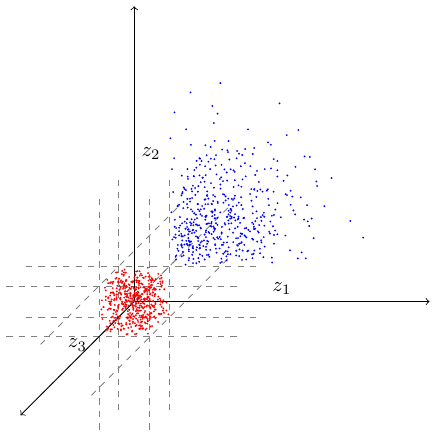}

\caption{\textsl{Cartesian stratification}. Stratified uniform $(U_1,U_2,U_3)$
(each into $m_0=3$ strata) and corresponding strata for $(Z_1,Z_2,Z_3)$.
In total $m=3^3=27$ strata.}
\label{fig:method1_2a}
\end{subfigure}
\hfill
% -------------------- RIGHT SUBFIGURE --------------------
\begin{subfigure}{0.48\textwidth}
\centering
\includegraphics[width=0.75\textwidth]{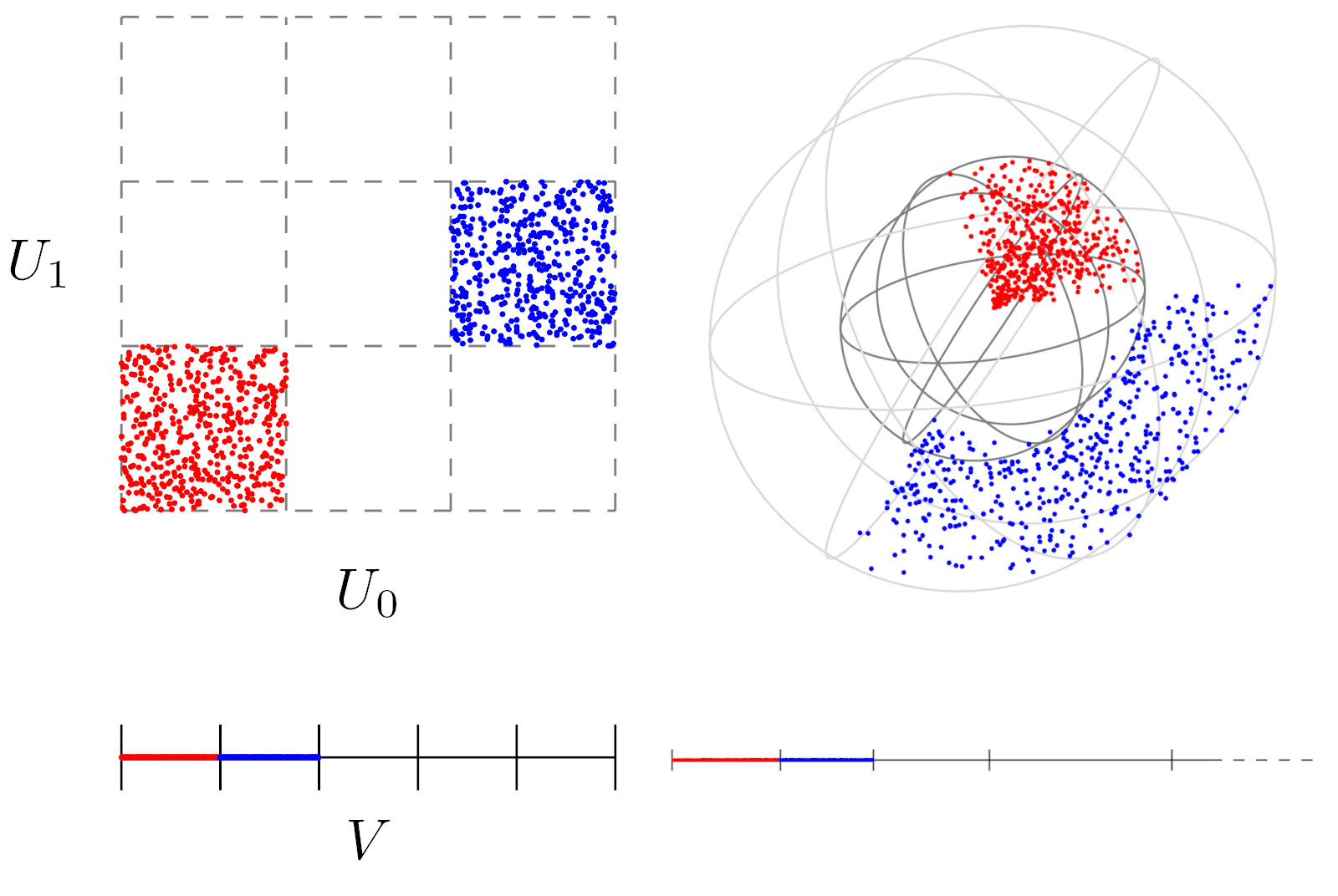}

\caption{\textsl{Spherical stratification}.
Radius is stratified into $m_r=5$ strata (two first strata shown on the left)
and each angle into $m_0=3$ strata; two strata distinguished.}
\label{fig:method1_2b}
\end{subfigure}

\caption{Comparison of Cartesian and Spherical stratifications.}
\label{fig:method1_combined}
\end{figure}
Consider the   example: to sample a 3D standard normal point $(Z_1,Z_2,Z_3)$, we can sample $(U_1,U_2,U_3)$
uniformly from the hypercube $[0,1]^3$ and compute $Z_i=\Phi^{-1}(U_i), i=1,2,3$. Now, if we
sample $(U_1,U_2,U_3)$ only from one of nine subhypercubes (of equal volume), the resulting $(Z_1,Z_2,Z_3)$ will be a normally distributed 3D point
sampled from the corresponding stratum. In Fig.~\ref{fig:method1_2a},
two such strata are illustrated.

\subsubsection{Method M2: Spherical stratification}\label{sec:method2}
\textbf{Stratification with respect to radius.}
The random variable
$D^2=Z_1^2+\ldots+Z_d^2$
follows a $\chi^2_d$ distribution (chi-squared with $d$ degrees of freedom) with cdf $F_{\chi^2_d}$.
Thus, if we sample $d$ iid standard normal $\calN(0,1)$ variables,
collect them as $\mathbf{Z}=(Z_1,\ldots,Z_d)$,
normalize them as
$ \mathbf{Z}'=\left( {Z_1\over ||\mathbf{Z}||},\ldots , {Z_d\over ||\mathbf{Z}||}\right), $
and then set $\mathbf{Z}''=(D Z_1',D Z_2',\ldots,D Z_d')$, where $D^2$ is sampled from a $\chi^2_d$ distribution, then $\mathbf{Z}''$ has the same distribution as $\mathbf{Z}$.
This leads to the following stratification. Fix the number of strata $m_r$, sample standard normal variables $Z_1,\ldots,Z_d$, and normalize them (obtaining $Z_1',\ldots,Z'_d$).
Now, stratify only the radius $D$ by sampling it from the $j$-th stratum, which is done by sampling $U\sim\mathcal{U}(0,1)$ and setting
$ (D^j)^2=F^{-1}_{\chi^2_d}\left({j\over m_r} + {1\over m_r} U\right), $
to finally obtain $\displaystyle \mathbf{Z}^j=(D^j Z_1',\ldots, D^j Z_d')$.
\smallskip\par
\textbf{Stratification w.r.t. angles.}
Multivariate standard normal $\mathbf{Z}=(Z_1,\ldots,Z_d)$ normalized
(i.e., $\mathbf{Z}'=\mathbf{Z}/\|\mathbf{Z}\|$) has a uniform distribution on the hypersphere
\[
S_{d-1}=\left\{\mathbf{x}=(x_1,\ldots,x_d): \sum_{j=1}^d x_j^2=1\right\}.
\]
A point $\mathbf{x}\in S_{d-1}$ can be represented in spherical coordinates as
\begin{eqnarray}
 x_1 &=& \cos \phi_1,\quad  x_2 = \sin \phi_1 \cos \phi_2,\ \ldots, \nonumber\\
 x_d &=& \sin \phi_1 \cdots \sin \phi_{d-3} \sin \phi_{d-2} \sin \theta,
 \label{eq:spher_coord_old}
\end{eqnarray}
where $\phi_1,\ldots,\phi_{d-2}\in[0,\pi)$ and $\theta\in[0,2\pi)$.
The density of the uniform distribution on $S_{d-1}$ (in these coordinates) factorizes as
\[
g(\theta,\phi_1,\ldots,\phi_{d-2}) \propto
 \sin^{d-2}\phi_1\, \sin^{d-3}\phi_2 \cdots \sin\phi_{d-2},
\]
which can be written as a product of one-dimensional densities
\[
g(\theta,\phi_1,\ldots,\phi_{d-2})
= h_0(\theta)\, h_1(\phi_1)\cdots h_{d-2}(\phi_{d-2}),
\]
where $h_0$ is the density of the uniform distribution on $(0,2\pi)$ and
\[
h_k(\phi)=\frac{1}{c_k}\sin^k(\phi),\qquad \phi\in(0,\pi),\quad k=1,\ldots,d-2,
\]
with normalization constants $c_k=\int_0^\pi \sin^k(x)\,dx$.
Consequently, sampling uniformly from $S_{d-1}$ reduces to independent sampling of
$\theta\sim h_0$ and $\phi_k\sim h_k$ for $k=1,\ldots,d-2$.

To obtain stratified sampling on the hypersphere, we stratify each angular coordinate.
For $\theta$ we split $(0,2\pi)$ into $m_0$ equal-length intervals.
For each $\phi_k$ we construct $m_0$ strata $(a_{k,j-1},a_{k,j}]$ such that
\[
\int_{a_{k,j-1}}^{a_{k,j}} h_k(\phi)\,d\phi = \frac{1}{m_0},\qquad j=1,\ldots,m_0,
\]
which is done numerically. Sampling from a given stratum then consists in
drawing $\theta$ and $\phi_k$ from the corresponding one-dimensional strata and
mapping $(\theta,\phi_1,\ldots,\phi_{d-2})$ back to a point on $S_{d-1}$ via
(\ref{eq:spher_coord_old}). Combining this with stratification of the radius
(as in the previous paragraph) yields $m=m_r m_0^{d-2}$ equally probable strata
for the Gaussian base distribution.

The one-dimensional densities $h_k(\phi)\propto\sin^k(\phi)$ do not admit
closed-form inverse distribution functions for general $k$.
In practice, we therefore generate samples from $h_k$ using a simple
acceptance--rejection scheme with a uniform proposal on $(0,\pi)$ and extend
it to stratified sampling by restricting the proposal to a chosen subinterval.
The expected number of accept--reject iterations remains moderate even for
large $k$, which makes the method practical in our settings.
Full derivations, the explicit form of the acceptance--rejection sampler
and construction of the boundaries $a_{k,j}$ are provided
in Appendix \ref{app:sampl_hk}.

\smallskip\par
\textbf{Full stratification.}
We simply independently stratify radius $D$ (using $m_r$ strata)
and angles $\theta, \phi_1,\ldots,\phi_{d-2}$ (using $m_0$ strata for each angle) as described above, thus finally we have $m=m_r m_0^{d-2}$ equally probable strata.
The case $d=3$ with $m_r=5$ and $m_0=3$ is depicted in Fig. \ref{fig:method1_2b}.

\subsection{Optimal allocation in practice}
First, we run $R'$ pilot simulations (usually much smaller than target $R$ simulations) with proportional allocation just to estimate
 variances $\hat{s}'^2_j$ in each strata. They serve to compute the final optimal split
 $ R_j={p_j \hat{s}'^2_j\over \sum_{k=1}^m p_k \hat{s}'^2_k} R.$
 Then we perform $R_j$ simulations $Y^j_1,\ldots,Y_{R_j}^j$ in  each stratum $j=1,\ldots,m$,
 estimate $I^j$ via $\hat{Y}^j_{R_j}$ and compute the final estimator
 $\hat{Y}_R^{\rm opt}=p_1 \hat{Y}^1_{R_1}+\ldots+p_m \hat{Y}^m_{R_m}.$
 \textsl{Afterwards} we recompute variances $\hat{s}^2_j$ within each stratum and estimate
 $\Var \hat{Y}^{\rm opt}_R$ via
 $\widehat{\Var} (Y_R^{\rm opt})=\left(\sum_{j=1}^m p_j\hat{s}_j\right)^2.$
 \subsection{High-dimensional stratification}

When $d$ is large, using Method M1 (cartesian) or Method M2 (spherical)
with $m_0>1$ strata in every dimension or angle leads to an exponential number
of strata ($m_0^d$ or $m_r m_0^{d-2}$), which is prohibitive in our
$30$- and $128$-dimensional examples. We therefore use two approximations.
\smallskip\par
\textsl{Method} $\textrm{M}\texttt{rad}$.
This method is a special case of spherical stratification. We stratify only the
radius $D$ into $m_r$ equally probable shells and do not stratify the angles
($m_0=1$). In other words, $\mathbb{R}^d$ is split into $m=m_r$ radial strata
$A^1,\ldots,A^m$, and we sample
$\mathbf{Z}^j \stackrel{\mathcal{D}}{=} (\mathbf{Z}\mid \mathbf{Z}\in A^j)$
by first sampling a direction on the sphere and then sampling the radius from
the $j$-th $\chi^2_d$-stratum as in Section~\ref{sec:method2}.
The resulting estimators are denoted
$\hat{Y}_{R,m}^{\F,\rm prop, M\texttt{rad}}$ and
$\hat{Y}_{R,m}^{\F,\rm opt, M\texttt{rad}}$ for proportional and optimal
allocation, respectively.
\par\smallskip
\textsl{Methods} $\textrm{M}\texttt{High3}$ \textsl{and} $\textrm{M}\texttt{Rand3}$.
Here we stratify only $\eta=3$ selected coordinates into $m_0$ strata each,
which yields $m=m_0^3$ equally probable strata. For a given triple
$(i,j,k)$ we use the one-dimensional normal stratification
from Section~\ref{sec:method1} on coordinates $Z_i,Z_j,Z_k$ and keep the remaining
coordinates unstratified. Sampling from a given stratum is then straightforward:
we sample $Z_t\sim\calN(0,1)$ for $t\notin\{i,j,k\}$ and sample $Z_i,Z_j,Z_k$
from the corresponding one-dimensional intervals.

In Method $\textrm{M}\texttt{Rand3}$ the indices $i,j,k$ are chosen uniformly
at random from $\{1,\ldots,d\}$ (independently in each repetition).
In Method $\textrm{M}\texttt{High3}$ the choice depends on the target function
$f$ to be estimated. We first run $R_0<R$ pilot simulations and, for each
coordinate $b$, construct a stratified estimator that splits only $Z_b$ into
$m_0$ strata. We then compute the empirical standard deviations of these
$d$ estimators and select the three coordinates with the largest variances
as $i,j,k$. The corresponding estimator is denoted
$\hat{Y}^{\F,\rm prop, M\texttt{High3}}_{R,m=m_0^3}$.
Further implementation details are provided in Appendix \ref{sec:app_more_high}.

\subsection{Stratified flow-based estimation}
Given $n$ observations $\mathbf{x}_1,\ldots, \mathbf{x}_n$ from $\mathbb{R}^d$,
we first approximate their distribution $p(\x)$ by $p_{\boldsymbol{\beta}}(\x)$
by training a flow model $\mathcal{F}_{\boldsymbol{\beta}}$ to minimize the negative log-likelihood
$
 \mathcal{L}=- \sum_{k=1}^n \log p_{\boldsymbol{\beta}}(\mathbf{x}_k),
$
with respect to the parameters $\boldsymbol{\beta}$,
where $\log p_{\boldsymbol{\beta}}(\cdot)$ is given by Eq.~\eqref{eq:H_beta}. This loss is optimized using a standard, gradient-based training procedure, as described in \cite{grathwohl2018ffjord}.
Once the flow model is trained, we can sample any number of points from $p_{\boldsymbol{\beta}}(\mathbf{x})$, which approximates the distribution $p(\mathbf{x})$ of the random variable $\X$. Direct stratification in the data space is challenging, so we propose stratifying in the latent space of the flow model, with the Gaussian base distribution $p(\mathbf{z})=\calN(\mathbf{z};\mathbf{0},\mathbf{I})$. We then apply either Method 1 (cartesian) or Method 2 (spherical) to split the Gaussian latent space and generate the desired number of points (proportional or optimal) in each region. The points generated in each region are then transformed to the data space using the transformation $ \mathbf{h}_{\boldsymbol{\beta}}( \mathbf{z} )$ defined by Eq.~\eqref{eq:H_beta}.
Afterward, we apply $f$ to the generated points and estimate $I=\Exp f(\X)$. Note that if
we need to estimate multiple quantities $I_k=\Exp f_k(\X), k=1,\ldots,M$, the flow model requires only a single training.

\section{Experimental results}\label{sec:exp_res}
For a random variable $\X$, we want to estimate $I=\Exp f(\X)$ based on observations $\x_1,\ldots,\x_n$.
An exact value of $I$ is known for all synthetic examples (in Example \examplesdf{}, an "exact" value is computed by sampling a large number of points).
Recall,  $\hat{Y}_n^{\rm obs}=\sum_{i=1}^n f(\x_i)/n$ estimates $I$ from observations.
After training a flow model $\F_{\boldsymbol{\beta}}$, we construct the following estimators of $I$:
i)~$\hat{Y}^{\F, \rm CMC}_R={1\over R}\sum_{i=1}^R f(\x_i^\F)$ -- we simply sample $R$ points $\x_1^\F,\ldots,\x_R^\F$ from $\F_{\boldsymbol{\beta}}$ (via independent sampling of base distribution rvs) and compute means of $f(\x_i^\F)$;
ii)~$\hat{Y}^{\F, \rm  prop, M}_{R,m=k}$ --
we sample $R$ points using stratification with proportional sampling and stratification
method $\text{M}\in\{$M1,M2,$\textrm{M}\texttt{rad}$,$\textrm{M}\texttt{High3}$,$\textrm{M}\texttt{Rand3}\}$ and $k$ strata;
iii)~$\hat{Y}^{\F, \rm opt}_{R,m=k}$ -- similarly, we use
stratification with optimal allocation (where we also used $R'$ pilot simulations to
estimate stratum variances). For optimal allocation, unless stated otherwise,  we  use  $R'=R/8$ pilot simulations to estimate
standard deviations.

We compute:
\texttt{E} -- the estimator value, an approximation of $I$;
\texttt{SD} -- the standard deviation of the estimator; and $\displaystyle \texttt{AC}=-\log_{10}\left|({I-\texttt{E}) / I}\right|$, the estimator’s accuracy (roughly speaking, the number of correct digits in $\texttt{E}$).
In tables, we report $I$, \texttt{E}, \texttt{SD}, \texttt{AC},
or (for readability)
$I^*=100I$, \texttt{SD${}^*$}=$100$\texttt{SD}, and \texttt{E}$^*$=$100$\texttt{E}.

All metric values are computed as
\textbf{means of 10 simulations.}
Note that the average of 10  \texttt{AC} values is \textbf{not} equal to \texttt{AC} computed from the average of 10 \texttt{E} values.
Metrics are presented in tables rounded to three or four digits, though they are computed with higher precision.
Data in tables allow for constructing a confidence interval at level
$\alpha$, namely $\Prob(\hat{Y}-z_{1-\alpha/2} \texttt{SD}\leq I \leq \hat{Y}+z_{1-\alpha/2} \texttt{SD})\approx 1-\alpha$, Fig.~\ref{fig:ExNO2d_cartesian_lines_short}
presents confidence intervals for Example \examplefirst{}.

In examples with known $I$, in tables, the best accuracy in each row is \textbf{bolded},
second best is \underline{underlined}, unless otherwise stated in a caption.
In examples with unknown $I$, neither $I$ nor \texttt{AC} are reported, then the best (smallest)
\texttt{SD} is \textbf{bolded}, and the second best (second smallest) \texttt{SD} is \underline{underlined}.
We consider different functions $f$ across examples, with the following
functions often used:
\begin{equation}\label{eq:fun_f1}
{\small
j^{+}_t(\x) = \mathbf{1}(\forall i:~x_i > t),~~ j^{-}_t(\x) = \mathbf{1}(\forall i:~ x_i \leq t),
}
\end{equation}
which estimate probabilities $\Prob(X_i>t, i=1,\ldots,d)$ and $\Prob(X_i\leq t, i=1,\ldots,d)$.

We demonstrate examples with different parameters $n, m, R$ and functions $f$ to show
the robustness of the method,   yielding superior results over observation-based estimations.

\begin{wraptable}{r}{0.6\textwidth}
\caption{(Example   \examplefirst{}) Results for \textbf{cartesian} method. }
\label{tab:Example2d_NO_cart}
\footnotesize
%\begin{tabular}%
%{|p{.35cm}|p{0.6cm}|p{.45cm}p{.6cm}|p{.4cm}|p{.5cm}p{.55cm}|p{.5cm}p{.55cm}|}
{\fontsize{8pt}{9.2pt}\selectfont
\setlength{\tabcolsep}{1.9pt}
\begin{tabular}{|p{0.04\textwidth}|p{0.05\textwidth}|p{0.05\textwidth}p{0.05\textwidth}|p{0.05\textwidth}|p{0.055\textwidth}p{0.055\textwidth}|p{0.055\textwidth}p{0.055\textwidth}|}
\hline
\multicolumn{1}{|c|}{$f$} & \multicolumn{1}{c|}{$I^*$} & \multicolumn{2}{c|}{ \small{$\hat{Y}_{1000}^{\rm obs}$}} & $R$
    & \multicolumn{2}{c|}{\small{$\hat{Y}^{\F, \rm CMC}_R$}}
    & \multicolumn{2}{c|}{\small{$\hat{Y}^{\F, \rm  opt, M1}_{R,m=4\times4}$}}
\\\hline
 &   & \texttt{E*} & \texttt{AC}
     &
     & \texttt{E*} & \texttt{AC}
     & \texttt{E*} & \texttt{AC}
\\\hline
           \multirow{2}{*}{$j^{+}_{1.2}$} & \multirow{2}{*}{3.24} & \multirow{2}{*}{2.70}  & \multirow{2}{*}{.776}
           &   $2^{12}$
            & 3.14  & \underline{1.25} % CMC
                & 3.24  & \textbf{1.82} % Opt m 16
           \\
           &   &   &   % &
                &   $2^{15}$
                & 3.23 & \underline{1.81} % CMC
                & 3.22  & \textbf{2.22} % Opt m 16
           \\ \hline
             \multirow{2}{*}{$j^{+}_{2.0}$} & \multirow{2}{*}{.080} & \multirow{2}{*}{.000}  & \multirow{2}{*}{.000}
             &   $2^{12}$
                & .08 & \underline{.613} % CMC
                & .09 & \textbf{1.06} % Opt m 16
           \\
           &   &   &   % &
                &   $2^{15}$
                & .10 & \underline{.668} % CMC
                & .10 & \textbf{.693} % Opt m 16
           \\ \hline
            \multirow{2}{*}{$h_1$} & \multirow{2}{*}{12.3} & \multirow{2}{*}{10.1} & \multirow{2}{*}{.755}
            &   $2^{12}$
                & 11.6  & \underline{1.10} % CMC
                & 12.0  & \textbf{1.63} % Opt m 16
           \\
           &   &   &   % &
           &   $2^{15}$
                & 12.2 & \underline{1.54} % CMC
                & 12.1 & \textbf{1.83} % Opt m 16
           \\ \hline
            \multirow{2}{*}{$h_2$} & \multirow{2}{*}{3.18} & \multirow{2}{*}{2.75}  & \multirow{2}{*}{.868}
           &   $2^{12}$
                & 3.03 & \underline{1.35} % CMC
                & 3.07 & \textbf{1.50} % Opt m 16
           \\
           &   &   &   % &
                &   $2^{15}$
                & 3.07 & \textbf{1.50} % CMC
                & 3.07 & \underline{1.47} % Opt m 16
           \\ \hline

\end{tabular}
}
\end{wraptable}

\smallskip\par
\textbf{Example \examplefirst{}.}
Random vector $\X=(X_1,X_2)^T$ has a density
$g(x_1,x_2)=x_1 \exp(-x_1(x_2+1))$ for $x_1,x_2\geq 0$.
The exponential $\displaystyle g_{X_1}(x_1)=e^{-x_1}$ and Pareto $ g_{X_2}(x_2)=(x_2+1)^{-2}$ distributions are marginal distributions of $g$
(in Appendix \ref{sec:app_Example2} we describe how to sample from $g$).
We will estimate $I=\Prob(X_1>t, X_2>t)$ for $t\in\{0.9, 1.0, 1.2\}$
(i.e., $\Exp j_t^{+}(\X)$) and expectations of:
{\small
\begin{equation}\label{eqn:h1h3}
 h_1(\x)   =   (\log|x_1 x_2|)^{-1}j^{+}_1(\x),  \quad h_2(\x)   =   \sin(x_1 x_2)j^{+}_1(\x),
 \quad  h_3(\x)   =    (x_1 x_2)^{-1}j^{+}_1(\x).
\end{equation}
}
We trained the model on 1000 observations, then estimated $I$
using $R\in\{2^{12},2^{15}\}$ replications.

The results for cartesian
stratification are gathered in
 \begin{wrapfigure}{r}{0.6\textwidth}%[h]%{r}{0.4\textwidth}
\includegraphics[width=0.45\linewidth]{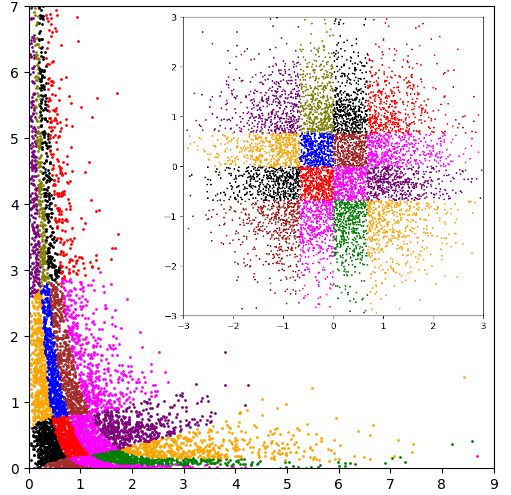}
\includegraphics[width=0.45\linewidth]{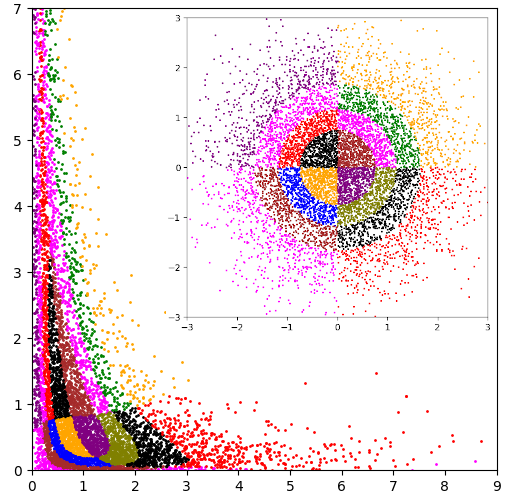}
\caption{Example \examplefirst{}: $R=2^{13}$ points and $m=16$ strata.
Cartesian (left column) and spherical (right column). Smaller plots: 2D iid standard normal; larger plots:
points mapped through $\F$, colors denote corresponding strata.}
 \label{fig:Example2d_norm_flow_16str}
\end{wrapfigure}
Table \ref{tab:Example2d_NO_cart} (only \texttt{E} and \texttt{AC} and one model reported,
more results are provided in Appendix in Table \ref{tab:Example2d_NO_cart_full}).

In Fig.~\ref{fig:Example2d_norm_flow_16str}, two stratification methods are illustrated.

In Fig.~\ref{fig:ExNO2d_cartesian_lines_short}, estimations and corresponding $95\%$ confidence
intervals from 100 simulations for $\hat{Y}_{R}^{\mathcal{F},\rm CMC}$ and
$\hat{Y}_{R,m=16}^{\mathcal{F},\rm opt,M1}$, both with $R=2^{12}$, are shown.
As expected from theory, the probability that a confidence interval includes the true
value of $I$ is 95$\%$: 5$\%$ of CMC estimator intervals and 8$\%$ of stratified
estimator intervals did not contain $I$ (the number of intervals not containing $I$
follows a binomial distribution with parameters 100 and 0.05; observing
$\{3,\ldots,8\}$ such intervals is quite likely with probability 0.81).
Confidence intervals for other estimators (as in Table~\ref{tab:Example2d_NO_cart} and
Table~\ref{tab:Example2d_NO_cart_full}), for spherical stratification (yielding similar
conclusions), and a comparison of methods M1 and M2 are provided in Appendix~\ref{sec:app_Example2}.

\textsl{Common conclusions for Example \examplefirst{} and Example \examplesecond{} (from the Appendix)}. In both cases
$n=1000$ points were not enough to estimate $I$ correctly from data, but they sufficed to train the flow model correctly,
which in turn resulted in better estimations of $I$ (in Example \examplefirst{} even $n=500$ points suffices, see Appendix \ref{sec:app_Example1} for details).
In both examples accuracy from flow models were better than
from data, best accuracy was always obtained for some stratified estimator. In all cases
the variance of CMC estimator was smaller than variance of $\hat{Y}^{\rm obs}_{1000}$
and in majority of cases, the variance any stratified estimator was better than the variance of CMC. Summarizing,
not only the point estimates (\texttt{AC}) are better using M1 or M2 methods, but also uncertainty in such   cases is smaller.

\textbf{Example \examplemultit{}.}
Let $\X=(X_1, \ldots,X_d)^T$ be a $d$ dimensional multivariate Student's $t$-distribution
with $\nu=5$ degrees of freedom, mean $\boldsymbol{\mu}=(0,\ldots,0)^T$
and with
  matrix $\boldsymbol{\Sigma}(i,i)=1$ and
$\boldsymbol{\Sigma}(i,j)=0.2$ for $i,j\in\{1,\ldots,d\}, j\neq i$, i.e.,
with density
${\Gamma((\nu+d)/2)\over \Gamma(\nu/2) \nu^{d/2} \pi^{d/2}
|\boldsymbol{\Sigma}|} \left(1+{1\over \nu} \x^T\boldsymbol{\Sigma}^{-1}\x\right)^{-(\nu+d)/2}.$
For $d=3$ and $d=4$ we aim at estimating $\Exp h_j(\X), j \in \{1, 2, 3\}$, where
$h_j(\x)$ are defined in (\ref{eqn:h1h3}) and adopted to $d=3$ and $d=4$.
In this example we performed simulations for larger $n=20k$ and $R\in\{50k,100k,500k\}$,
we considered up to $m=1296$ strata.
\begin{table*}%[htbp]
 \caption{(Example \examplemultit{}, $d=4$) Numerical results for
 \textbf{cartesian} (method M1) and \textbf{spherical} (method M2) stratification.
 The true value is $I^*$ is $-0.033$ for $h_1$, $.0993$ for $h_2$ and $.400$ for $h_3$.
 }
\label{tab:Example_stud_d4}
%\footnotesize
%\small
{\fontsize{8pt}{9.2pt}\selectfont
\setlength{\tabcolsep}{1.9pt}
%     \begin{tabular}{|p{0.15cm}|p{0.2cm}p{0.3cm}p{0.2cm}|p{0.45cm}|p{0.4cm}p{0.4cm}p{0.49cm}|p{0.49cm}p{0.49cm}p{0.49cm}|p{0.49cm}p{0.49cm}p{0.49cm}|p{0.49cm}p{0.49cm}p{0.55cm}|p{0.49cm}p{0.49cm}p{0.55cm}|p{0.49cm}p{0.49cm}p{0.55cm}} \hline
\begin{tabular}{|p{0.1815cm}|p{0.242cm}p{0.363cm}p{0.242cm}|p{0.5445cm}|
p{0.6cm}p{0.484cm}p{0.45cm}|
p{0.5929cm}p{0.5529cm}p{0.45cm}|
p{0.5929cm}p{0.5529cm}p{0.45cm}|
p{0.5929cm}p{0.5529cm}p{0.45cm}|
p{0.5929cm}p{0.5529cm}p{0.45cm}|}
\hline
 % &  & & & & & & & &  \multicolumn{6}{c|}{Cartesian} &\multicolumn{6}{c|}{Spherical} \\\hline
  $t$ &     \multicolumn{3}{c|}{ {$\hat{Y}_{20k}^{\rm obs}$}}& $R$
    &\multicolumn{3}{c|}{{$\hat{Y}^{\F, \rm CMC}_R$}}&
     \multicolumn{3}{c|}{{$\hat{Y}^{\F, \rm  prop, M1}_{R, m=6\!\times\! 6\!\times\! 6 \!\times\! 6}$}}&
  \multicolumn{3}{c|}{{$\hat{Y}^{\F, \rm  opt, M1}_{R, m=6\!\times\! 6\!\times\! 6 \!\times\! 6}$}}&
        \multicolumn{3}{c|}{{$\hat{Y}^{\F, \rm  prop, M2}_{R, m=10\!\times\! 5\!\times\! 5 \!\times\! 5}$}}&
  \multicolumn{3}{c|}{{$\hat{Y}^{\F, \rm  opt, M2}_{R, m=10\!\times\! 5\!\times\! 5 \!\times\! 5}$}}\\
    \hline
          &
            \texttt{E}$^*$ & \texttt{SD}${}^*$ & \texttt{AC} &
               &
           \texttt{E}$^*$ & \texttt{SD}${}^*$ & \texttt{AC}&
            \texttt{E}$^*$ & \texttt{SD}${}^*$ &\texttt{AC}&
            \texttt{E}$^*$ &\texttt{SD}${}^*$& \texttt{AC}&
            \texttt{E}$^*$ & \texttt{SD}${}^*$& \texttt{AC}&
          \texttt{E}$^*$ & \texttt{SD}${}^*$ & \texttt{AC}
    \\ \hline
&\multirow{3}{*}{\rotatebox{90}{.015}} & \multirow{3}{*}{\rotatebox{90}{.485}} & \multirow{3}{*}{\rotatebox{90}{-.157}} & 50k & \mbox{-.055} & .299 & \underline{.20} & \mbox{-.003} & .313 & \mbox{-.13} & \mbox{-.047} & .043 & .09 & \mbox{-.028} & .296 & \textbf{.36} & \mbox{-.072} & .048 & \mbox{-.07} \\
$h_1$ & & &  & 100k & \mbox{-.023} & .214 & \textbf{.62} & \mbox{-.026} & .216 & .38 & \mbox{-.020} & .034 & .28 & \mbox{-.012} & .211 & .12 & \mbox{-.037} & .035 & \underline{.47} \\
& & & &  500k & \mbox{-.028} & .097 & .44 & \mbox{-.023} & .096 & .571 & \mbox{-.041} & .020 & .72 & \mbox{-.035} & .095 & \textbf{1.1} & \mbox{-.029} & .016 & \underline{.90} \\ \hline
& \multirow{3}{*}{\rotatebox{90}{.103}} & \multirow{3}{*}{\rotatebox{90}{.105}} & \multirow{3}{*}{\rotatebox{90}{.986}} &  50k & .098 & .062 & 1.3 & .098 & .050 & \underline{1.3} & .109 & .007 & .82 & .091 & .052 & \textbf{1.5} & .084 & .006 & .99 \\
$h_2$ & & & & 100k & .096 & .044 & 1.5 & .094 & .035 & \textbf{1.8} & .092 & .006 & 1.6 & .093 & .036 & \underline{1.6} & .084 & .005 & 1.1 \\
& & &  &  500k & .095 & .019 & 1.7 & .096 & .016 & \underline{1.8} & .095 & .003 & 1.8 & .092 & .016 & \textbf{2.3} & .093 & .003 & 1.7 \\ \hline
& \multirow{3}{*}{\rotatebox{90}{.423}} & \multirow{3}{*}{\rotatebox{90}{.351}} & \multirow{3}{*}{\rotatebox{90}{1.24}} &  50k & .416 & .215 & 1.5 & .408 & .161 & \underline{1.5} & .427 & .021 & 1.4 & .395 & .173 & \textbf{1.6} & .402 & .021 & 1.2 \\
$h_3$ & & & & 100k & .410 & .151 & \underline{1.7} & .414 & .111 & 1.5 & .393 & .017 & 1.2 & .394 & .126 & \textbf{1.8} & .392 & .017 & 1.4 \\
& & & &  500k & .412 & .068 & 1.7 & .411 & .050 & 1.6 & .400 & .009 & \textbf{3.1} & .401 & .054 & \underline{2.1} & .432 & .009 & 1.4 \\
              \hline
\end{tabular}
}
\end{table*}

Results for $d=4$ are presented in Table  \ref{tab:Example_stud_d4}
(results for $d=3$ are provided in Appendix \ref{sec:app_Example2}).
As expected the variance of the stratified estimator with optimal allocation is lower than the variance of the proportionally stratified counterpart, which in turn is  lower than the variance of CMC estimator.
In most cases, the accuracy of the stratified estimator is higher than the one obtained in CMC method. In majority of considered cases the flow-based estimate yielded higher accuracy than the estimate obtained from train data. Similar conclusions can be drawn for estimations of $\Exp j_t^{+}(\X)$  (experiments results not shown).

\textbf{Example \examplereal{} (Real-world example)}.
For a real-world scenario, we considered 2-dimensional data from the AmeriGEOSS \cite{AmeriGEOSS} dataset
containing wind measurements in Papua New Guinea.

%\begin{table*}[htbp]
%\begin{table}[htbp]
\begin{wraptable}{r}{0.6\textwidth}
\caption{(Example \examplereal{}) Numerical results for \textbf{spherical} (method M2) stratification.}
\label{tab:Example2d_WIND_spher}
{\fontsize{8pt}{9.2pt}\selectfont
\setlength{\tabcolsep}{1.9pt}
\begin{tabular}{|p{.25cm}|p{0.6cm}p{0.6cm}|p{0.5cm}|p{0.6cm}p{0.6cm}|p{0.6cm}p{0.6cm}|p{0.7cm}p{0.7cm}|}
\hline
$F$ & \multicolumn{2}{c|}{$\hat{Y}_{3000}^{\rm obs}$} & $R$
    & \multicolumn{2}{c|}{$\hat{Y}^{\F, \rm CMC}_R$}
    & \multicolumn{2}{c|}{$\hat{Y}^{\F, \rm opt}_{R,2\times2}$}
    & \multicolumn{2}{c|}{$\hat{Y}^{\F, \rm opt}_{R,4\times4}$} \\
\hline
 & \texttt{E} & \texttt{SD}${}^*$ & & \texttt{E} & \texttt{SD}${}^*$ & \texttt{E} & \texttt{SD}${}^*$ & \texttt{E} & \texttt{SD}${}^*$ \\
\hline
\multirow{2}{*}{$g_1$} & \multirow{2}{*}{.192} & \multirow{2}{*}{.719}  & $2^{12}$
    & .194 & .618
    & .191 & \underline{.371}
    & .193 & \textbf{.175} \\
 &  &  & $2^{15}$
    & .194 & .218
    & .193 & \underline{.131}
    & .193 & \textbf{.064} \\
\hline
\multirow{2}{*}{$g_2$} & \multirow{2}{*}{.001} & \multirow{2}{*}{.087} & $2^{12}$
    & .001 & .075
    & .001 & \underline{.035}
    & .001 & \textbf{.034} \\
 &  &  & $2^{15}$
    & .001 & .026
    & .001 & \underline{.012}
    & .001 & \textbf{.010} \\
\hline
\multirow{2}{*}{$g_3$} & \multirow{2}{*}{.004} & \multirow{2}{*}{.009} & $2^{12}$
    & .004 & .008
    & .004 & \underline{.005}
    & .004 & \textbf{.004} \\
 &  &  & $2^{15}$
    & .004 & .003
    & .004 & \underline{.001}
    & .004 & \textbf{.001} \\
\hline
\end{tabular}
}
\end{wraptable}
%\end{table}

We chose the mean and standard deviation of air pressure (after applying a standard difference transform) collected between 02/15/18 and 02/28/19. We trained the model on the first 3000 observations, with training conducted over 2000 epochs (all other parameters of the flow model are as in Example \examplefirst{}). We estimated the following functions:
$$ g_1(x_1,x_2)=\mathbf{1}(\max(x_1,x_2)>0.01),\
 g_2(x_1,x_2)={x_2\over (1+x_1^2)}, \ g_3(x_1,x_2)=|x_1 x_2|.
 $$
The results for spherical stratification are provided in
Table~\ref{tab:Example2d_WIND_spher}.
In Appendix \ref{sec:app_ExampleWIND}
we present results for additional functions, as well as for cartesian stratification,
and show the training data and samples from the flow model.
Note that we do not know the exact values of $I_i=\Exp g_i(\X)$, so we focus on methods with the smallest \texttt{SD}.
In all cases, the stratified estimator with optimal allocation and 16 strata gave the best results.
In many cases, the estimated values are very similar – e.g., for $g_3$, all estimators
produced a value of 0.0046 (differing only in further decimal places); however, using
spherical stratification with 16 strata and optimal allocation yielded the smallest \texttt{SD},
indicating that with high probability the true value of $I$ lies in a narrow interval
$[0.0046 \pm z_{1-\alpha/2}\texttt{SD}]$ (cf.~Fig.~\ref{fig:ExNO2d_cartesian_lines_short}).

The next two examples address a problem of stratification in higher dimensions. 
\smallskip\par
\textbf{Example \examplelarged{} (30D)}.
Synthetic 30-dimensional example: most coordinates are independent with different distributions and parameters,
while some (non-consecutive) are multivariate non-standard normal. A precise description, including a pairplot,
is provided in
Appendix \ref{sec:app_Example30d}.
The flow model was trained on only $n=500$ observations.
For the M\texttt{rad} method, we stratify the radius into $m\in\{3,7\}$ strata, while for $\textrm{M}\texttt{High3}$
and $\textrm{M}\texttt{Rand3}$, we stratify three selected dimensions into $m_0=3$ strata (resulting in $3^3=27$ strata).
Detailed numerical results for estimating $j^+_{t}$  for $t\in\{-1.0, -0.8, -0.6\}$ with $R\in\{2^{12}, 2^{13}\}$, namely estimators
$\hat{Y}^{\F, \rm  prop, \mathrm{M}\texttt{rad}}_{R,m=3}$,
$\hat{Y}^{\F, \rm  opt, \mathrm{M}\texttt{rad}}_{R,m=3}$,
and
$\hat{Y}^{\F, \rm  prop, \mathrm{M}\texttt{rad}}_{R,m=7}$.
are presented in Table~\ref{tab:Example_30d_rad_part1}.
Estimators
$\hat{Y}^{\F, \rm  opt, \mathrm{M}\texttt{rad}}_{R,m=7}$,
$\hat{Y}^{\F, \rm  prop, M\texttt{Rand3}}_{{R},m=3\times 3\times 3}$, and
$\hat{Y}^{\F, \rm  prop, M\texttt{High3}}_{{R},m=3\times 3\times 3}$
are provided in Appendix \ref{sec:app_Example30d}.

For $t=-0.6$, estimation of $I$ from observations yielded 0, as the probability $I=0.001944$ is small,
and none of the $n=500$ observations had all coordinates larger than $t$.
Even in this case, the flow model captured the distribution well, yielding reasonable results --
especially $\textrm{M}\texttt{Rand3}$ and $\textrm{M}\texttt{High3}$ produced good estimates
(\texttt{E}$=0.001913$).
In one case, CMC yielded the best accuracy; in all other cases, stratified methods provided the best results.
\begin{table*}[h]
\caption{(Example  \examplelarged{}) Results for $\mathrm{M}\texttt{rad}$,
$\mathrm{M}\texttt{High3}$ and $\mathrm{M}\texttt{Rand3}$ methods. Part 1  }
\label{tab:Example_30d_rad_part1}
{\fontsize{8pt}{9.2pt}\selectfont
\setlength{\tabcolsep}{1.9pt}
\begin{tabular}{|p{0.65cm}|p{0.45cm}|p{0.5cm}p{0.5cm}p{0.55cm}|p{0.45cm}|p{0.55cm}p{0.55cm}p{0.55cm}|p{0.55cm}p{0.55cm}p{.55cm}|p{0.55cm}p{0.55cm}p{0.55cm}|p{0.55cm}p{0.55cm}p{0.55cm}|}\hline
    \multicolumn{1}{|c|}{$f$} &   $I^*$& \multicolumn{3}{c|}{ \small{$\hat{Y}_{n=500}^{\rm obs}$}}& $R$
     &\multicolumn{3}{c|}{{$\hat{Y}^{\F, \rm CMC}_R$}}&
      \multicolumn{3}{c|}{{$\hat{Y}^{\F, \rm  prop, \mathrm{M}\texttt{rad}}_{R,m=3}$}}&
   \multicolumn{3}{c|}{{$\hat{Y}^{\F, \rm  opt, \mathrm{M}\texttt{rad}}_{R,m=3}$}}&
         \multicolumn{3}{c|}{{$\hat{Y}^{\F, \rm  prop, \mathrm{M}\texttt{rad}}_{R,m=7}$}} \\
     \hline
           & &
               {\texttt{E}${}^*$}&{\texttt{SD}${}^*$}&{\texttt{AC}}&    &
               {\texttt{E}${}^*$}&{\texttt{SD}${}^*$}&{\texttt{AC}}&
               {\texttt{E}${}^*$}&{\texttt{SD}${}^*$}&{\texttt{AC}}&
               {\texttt{E}${}^*$}&{\texttt{SD}${}^*$}&{\texttt{AC}}&
               {\texttt{E}${}^*$}&{\texttt{SD}${}^*$}&{\texttt{AC}}
     \\
    \hline
   \multirow{2}{*}{$j^+_{-1.0}$} & \multirow{2}{*}{\!1.37} &\multirow{2}{*}{1.80}& \multirow{2}{*}{.595} &\multirow{2}{*}{.504}  &     $2^{12}$
   & 1.67 & .200 & .733% CMC
   & 1.52 & .191 & \textbf{1.18}	 % Prp m 4
   & 1.56 & .191 & \underline{.832}% Opt m 4
   & 1.63 & .197 & .771	% PRop m 16
   \\
   &   &   &   &   &   $2^{13}$
 & 1.67 &   .141  & 	\underline{.703}% CMC
   & 1.71 & .143 & .633	 % Prp m 4
   & 1.71 & .145 & .683 % Opt m 4
   & 1.66& .141	& .686% PRop m 16
   \\ \hline
     \multirow{2}{*}{$j^+_{-0.8}$} & \multirow{2}{*}{\!.560} &\multirow{2}{*}{.400}& \multirow{2}{*}{.282} & \multirow{2}{*}{.543} &   $2^{12}$
   & .735	 & .133 & .635% CMC
   & .681 &.128	 &\textbf{1.05} % Prp m 4
   & .477 &	.089 &	.768% Opt m 4
   & .662	& .126 & \underline{1.01} % PRop m 16
   \\
   &   &   &   &   &   $2^{13}$
   & .711	& .092	& .589 % CMC
&   .669	& .089	& .807 % Prp m 4
   & .590 &  .082 & \underline{.936}% Opt m 4
   & .655 & .089 & .777 % PRop m 16
   \\ \hline
        \multirow{2}{*}{$ j^+_{-0.6}$} & \multirow{2}{*}{\!.194} &\multirow{2}{*}{.000} & \multirow{2}{*}{.000} & \multirow{2}{*}{.000} &   $ 2^{12}$
        & 	.236 &	.075	& \textbf{1.00} % CMC
    &  .227 & .073 &	.683 %  Prp m 4
   &.170	&.057& .691  %Opt m 4
   & .227	&.073 & .805% PRop m 16
   \\
   &   &   &   &   &   $ 2^{13}$
   & 	.267	 & .056 & .506 % CMC
    &  .250	&.055 & .613 %  Prp m 4
   & 3.45	& .029 & .278  %Opt m 4
   & .238	&.053	& .784% PRop m 16
   \\ \hline
\end{tabular}
}
\end{table*}

\subsubsection*{Model architecture and training details.}
The CNF component of our model follows the FFJORD architecture
\cite{grathwohl2018ffjord}, and our implementation is based on the original
code\footnote{\url{https://github.com/rtqichen/ffjord}}.
We use two stacked CNF blocks, each composed of two concatsquash layers with
$\tanh$ activations and batch normalization. The number of hidden units
(\texttt{dims}) depends on the example: 16 units for Examples \examplesecond{}
and \examplemultit{}, 64 units for Examples \examplefirst{} and \examplereal{},
256 units for Examples \examplesdf{} and \examplelargestd{}, and 128 units
for Example \examplelarged{}. The models are trained with the ADAM optimizer,
using learning rates between $0.0001$ and $0.01$ and between several hundred
and $5000$ epochs, depending on the example.

\section{Comprehensive Comparison}
For Example \examplefirst{}, we conducted two additional comparisons. First, using the same training observations,
we trained GMMs with 4, 10, 20, 50, and 100 components (denoted as GMM($k$)),
then sampled from the trained models and computed function estimations.
Results are shown in Table~\ref{tab:Example2d_NO_gmm}.
\begin{table}[h!]    
\caption{(Example 2) Comparison with GMM and with flow model $\F_{5k}$ trained on 5000 observations. }%No 
\label{tab:Example2d_NO_gmm}
{\fontsize{8pt}{9.2pt}\selectfont
\setlength{\tabcolsep}{1.9pt}
 \begin{tabular}{|p{.4cm}|p{0.5cm}|p{.4cm}|p{.6cm}p{.5cm}|>{\columncolor[gray]{0.8}}c>{\columncolor[gray]{0.8}}cp{.55cm}p{.15cm}|p{.5cm}p{.45cm}|p{.55cm}p{.45cm}|p{.55cm}p{.45cm}|p{.6cm}p{0.2cm}|}\hline
   $f$ &   $I$&   $R$
    &\multicolumn{2}{c|}{{$\hat{Y}^{\F_{1k} }_R$}}&
     \multicolumn{2}{c|}{{$\hat{Y}^{\F_{5k} }_R$}}&
          \multicolumn{2}{c|}{{$\hat{Y}^{\rm GMM(4) }_R\quad$}}&
               \multicolumn{2}{c|}{{$\hat{Y}^{\rm GMM(10) }_R$}}&
                    \multicolumn{2}{c|}{{$\hat{Y}^{\rm GMM(20) }_R$}}&
                    \multicolumn{2}{c|}{{$\hat{Y}^{\rm GMM(50) }_R$}}&
                    \multicolumn{2}{c|}{{$\hat{Y}^{\rm \tiny{GMM(100)}}_R$}}\\
    \hline
           & &
             & \texttt{E} &     \texttt{AC} &

             \texttt{E} &   \texttt{AC}&
              \texttt{E} &  \texttt{AC}&
              \texttt{E} &  \texttt{AC}&
              \texttt{E} &  \texttt{AC}&  \texttt{E} &  \texttt{AC}&  \texttt{E} &  \texttt{AC}
    \\
    \hline
               \multirow{2}{*}{$j^{+}_{0.5}$} & \multirow{2}{*}{.315}   &   $2^{12}$
                & .304 &    \textbf{1.68} % CMC
                & .321 &  1.72  &  .359 & .857
                & .305 &  \underline{1.54}   % 10 components
                & .306 &  1.22  % 20 components
                & .323 & 1.52
                & .335 & 1.20
           \\
           &   &       $2^{15}$
                & .307 &    \textbf{1.63} % CMC
                & .306 &  1.54  & .355 & .898
                & .319 & \underline{1.52}    % 10 components
                & .321 & 1.43    % 20 components
                & .327 & 1.41
                & .327 & 1.40
           \\ \hline
           \multirow{2}{*}{$j^{+}_{1.2}$} & \multirow{2}{*}{.032}  &   $2^{12}$
                & .031 &    \textbf{1.25} % CMC
                & .025 &  1.68  & .063 & .019
                & .040 & .060 % 10 components
                & .039 & \underline{.644}                     % 20 components
                & .023 & .557
               & .024&  .569
           \\
           &   &       $2^{15}$
                & .032 &    \textbf{1.81} % CMC
                & .029 &  .882 & .068& {\mbox{-.035}}
                & .040 & .060 % 10 components
                & .038 & \underline{.796}    % 20 components
                & .022 & .504
                & .024 & .602
           \\ \hline
             \multirow{2}{*}{$j^{+}_{2.0}$} & \multirow{2}{*}{.001} &   $2^{12}$
                & .001 &     \textbf{.613} % CMC
                & .001 &  .945 &  .002 & {\mbox{-.134}}
               & 0 & 0     % 10 components
                &  .000 & \underline{.388}   % 20 components
                & .000 & .388
                & .002 & .152
           \\
           &   &       $2^{15}$
                & .001 &    \textbf{.668} % CMC
                & .001 &  1.00  & 0.02 & {\mbox{-0.13}} % Prop m 4
                  &.000 & .130  % 10 comp
                &  .000 & \underline{.284} % 20 comp
                &.000 & .089
                &.000 & .109
           \\ \hline
           \multirow{2}{*}{$h_1$} & \multirow{2}{*}{.033}  &  $2^{12}$
                & .031 &     \textbf{1.408} % CMC
                & .033 &  1.71  & .052 & .261
                 &.050 & .303    % 10 components
                 & .042 & .577    % 20 components
                 & .037 & \underline{.966}
                 & .038 & .855
           \\
           &   &       $2^{15}$
                & .032 &     \textbf{1.51} % CMC
                & .032 &  1.51  &  .058 & .131
                & .052 & .251     % 10 components
                  &.040 & .723  % 20 components
                  & .028 & .776
                  & .034 & \underline{1.33}
           \\ \hline
           \multirow{2}{*}{$h_2$} & \multirow{2}{*}{.032} &   $2^{12}$
                & .030 &     \textbf{1.35} % CMC
                & .031 &  1.71 & .051 & .272
                 & .040 & .588    % 10 components
                 & .034 & \underline{1.14}    % 20 components
                 &.034 & 1.07
                 &.040 & .064
           \\
           &   &       $2^{15}$
                & .031 &     \textbf{1.50} % CMC
                & .029 &  .998  & .054 & .154
                & .052 & .251    % 10 components
                & .034 & \underline{1.22}    % 20 components
                &.034 & 1.21
                & .035 & .957
           \\ \hline
            \multirow{2}{*}{$h_3$} & \multirow{2}{*}{.123}  &   $2^{12}$
                & .112 &     \textbf{1.08} % CMC
                & .118&      1.14 & .185 & .296
                 &.139 & .894   % 10 components
                &  .112 & \underline{1.05}   % 20 components
                &.148 & .698
                &.158 & .551
           \\
           &   &      $2^{15}$
                & .122  &    \textbf{1.53} % CMC
                & .116 &     1.41 & .185 & .296
                 &.160 & .525    % 10 components
                 &.132 & \underline{1.14}   % 20 components
                 & .132 & 1.12
                 & .139 & .882
           \\ \hline
\end{tabular}
}
\end{table}

Estimations from the flow model, denoted
$\hat{Y}^{\F_{1k}}_R$ in the Table~\ref{tab:Example2d_NO_gmm}, are the same as
$\hat{Y}^{\F,\rm CMC}_R$ from Table~\ref{tab:Example2d_NO_cart}. As clearly seen,
in all cases, the flow model has the best accuracy, often by a large margin
from any GMM (in each row, the best accuracy is \textbf{bolded}, second-best is \underline{underlined};
note that the column $\hat{Y}^{\F_{5k}}_R$ is not considered here).
We did not apply stratification in this comparison. With stratification, the flow model achieves even better accuracy
(as evidenced in Table~\ref{tab:Example2d_NO_cart}). It is worth noting that if
the data distribution were a mixture of Gaussians, GMMs would yield better results.
%\smallskip\par
We also trained a flow model $\hat{Y}^{\F_{5k}}$ on a training set of size 5000. The results are provided in Table~\ref{tab:Example2d_NO_gmm} (gray column).
The performance of $\hat{Y}^{\F_{1k}}$ and $\hat{Y}^{\F_{5k}}$ remains consistent,
showing that 1000 points were \textsl{enough} to train the flow model accurately.
More detailed model performance as the function training dataset size is further investigated in next section.

\section{Sample size performance}\label{sec:sample_perform}
Fig.~\ref{fig:ablation} illustrates the accuracy of estimating 
$I=\Exp f(\X)$ from Example  \examplefirst{} using samples generated from a normalizing flow 
trained on datasets of varying size $n_{\text{train}}$. 
\begin{wrapfigure}{r}{0.6\textwidth}
%\centering
\includegraphics[width=1\linewidth]{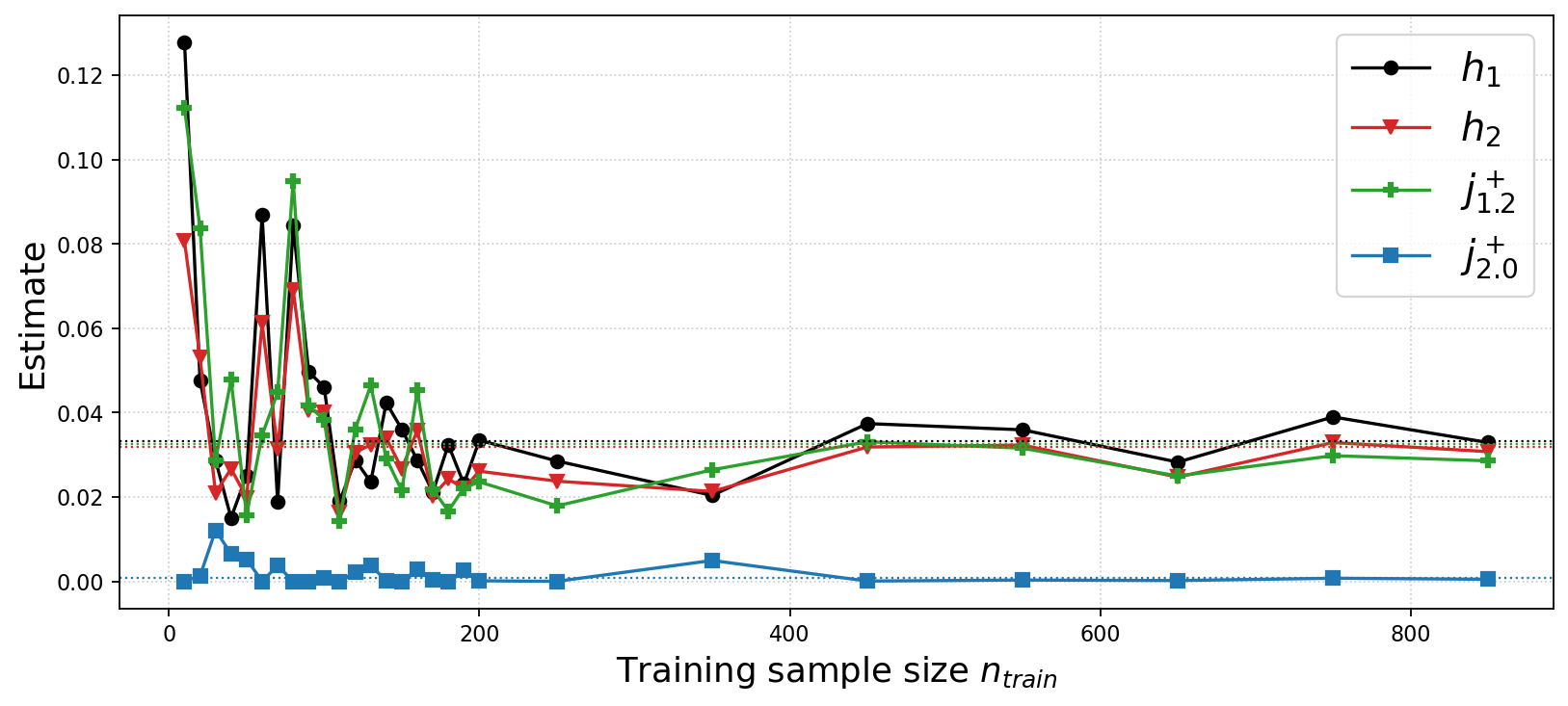}
\caption{Estimates vs.~training sample size $n_{\text{train}}$.}
\label{fig:ablation}
\end{wrapfigure}
Horizontal lines represent the true values of $I$, while the curves correspond 
to flow-based estimates. The results show that already for training samples of several hundreds, 
the flow achieves sufficient accuracy and the estimates stabilize around 
the true values (cf.~Table~\ref{tab:Example2d_NO_cart}).

\section{Related work}
The scenario where the distribution of $\X$ is known, and the goal is to construct an estimator of (\ref{eq:I_ExpfX}) with reduced variance using a neural network, has been studied in the literature.
For example, the \textbf{control variates} variance reduction technique, which aims to find a family of functions $g$ for which $\Exp g(\X)=0$, was investigated in \cite{mira2013zero}. In \cite{wan2018neural}, the authors proposed \textsl{neural control variates} (NCV), where $g$ is designed as a neural network.
A similar approach (with applications in light transport for computer vision), using a normalizing flow for $g$, was studied in \cite{muller2020neural}. An application of Stein's identity to control variates, where $g$ can be a neural network, was presented in \cite{liu2017action}. Another commonly used method in the context of integration is the \textbf{importance sampling} variance reduction technique. In \cite{muller2019neural}, the authors used a nonlinear independent components estimation (NICE) flow model \cite{dinh2014nice} to learn a simple variational distribution.

Stratified sampling has also been used to estimate (\ref{eq:I_ExpfX}) with normalizing flows. In \cite{cundy2020flexible}, the authors tackle a slightly different problem-estimating high-dimensional
integrals and controlling the bias-variance trade-off by interpolating between sampling and variational regimes. They employed the RealNVP model \cite{dinh2016density} with a uniform base distribution. Spherical stratification with adaptive radial splitting was used in
\cite{song2023adaptive}.
Estimating the expectation of a real-valued function of a random variable using only
observations has been studied in the context of learning tail behavior.
The ability of popular normalizing flow models to capture tail probabilities was examined
in \cite{jaini2020tails}, where the authors propose tail-adaptive flows, which perform
better for non-Gaussian tails. This work is extended in \cite{laszkiewicz2022marginal},
where the authors allow for heavy-tailed base distributions on selected coordinates.

\section{Conclusions and limitations}
We introduce a flow-based model that uses a parametrized neural network, allowing
for greater flexibility in modeling unknown data distributions. Our approach is validated
in the context of stratified sampling and outperforms traditional methods in reducing
estimation uncertainty on both synthetic and real-world datasets, including high-dimensional
cases. Overall, this work offers a promising path for improving the accuracy and reliability
of statistical data analysis. In future work, we plan to evaluate the approach across
a wide range of real-life applications. One limitation is that the normalizing flow used
is less accurate for heavy-tailed data distributions. Additionally, the model requires
training, which can be time-consuming.
%\smallskip
\medskip
\par
\textbf{Training time}. 
Training times are available in Appendix \ref{sec:app_training_time}.
%\smallskip
\medskip
\par
\textbf{Future work}.
An in-depth study of alternative  flow base distributions (e.g., those capable of
modeling heavy-tailed distributions) and methods for \textbf{choosing strata}
(note that the optimal split $R_1, \ldots, R_m$ presented in Section \ref{sec:strat_est}
is for \textsl{given} strata)
are left for future work.
In Appendix \ref{sec:app_choosing_strata}, we present preliminary results for a method that ``rotates'' spherical strata based on the estimated function to improve estimation.
For example, rotating the strata from Fig.~\ref{fig:Example2d_norm_flow_16str} (right) by $29.33^{\circ}$ (see the corresponding figure in Appendix \ref{sec:app_choosing_strata}) yields better estimations.
%\smallskip
\medskip
\par
 \textbf{Acknowledgments}. R.T. acknowledges support from the Dioscuri program initiated by the Max Planck Society, jointly managed with the National Science Centre (Poland), and mutually funded by the Polish Ministry of Science and Higher Education and the German Federal Ministry of Research, Technology and Space.

\bibliographystyle{abbrv}
\bibliography{LNCS_library}

\clearpage
\appendix
\thispagestyle{empty}

% Supplementary material: To improve readability, you must use a single-column format for the supplementary material.
\onecolumn

\renewcommand{\thesection}{A.\arabic{section}} % Reformat the section counter to start with "A"
\renewcommand{\thesubsection}{\thesection.\arabic{subsection}} % Subsection includes section number

\setcounter{theorem}{0}
\setcounter{figure}{0}
\setcounter{equation}{0}
\setcounter{table}{0}
\renewcommand{\thetheorem}{A\arabic{theorem}} % Reformat the theorem counter to start with "A"
\renewcommand{\thefigure}{A\arabic{figure}}
\renewcommand{\theequation}{A\arabic{equation}}
\renewcommand{\thetable}{A\arabic{table}}

\section{More on stratified sampling}\label{sec:more_on_str}
As mentioned, the proof of Theorem \ref{thm:optimal} and the decomposition of
a variance of CMC estimator may be found e.g., in Madras (Theorem 3.3).
In next two subsections we provide   different proofs of those.
\subsection{Variance decomposition}
 Recall that $I=\Exp Y$ and that expectation of $I$ on $j$-th strata was defined as $I^j = \Exp[Y |Y \in A^j]$. Define a random variable $J(\omega)=j$ iff $\omega \in A^j$.
 We have
 \begin{eqnarray*}
\Exp \Var(Y|J) & = & \sum_{j=1}^m p_j \Var (Y|J=j) =
\sum_{j=1}^m p_j \sigma_j^2,\\
\Var\Exp (Y|J) & = & \Exp( \Exp(Y|J)-\Exp(\Exp(Y|J)))^2
=\Exp(\Exp(Y|J)-EY)^2 \\
&=&\Exp(\Exp(Y|J)-I)^2=\sum_{j=1}^m (E(Y|J=j)-I)^2=\sum_{j=1}^m p_j(I^j-I)^2.
\end{eqnarray*}
Using the \textsl{law of total variance} $\Var Y=\Exp \Var(Y|J)+\Var\Exp(Y|J)$, we may decompose a variance into \textbf{within} and \textbf{between-stratum} components
given in (\ref{eq:within_between_stratum_var}), i.e.,
\begin{equation}\label{eq:within_between_stratum_var}
\Var Y = \sum_{j=1}^m p_j \sigma_j^2 + \sum_{j=1}^m p_j(I^j-I)^2.
\end{equation}

In below compuations we ignore    issues
with roundings --  e.g., in practice we take $R_j=\lceil{R p_j\rceil}$.

\begin{lemma}\label{eq:app_lem_decomp}
 We have a following decomposition of a variance of CMC estimator involving
 variance of proportional allocation estimator:
$$
\begin{array}{llll}
\displaystyle
\Var(\hat{Y}^{\rm CMC}_R) &=& \displaystyle
\Var(\hat{Y}^{\rm pa}_R)+{1\over R} \sum_{j=1}^m p_j (I^j-I)^2\geq \Var(\hat{Y}^{\rm pa}_R).
\end{array}
$$
\end{lemma}
\begin{proof}
 Let us introduce the following notation. For $\x=(x_1,\ldots,x_m)$
(here $x_i> 0$ and $x_1+\ldots +x_m=1$)   define
\begin{equation}\label{eq:sigma.strata}
 \sigma^2_{\rm str}(\x):=\sum_{i=1} {p_j^2\over x_j}\sigma^2_j.
\end{equation}
In proportional allocation we have $R_j=p_jR$, thus
$$\Var (\hat{Y}^{\rm pa}_R) = {1\over R}\sigma^2_{\rm str}(R_1/R,\ldots,R_m/R)={1\over R}\sum_{j=1}^m p_j\sigma_j^2.$$
Dividing both sides of (\ref{eq:within_between_stratum_var}) we have
$$
{1\over R}\Var Y = {1\over R}\sum_{j=1}^m p_j \sigma_j^2 + {1\over R}\sum_{j=1}^m p_j(I^j-I)^2
$$
i.e.,
$$
\Var(\hat{Y}^{\rm CMC}_R)=
\Var(\hat{Y}^{\rm pa}_R)+{1\over R}\sum_{j=1}^m p_j \left(I^j-I\right)^2.
$$
\end{proof}

 \subsection{Optimal allocation -- proof}

Let us start with describing an intuition on why the proportional allocation may not be optimal: for some strata $A^j$ we may have
 that the variability (sample variance) of replications of $Y^j$ is smaller
 than the variability of replications of $Y^{j'}$ sampled from strata $A^{j'}$.
 It is reasonable  to  ``spend'' more replications for strata $A^{j'}$ than $A^j$.
 The following theorem
 makes this intuition precise and provides the best possible (for fixed strata) split of~$R$.
 \begin{theorem}\label{thm:optimal}
 Let $m$ strata $A^1,\ldots,A^m$ and total budget $R$ of simulations be fixed.
 Let $\hat{Y}^{\rm str}_R$ be a~stratified estimator with a general split,
 whereas let $\hat{Y}^{\rm opt}_R$ be the stratified estimator with \textsl{optimal split}
 \begin{equation}\label{eq:split}
R_j={p_j\sigma_j\over \sum_{i=1}^m p_i \sigma_i} R, \quad j=1,\ldots,m.
 \end{equation}
Then  we have
 $\displaystyle \Var(\hat{Y}^{\rm opt}_R)\leq \Var(\hat{Y}^{\rm str}_R).$
 \end{theorem}

\begin{proof}
We may rewrite rhs of variance given in (\ref{eq:estim_str}) as
\begin{eqnarray}\label{eq:stratified_var}
{1\over R}\left(p_1^2\frac{R}{R_1}\sigma_1^2+\ldots+p_m^2\frac{R}{R_m}\sigma_m^2\right)
&=&{1\over R}\left(\frac{p_1^2}{R_1/R}\sigma_1^2+\ldots+\frac{p_m^2}{R_m/R}\sigma_m^2\right)\\
&=&{1\over R}\sigma^2_{\rm str}((R_1/R,\ldots,R_m/R))
      \;.\label{eq.war:str}
\end{eqnarray}
Denote $D=\sum_{j=1}^m p_j\sigma_j$.
The  stratified estimator with split given in (\ref{eq:split}),
i.e., with $ $   $R_i={p_j\sigma_j\over D} R$ is denoted by $\hat{Y}^{\rm opt}_R$.
Recalling (\ref{eq:stratified_var}), using
(\ref{eq:sigma.strata}) and $R_i/R={p_j\sigma_j\over D}$ we have
$$\Var\hat{Y}^{\rm opt}_R={1\over R} \sigma^2_{\rm str}((R_1/R,\ldots,R_m/R))
={1\over R}\sum_{j=1}^m {p_j^2\over p_j\sigma_j/D} \sigma_j^2
={1\over R}D \sum_{j=1}^m p_j\sigma_j={1\over R} D^2.
$$
\noindent
Let $\hat{Y}^{\rm str}_R$ be some general stratified
estimator with $R_i=x_i R$, where $x_1+\ldots+ x_m=1, x_j\geq 0$.
Denote  $\x=(x_1,\ldots,x_m)$.
Let
$$\boldsymbol{\gamma}^*=(\gamma_1^{*},\ldots, \gamma_m^{*}), \quad \gamma_j^{*}=\frac{p_j\sigma_j}{D}.$$
Using Cauchy-Schwarz inequality, for any $\x$ such that $x_1+\ldots+x_m=1, x_j>0$
we have
\begin{eqnarray*}
\Var \hat{Y}^{\rm opt}_R &=&{1\over R}D^2={1\over R}\left(\sum_{j=1}^m p_j\sigma_j\right)^2  ={1\over R} \left(\sum_{j=1}^m\sqrt{x_j}\frac{p_j\sigma_j}{\sqrt{x_j}}\right)^2\\
&\le & {1\over R} \sum_{k=1}^m x_k\sum_{j=1}^m\frac{p_j^2\sigma_j^2}{x_j}={1\over R}\sum_{j=1}^m\frac{p_j^2}{x_j}\sigma_j^2={1\over R}\sigma^2_{\rm str}(\x)=\Var \hat{Y}^{\rm str}_R\;,
\end{eqnarray*}
thus, the variance is minimial for $\x=\boldsymbol{\gamma}^*$.
\end{proof}
 \section{Spherical stratification w.r.t. to angles}\label{app:sampl_hk}
In Section \ref{sec:method2} we shortly pointed out how to perform spherical stratification w.r.t. angles.
Recall, the main idea is to sample the point uniformly from sphere the hypersphere
$$S_{d-1}=\left\{\x=(x_1,\ldots,x_d): \sum_{j=1}^d x_j^2=1\right\}.$$
Nothe that a point
$\x\in S_{d-1}$ has spherical coordinates
\begin{eqnarray*}
 x_1 &=&\cos \phi_1,\quad  x_2   =   \sin \phi_1 \cos \phi_2,  \ldots, \nonumber\\
 x_{d}  & =&   \sin \phi_1 \cdots  \sin \phi_{d-3} \sin \phi_{d-2} \sin \theta,\label{eq:spher_coord_old}
\end{eqnarray*}
where $\phi_1,\ldots,\phi_{d-2}\in[0,\pi)$ and  $\theta \in[0,2\pi).$
 The transformation between Cartesian and spherical coordinates has Jacobian
  $$J(\phi_1,\ldots,\phi_{d-2},\theta)=\sin^{d-2}\phi_1 \sin^{d-3}\phi_2\cdots \sin\phi_{d-2}.$$
 The area of $S_{d-1}$ is $2\pi^{d/2}/\Gamma(d/2)$, thus the density of a uniform distribution on $S_{d-1}$
 is
 \begin{equation}\label{eqn:density_on_sphere2}
 g(\theta,\phi_1,\ldots,\phi_{d-2})={\Gamma(d/2)\over 2\pi^{d/2}}
 \sin^{d-2}\phi_1 \sin^{d-3}\phi_2\cdots \sin\phi_{d-2}.
 \end{equation}

Using the identity
$c_k=\int_0^{\pi} \sin^k(x)\;dx = B\left(\frac{1}{2}, \frac{k+1}{2}\right), $
where $B(a, b)$ is a~beta function,
one may decompose (\ref{eqn:density_on_sphere2}) into
$$
\frac{1}{2\pi}\prod_{k=2}^{d-1}\frac{1}{c_{d-k}}\sin^{d-k}(\phi_{k-1}) = h_0(\theta) \prod_{k=2}^{d-1}h_{d-k}(\phi_{k-1}),
$$
where $h_0$ is a~density of the uniform distribution on $(0,2\pi)$ and
 $$h_k(\phi)={1\over c_k} \sin^k(\phi), \phi\in(0,2\pi), \quad c_k=\int_\pi \sin^k(x)\;dx = B\left({1\over 2},{k+1\over 2}\right),$$
 Since we have a product,
it means that we may sample from a uniform distribution on $S_{d-1}$ by
sampling $\theta, \phi_1,\ldots,\phi_{d-2}$ with distributions $h_0,h_1,\ldots,h_{d-2}$ correspondingly.
sampling from $h_k$ for general $k$ is not straightforward, mainly because there is no closed formula for $\int_0^t \sin^k(x)\;dx$ and thus for a distribution function of $h_k$  (what would allow for using an inverse method). We will use \textbf{acceptance-rejection} (AR) method.
Recall, to sample from $h_k(x)$ we may sample from another density $q_k(x)$ (also on $(0,\pi)$) assuming we know a constant $C_k$ such that
$h_k(x)\leq C_k q_k(x).$
We may simply take $q_k(x)$ to be a uniform distribution on $(0,\pi)$, i.e., $q_k(x)=1/\pi$.
Since $h_k(x)=c_k\sin^k(x)\leq c_k$, we may simply take $C_k=\pi/c_k$.
Finally, the procedure is as follows: We generate $U, U'$ iid $\mathcal{U}(0,1)$,
set $T=U'\pi$ and accept it if $C_kU/\pi\leq h_k(T)$ (note that it may be simplified
to $U\leq \sin^k(T)$). We repeat the procedure until acceptance.
 The procedure is provided in Algorithm \ref{alg:AR}.   Note that line
 \ref{alg:AR_until} can be simplified to:
 \textbf{until} $U\leq \sin^k(T).$
\medskip\par
  \begin{algorithm}[H]
 \caption{AR Algorithm for sampling from $h_k$.}
 \label{alg:AR}
 \begin{algorithmic}[1]
   \Require constant $c_k$
 \State Set $C_k=\pi/c_k$
   \Repeat
 \State Generate $U'\sim \mathcal{U}(0,1)$ and set $T=U'\pi$\label{alg:AR_sampl}
 \State Generate $U\sim\mathcal{U}(0,1)$ (independent from $U'$)
  \Until{$C_k U /\pi \le h_k(T) $ }\label{alg:AR_until}
 \Return $T$
 \end{algorithmic}
 \end{algorithm}
\medskip\par\noindent
To conduct \textbf{stratified sampling } we
proceed with  AR algorithm with the change  that
we sample $U'$ from some strata \eqref{eq:strata_A},
where $a_0=0, a_{m_0}=1$ and $a_j, j=1\ldots,m_0=1$ are chosen  so that
the strata are equally likely, i.e.,
${1\over c_k}\int_{a_{j-1}}^{a_j} \sin^k(x)\; dx={1\over m_0}.$
Values of $a_j$ are found by numerical integration.
\smallskip\par\noindent
 \textbf{Expected running time of AR Algorithm \ref{alg:AR}.}
 Acceptance-rejection expected running time is $C_k$, which is $\pi/c_k$.
 For example for $k=10$ (recall that $k$ is at most $d-2$) on average  $C_{10}=\pi/.7731=4.063$
 iterations of AR are needed (to sample  one random point), whereas  e.g.,  for $k=200$ on average  $C_{200}=\pi/.1770=17.746$ iterations
 are needed. Thus, in practice, the running time of AR algorithm is not problematic.

\section{Details on  high-dimensional stratification}\label{sec:app_more_high}
We aim to sample multiviariate standard normal vector $(Z_1,\ldots,Z_d)$ in a stratify way using $m$
equally likely strata.
The cartesian method M1 splits each coordinates into $m_0$ strata, see Eq. (\ref{eq:strata_A}).
The spherical method M2 splits radius into $m_r$ strata and each angle into $m_0$ strata.
Thus, if decide to split each dimension in M1 method into $m_0=2$ strata or each angle into $m_0=2$ strata,
we have exponential ($2^d, m_r2^d)$ number of total strata, which is prohibitive for large $d$ (in our Examples \examplelarged{} and \examplelargestd{}
we had $d=30$ and $d=128$ respectively). Here we describe two methods mitigating this problem.

\subsection{$\textrm{M}\texttt{rad}$ method: stratification w.r.t radius}
This is a special case of the spherical method: independently of $d$, in spherical method we may decide to stratify only w.r.t
radius
$D$. We thus  set $m=m_r>1$ and $m_0=1$ and perform the procedure described in the first paragraph of section \ref{sec:method2}. In other words, we split
$$\mathbb{R}^d = A^1 \cup A^2 \cup \ldots \cup A^m$$
into $m$ equally probable strata (e.g., in Fig. \ref{fig:Example2d_norm_flow_16str} top-right we have $m=4$ strata
if we disregard splitting w.r.t angles). Sampling  multivariate standard normal random vector from strata $A^j$, i.e.,
$$\mathbf{Z}^j \stackrel{\mathcal{D}}{=}  (\mathbf{Z} | \mathbf{Z}\in A^j)$$
is given by the procedure:
\begin{itemize}[leftmargin=0.5cm,align=left, itemsep=0.2cm]
 \item sample $\mathbf{Z}=(Z_1,\ldots,Z_d)$ with  $d$ iid standard normal $\calN(0,1)$ random variables,
 \item normalize $\mathbf{Z}$ to get $ \mathbf{Z}'=\left( {Z_1\over ||\mathbf{Z}||},\ldots , {Z_1\over ||\mathbf{Z}||}\right) $,
 \item sample $U\sim\mathcal{U}(0,1)$,
 \item Compute $(D^j)^2=F^{-1}_{\chi^2_d}\left({j\over m_r} + {1\over m_r} U\right)$,
 \item Set $\displaystyle \mathbf{Z}^j=(D^j Z_1',\ldots, D^j Z_d')$.
\end{itemize}
We denote the resulting estimators as $\hat{Y}_{R, m_r=m}^{\mathcal{F},\rm opt, \textrm{M}\texttt{rad}}$ (proportional allocation)  and $\hat{Y}_{R, m_r=m}^{\mathcal{F},\rm opt, \textrm{M}\texttt{rad}}$ (optimal allocation).
Note that the method does not depend on a function $f$ we aim to estimate (recall, the goal is to estimate $I=\Exp f(\X)$).

\subsection{$\textrm{M}\texttt{High3}$ and $\textrm{M}\texttt{Rand3}$ methods: stratifying three chosen coordinates}
Using an obvious decomposition
$\mathbb{R}^d = \mathbb{R} \times \mathbb{R}\ldots \times \mathbb{R}$ ($d$ times), we will split $\eta=3$ chosen \textsl{dimensions} (coordinates)
into $m_0$ strata, obtainig thus in total $m=m_0^3$ strata. Say, the chosen dimensions (we describe below how to choose them)
are $i,j,k$. Let each $A_i^s, A_j^s, A_k^s, s=1,\ldots,m_0$ be the a split of $\mathbb{R}$ into $m_0$ intervals,
such that $\Prob(Z\in A_i^s)=1/m_0$, i.e., as it was given in (\ref{eq:strata_A}):
$$\begin{array}{lll}
 A_i^1=(a_0,a_1],&  \ldots,& A_i^{m_0}=(a_{m_0-1},a_{m_0}],\\[6pt]
 A_j^1=(a_0,a_1],&  \ldots,& A_j^{m_0}=(a_{m_0-1},a_{m_0}],\\[6pt]
 A_k^1=(a_0,a_1],&  \ldots,& A_k^{m_0}=(a_{m_0-1},a_{m_0}],
  \end{array}
$$
where
$$a_0=-\infty, \quad a_1=\Phi^{-1}\left({1\over m_0}\right), \quad a_2=\Phi^{-1}\left({2\over m_0}\right), \quad \ldots,\quad a_{m_0}=\infty.$$
\smallskip\par
Denoting ($A_i^{s_i}, A_j^{s_j}, A_k^{s_k}$ are on coordinates $i,j,k$)
$$A_{ijk}^{s_i,s_j,s_k}=\mathbb{R}\times \mathbb{R} \times \ldots A_i^{s_i} \times \mathbb{R} \ldots \times A_i^{s_k} \times \mathbb{R} \ldots \times A_k^{s_k} \times \mathbb{R}\ldots \times \mathbb{R}$$
we have $\Prob(\mathbf{Z}\in A_{ijk}^{s_i,s_j,s_k})=1/(m_0^3)=1/m$. In other words, we stratified $R^d$ into
disjoint equally probable regions
$$\mathbb{R}^d=\bigcup_{(s_i, s_j, s_k)\in \{1,\ldots,m_0\}^d} A_{ijk}^{s_i,s_j,s_k}.$$
Sampling $\mathbf{Z}=(Z_1,\ldots,Z_d)$ from strata $A_{ijk}^{s_i,s_j,s_k}$
is straightforward: we sample $Z_t\sim \mathcal{N}(0,1)$ for $t\notin\{i,j,k\}$ and sample $Z_i$ from $A_i^{s_i}$,
$Z_j$ from $A_j^{s_j}$ and $Z_k$ from $A_k^{s_k}$ as described in section \ref{sec:method1}.
For example, to sample $Z_i$ from  $A_i^{s_i}$ we sample $U\sim\mathcal{U}(0,1)$ and set
$$Z_i = \Phi^{-1}(V^{s_i}), \quad \textrm{where} \quad  V^{s_i}=a_{s_i-1} + (a_{s_i}-a_{s_i-1}) U.$$

The way of choosing $i,j,k$ determines the method.
\medskip\par
\begin{itemize}[leftmargin=0.5cm,align=left, itemsep=0.2cm]
 \item \textbf{Method }$\textrm{M}\texttt{Rand3}$. We simply choose (different) $i,j,k$ uniformly at random from $\{1,\ldots,d\}$.

 \item \textbf{Method }$\textrm{M}\texttt{High3}$.
 Here, the choice depends on a final function we are to estimate ($I=\Exp f(\mathbf{X})$).
 First, we separately consider stratification w.r.t to one dimension only, i.e., we fix $b$ and  consider stratification
 $$\mathbb{R}^d=\bigcup_{s=1}^{m_0} A_{b,s}, \quad \textrm{where} \quad A_{b,s}=\mathbb{R}\times\ldots\times \mathbb{R} \times \underbrace{A_b^s}_{b\textrm{-th dimension}}\times \mathbb{R}\times\ldots\times \mathbb{R}.$$
 Here $A_i^s=(a_{s-1}, a_s)$.
 \smallskip\par
 Now we perform $R_0<R$ pilot simulations and use stratification with proportional allocation
 to obtain the estimator $\hat{Y}_{R_0}^{\mathcal F, \textrm{dim}=b}$. Denote its standard deviation
 by $\sigma^{\textrm{dim}=b}=\sqrt{\Var\hat{Y}_{R_0}^{\mathcal F, \textrm{dim}=b}}$.
 Now we consider three different $b$'s for which $\sigma^{\textrm{dim}=b}$ were largest -- we take these
 as $i,j,k$.
\end{itemize}

Note that $\textrm{M}\texttt{rad}$ and $\textrm{M}\texttt{Rand3}$ do not depend on $f$, whereas
$\textrm{M}\texttt{High3}$ does. All our experiments were done for $\eta=3$ chosen coordinates,
of course it may be other number -- it must be relatively small so that total number of strata $m_0^\eta$ is not prohibitive.
Recall also that in all experiments we present averages from 10 simulations. In particular,
for method $\textrm{M}\texttt{Rand3}$, in each simulation $i,j,k$ were sampled independently.
%\title{Reducing Estimation Uncertainty \\ Using  Normalizing Flows and Stratification}

\section{More details on experimental results}

\subsubsection*{Model architecture and training details.}
The CNF component for our model was inspired by FFJORD. Our implementation is based on the original code provided by the authors\footnote{\url{https://github.com/rtqichen/ffjord}}. We use two stacked blocks of CNFs, each composed of two
hidden concatsquash layers, 16 units (\texttt{dims}) (Example \examplesecond{} and \examplemultit{}), 64 units (Example \examplefirst{} and \examplereal{}), 256 units (Example \examplesdf{} and \examplelargestd{}) or 128 units (Example \examplelarged{})
each, with $\tanh({\cdot})$ activation. After each of the units, we apply the batch-norm operation. We train the model using an ADAM optimizer with a learning rate
ranging from  $0.0001$ to $0.01$ and  number of epochs ranging from several hundreds to  $5000$ (depending on example, details can be checked in code).

\noindent\textbf{Remark.}
In all tables presented in this appendix that consist of parts (a) and (b),
the highlighting conventions (\textbf{bold} for the best value and \underline{underlined} for the second best)
are applied jointly across both parts, i.e., comparisons are made between all rows appearing in (a) and (b).

\subsection{Example \examplefirst{}}\label{sec:app_Example1}

We provide here more numerical results.
Recall, in Table \ref{tab:Example2d_NO_cart} we provided results for cartesian stratification,
for selected model  -- in Table \ref{tab:Example2d_NO_cart_full} we provide results
for more models including both, \texttt{AC} and \texttt{SD${}^*$}.
Moreover, in Table \ref{tab:Example2d_NO_spher} we provide results for spherical stratification.

\begin{table*}[h]
\caption{(Example 1) Numerical results for \textbf{cartesian} (method M1) stratification.}
\label{tab:Example2d_NO_cart_full}
\begin{subtable}{\textwidth}
\centering
{\small
\begin{tabular}{|p{.45cm}|p{0.75cm}|p{.7cm}p{.7cm}p{.85cm}|p{.45cm}|p{.73cm}p{.73cm}p{.85cm}|p{.73cm}p{.73cm}p{.85cm}|p{.73cm}p{.73cm}p{.85cm}|}\hline
% \begin{tabular}{|p{.35cm}|p{0.7cm}|p{.7cm}p{.7cm}p{.8cm}|p{.3cm}|p{.7cm}p{.7cm}p{.8cm}|p{.7cm}p{.7cm}p{.8cm}|p{.7cm}p{.7cm}p{.8cm}|p{.7cm}p{.7cm}p{.8cm}|p{.7cm}p{.7cm}p{.8cm}|}\hline
  $f$ &   $I$& \multicolumn{3}{c|}{ \small{$\hat{Y}_{1000}^{\rm obs}$}}& $R$
    &\multicolumn{3}{c|}{\small{$\hat{Y}^{\F, \rm CMC}_R$}}&
     \multicolumn{3}{c|}{\small{$\hat{Y}^{\F, \rm  prop, M1}_{R,m=2\times2}$}}&
  \multicolumn{3}{c|}{\small{$\hat{Y}^{\F, \rm  opt, M1}_{R,m=2\times2}$}}\\
    \hline
           & &
              \texttt{EST} & \texttt{SD}${}^*$ & \texttt{ACC} &
                   &
             \texttt{EST} & \texttt{SD}${}^*$ & \texttt{ACC}&
              \texttt{EST} & \texttt{SD}${}^*$ &\texttt{ACC}&
              \texttt{EST} &\texttt{SD}${}^*$& \texttt{ACC}
    \\
    \hline
                \multirow{2}{*}{$j^{+}_{0.5}$} & \multirow{2}{*}{.3149} & \multirow{2}{*}{.3130} & \multirow{2}{*}{1.466}  & \multirow{2}{*}{2.216} &   $2^{12}$
                & .3037 &  .7184  & 1.683 % CMC
                & .3080 &  .6486  & \textbf{1.795} % Prp m 4
                & .3023 &  .6199  & 1.447 % Opt m 4
           \\
           &   &   &   &   &   $2^{15}$
                & .3070 &  .2548  & \underline{1.633} % CMC
                & .3065 &  .2295  & 1.576 % Prp m 4
                & .3055 &  .2219  & 1.537 % Opt m 4
           \\ \hline
           \multirow{2}{*}{$j^{+}_{1.2}$} & \multirow{2}{*}{.0324} & \multirow{2}{*}{.0270} & \multirow{2}{*}{.5126}  & \multirow{2}{*}{.7757} &   $2^{12}$
                & .0314 &  .2722  & 1.247 % CMC
                & .0323 &  .2655  & 1.360 % Prp m 4
                & .0316 &  .1674  & \underline{1.626} % Opt m 4
           \\
           &   &   &   &   &   $2^{15}$
                & .0323 &  .0976  & 1.809 % CMC
                & .0322 &  .0938  & 1.837 % Prp m 4
                & .0320 &  .0600  & \underline{2.144} % Opt m 4
           \\ \hline
             \multirow{2}{*}{$j^{+}_{2.0}$} & \multirow{2}{*}{.0008} & \multirow{2}{*}{.0000} & \multirow{2}{*}{.0000}  & \multirow{2}{*}{.0000} &   $2^{12}$
                & .0008 &  .0443  & \underline{.6127} % CMC
                & .0011 &  .0498  & .4672 % Prp m 4
                & .0012 &  .0271  & .5694 % Opt m 4
           \\
           &   &   &   &   &   $2^{15}$
                & .0010 &  .0178  & .6684 % CMC
                & .0011 &  .0179  & .6904 % Prp m 4
                & .0010 &  .0102  & \underline{.7675} % Opt m 4
           \\ \hline
              \multirow{2}{*}{$h_1$} & \multirow{2}{*}{.0333} & \multirow{2}{*}{.0277} & \multirow{2}{*}{.6144}  & \multirow{2}{*}{.7722} &   $2^{12}$
                & .0314 &  .3102  & 1.408 % CMC
                & .0337 &  .3025  & 1.363 % Prp m 4
                & .0328 &  .1878  & 1.447 % Opt m 4
           \\
           &   &   &   &   &   $2^{15}$
                & .0320 &  .1109  & \underline{1.507} % CMC
                & .0316 &  .1057  & 1.330 % Prp m 4
                & .0321 &  .0671  & 1.456 % Opt m 4
           \\ \hline
           \multirow{2}{*}{$h_2$} & \multirow{2}{*}{.0318} & \multirow{2}{*}{.0275} & \multirow{2}{*}{.3722}  & \multirow{2}{*}{.8679} &   $2^{12}$
                & .0303 &  .1937  & 1.349 % CMC
                & .0317 &  .1837  & 1.392 % Prp m 4
                & .0310 &  .1083  & \underline{1.593} % Opt m 4
           \\
           &   &   &   &   &   $2^{15}$
                & .0307 &  .0692  & 1.504 % CMC
                & .0306 &  .0638  & 1.488 % Prp m 4
                & .0309 &  .0383  & \textbf{1.597} % Opt m 4
           \\ \hline
            \multirow{2}{*}{$h_3$} & \multirow{2}{*}{.1230} & \multirow{2}{*}{.1014} & \multirow{2}{*}{1.718}  & \multirow{2}{*}{.7548} &   $2^{12}$
                & .1165 &  1.072  & 1.100 % CMC
                & .1255 &  1.236  & 1.091 % Prp m 4
                & .1203 &  .6756  & \underline{1.461} % Opt m 4
           \\
           &   &   &   &   &   $2^{15}$
                & .1222 &  .6208  & 1.537 % CMC
                & .1199 &  .4325  & 1.613 % Prp m 4
                & .1199 &  .2767  & 1.622 % Opt m 4
           \\ \hline
\end{tabular}

}\vspace{0.4em}
\caption*{(a) Results for $m=2\times2$}
\end{subtable}

\vspace{0.6em}

\begin{subtable}{\textwidth}
\centering
{
% \begin{tabular}{|p{.35cm}|p{0.7cm}|p{.7cm}p{.7cm}p{.8cm}|p{.3cm}|p{.7cm}p{.7cm}p{.8cm}|p{.7cm}p{.7cm}p{.8cm}|p{.7cm}p{.7cm}p{.8cm}|p{.7cm}p{.7cm}p{.8cm}|p{.7cm}p{.7cm}p{.8cm}|}\hline
\begin{tabular}{|p{.45cm}|p{0.75cm}|p{.7cm}p{.7cm}p{.85cm}|p{.45cm}|p{.73cm}p{.73cm}p{.85cm}|p{.73cm}p{.73cm}p{.85cm}|p{.73cm}p{.73cm}p{.85cm}|}\hline
  $f$ &   $I$& \multicolumn{3}{c|}{ \small{$\hat{Y}_{1000}^{\rm obs}$}}& $R$
    &\multicolumn{3}{c|}{\small{$\hat{Y}^{\F, \rm CMC}_R$}}&
        \multicolumn{3}{c|}{\small{$\hat{Y}^{\F, \rm  prop, M1}_{R,m=4\times4}$}}&
  \multicolumn{3}{c|}{\small{$\hat{Y}^{\F, \rm  opt, M1}_{R,m=4\times4}$}}\\
    \hline
           & &
              \texttt{EST} & \texttt{SD}${}^*$ & \texttt{ACC} &
                   &
             \texttt{EST} & \texttt{SD}${}^*$ & \texttt{ACC}&
              \texttt{EST} & \texttt{SD}${}^*$ &\texttt{ACC}&
              \texttt{EST} &\texttt{SD}${}^*$& \texttt{ACC}
    \\
    \hline
                \multirow{2}{*}{$j^{+}_{0.5}$} & \multirow{2}{*}{.3149} & \multirow{2}{*}{.3130} & \multirow{2}{*}{1.466}  & \multirow{2}{*}{2.216} &   $2^{12}$
                & .3037 &  .7184  & 1.683 % CMC
                & .3080 &  .4303  & \underline{1.727} % Prp m 16
                & .3054 &  .3123  & 1.526 % Opt m 16
           \\
           &   &   &   &   &   $2^{15}$
                & .3070 &  .2548  & \underline{1.633} % CMC
                & .3078 &  .1525  & \textbf{1.651} % Prp m 16
                & .3064 &  .1112  & 1.568 % Opt m 16
           \\ \hline
           \multirow{2}{*}{$j^{+}_{1.2}$} & \multirow{2}{*}{.0324} & \multirow{2}{*}{.0270} & \multirow{2}{*}{.5126}  & \multirow{2}{*}{.7757} &   $2^{12}$
                & .0314 &  .2722  & 1.247 % CMC
                & .0324 &  .2153  & 1.567 % Prp m 16
                & .0324 &  .0726  & \textbf{1.819} % Opt m 16
           \\
           &   &   &   &   &   $2^{15}$
                & .0323 &  .0976  & 1.809 % CMC
                & .0324 &  .0760  & 1.846 % Prp m 16
                & .0322 &  .0259  & \textbf{2.216} % Opt m 16
           \\ \hline
             \multirow{2}{*}{$j^{+}_{2.0}$} & \multirow{2}{*}{.0008} & \multirow{2}{*}{.0000} & \multirow{2}{*}{.0000}  & \multirow{2}{*}{.0000} &   $2^{12}$
                & .0008 &  .0443  & \underline{.6127} % CMC
                & .0009 &  .0460  & .5850 % Prp m 16
                & .0009 &  .0128  & \textbf{1.056} % Opt m 16
           \\
           &   &   &   &   &   $2^{15}$
                & .0010 &  .0178  & .6684 % CMC
                & .0010 &  .0174  & \textbf{.9000} % Prp m 16
                & .0010 &  .0052  & .6933 % Opt m 16
           \\ \hline
              \multirow{2}{*}{$h_1$} & \multirow{2}{*}{.0333} & \multirow{2}{*}{.0277} & \multirow{2}{*}{.6144}  & \multirow{2}{*}{.7722} &   $2^{12}$
                & .0314 &  .3102  & 1.408 % CMC
                & .0330 &  .2634  & \underline{1.538} % Prp m 16
                & .0323 &  .1151  & \textbf{1.548} % Opt m 16
           \\
           &   &   &   &   &   $2^{15}$
                & .0320 &  .1109  & \underline{1.507} % CMC
                & .0323 &  .0945  & \textbf{1.548} % Prp m 16
                & .0321 &  .0415  & 1.463 % Opt m 16
           \\ \hline

           \multirow{2}{*}{$h_2$} & \multirow{2}{*}{.0318} & \multirow{2}{*}{.0275} & \multirow{2}{*}{.3722}  & \multirow{2}{*}{.8679} &   $2^{12}$
                & .0303 &  .1937  & 1.349 % CMC
                & .0310 &  .1430  & \textbf{1.829} % Prp m 16
                & .0307 &  .0599  & 1.501 % Opt m 16
           \\
           &   &   &   &   &   $2^{15}$
                & .0307 &  .0692  & 1.504 % CMC
                & .0308 &  .0507  & \underline{1.549} % Prp m 16
                & .0307 &  .0212  & 1.471 % Opt m 16
           \\ \hline
            \multirow{2}{*}{$h_3$} & \multirow{2}{*}{.1230} & \multirow{2}{*}{.1014} & \multirow{2}{*}{1.718}  & \multirow{2}{*}{.7548} &   $2^{12}$
                & .1165 &  1.072  & 1.100 % CMC
                & .1216 &  1.077  & 1.346 % Prp m 16
                & .1198 &  .3774  & \textbf{1.627} % Opt m 16
           \\
           &   &   &   &   &   $2^{15}$
                & .1222 &  .6208  & 1.537 % CMC
                & .1211 &  .4327  & \underline{1.656} % Prp m 16
                & .1205 &  .1789  & \textbf{1.829} % Opt m 16
           \\ \hline

\end{tabular}
}\vspace{0.4em}
\caption*{(b) Results for $m=4\times4$}
\end{subtable}
\end{table*}

\begin{table*}[h]
  \caption{(Example \examplefirst{}) Numerical results for \textbf{spherical} (method M2) stratification. }
\label{tab:Example2d_NO_spher}
\begin{subtable}{\textwidth}
\centering
{
\begin{tabular}{|p{.45cm}|p{0.75cm}|p{.7cm}p{.7cm}p{.85cm}|p{.45cm}|p{.73cm}p{.73cm}p{.85cm}|p{.73cm}p{.73cm}p{.85cm}|p{.73cm}p{.73cm}p{.85cm}|}\hline
 $f$ &   $I$& \multicolumn{3}{c|}{ \small{$\hat{Y}_{1000}^{\rm obs}$}}& $R$

    &\multicolumn{3}{c|}{\small{$\hat{Y}^{\F, \rm CMC}_R$}}&
     \multicolumn{3}{c|}{\small{$\hat{Y}^{\F, \rm  prop, M2}_{R,m=2\times2}$}}&
  \multicolumn{3}{c|}{\small{$\hat{Y}^{\F, \rm  opt, M2}_{R,m=2\times2}$}}
        \\
    \hline
           & &
              \texttt{E} & \texttt{SD}${}^*$ & \texttt{AC} &
                   &
             \texttt{E} & \texttt{SD}${}^*$ & \texttt{AC}&
              \texttt{E} & \texttt{SD}${}^*$ &\texttt{AC}&
              \texttt{E} &\texttt{SD}${}^*$& \texttt{AC}
    \\
    \hline
                 \multirow{2}{*}{$j^+_{0.5}$} & \multirow{2}{*}{.3149} & \multirow{2}{*}{.3130} & \multirow{2}{*}{1.466}  & \multirow{2}{*}{2.216} &   $2^{12}$
                & .3037 &  .7184  & 1.683 % CMC
                & .3074 &  .6622  & 1.688 % Prp m 4
                & .3041 &  .6573  & 1.523 % Opt m 4
           \\
           &   &   &   &   &   $2^{15}$
                & .3070 &  .2548  & 1.633 % CMC
                & .3076 &  .2341  & \underline{1.654} % Prp m 4
                & .3069 &  .2330  & 1.619 % Opt m 4
           \\ \hline
               \multirow{2}{*}{$j^+_{1.2}$} & \multirow{2}{*}{.0324} & \multirow{2}{*}{.0270} & \multirow{2}{*}{.5126}  & \multirow{2}{*}{.7757} &   $2^{12}$
                & .0314 &  .2722  & 1.247 % CMC
                & .0327 &  .2715  & 1.227 % Prp m 4
                & .0330 &  .2210  & 1.301 % Opt m 4
           \\
           &   &   &   &   &   $2^{15}$
                & .0323 &  .0976  & 1.809 % CMC
                & .0325 &  .0961  & \underline{1.889} % Prp m 4
                & .0322 &  .0777  & \textbf{1.992} % Opt m 4
           \\ \hline

            \multirow{2}{*}{$j^+_{2.0}$} & \multirow{2}{*}{.0008} & \multirow{2}{*}{.0000} & \multirow{2}{*}{.0000}  & \multirow{2}{*}{.0000} &   $2^{12}$
                & .0008 &  .0443  & .6127 % CMC
                & .0010 &  .0485  & .3751 % Prp m 4
                & .0009 &  .0235  & \underline{.6191} % Opt m 4
           \\
           &   &   &   &   &   $2^{15}$
                & .0010 &  .0178  & .6684 % CMC
                & .0010 &  .0173  & \textbf{.9075} % Prp m 4
                & .0011 &  .0103  & .5958 % Opt m 4
           \\ \hline

      \multirow{2}{*}{$h_1$} & \multirow{2}{*}{.0333} & \multirow{2}{*}{.0277} & \multirow{2}{*}{.6144}  & \multirow{2}{*}{.7722} &   $2^{12}$
                & .0314 &  .3102  & \underline{1.408} % CMC
                & .0305 &  .3046  & 1.279 % Prp m 4
                & .0306 &  .2568  & 1.091 % Opt m 4
           \\
           &   &   &   &   &   $2^{15}$
                & .0320 &  .1109  & 1.507 % CMC
                & .0322 &  .1082  & \textbf{1.580} % Prp m 4
                & .0321 &  .0916  & 1.531 % Opt m 4
           \\ \hline
                \multirow{2}{*}{$h_2$} & \multirow{2}{*}{.0318} & \multirow{2}{*}{.0275} & \multirow{2}{*}{.3722}  & \multirow{2}{*}{.8679} &   $2^{12}$
                & .0303 &  .1937  & 1.349 % CMC
                & .0303 &  .1881  & 1.371 % Prp m 4
                & .0304 &  .1481  & 1.285 % Opt m 4
           \\
           &   &   &   &   &   $2^{15}$
                & .0307 &  .0692  & \underline{1.504} % CMC
                & .0309 &  .0676  & \textbf{1.601} % Prp m 4
                & .0305 &  .0533  & 1.437 % Opt m 4
           \\ \hline
            \multirow{2}{*}{$h_3$} & \multirow{2}{*}{.1230} & \multirow{2}{*}{.1014} & \multirow{2}{*}{1.718}  & \multirow{2}{*}{.7548} &   $2^{12}$
                & .1165 &  1.072  & 1.100 % CMC
                & .1162 &  1.099  & 1.183 % Prp m 4
                & .1211 &  .7953  & 1.346 % Opt m 4
           \\
           &   &   &   &   &   $2^{15}$
                & .1222 &  .6208  & 1.537 % CMC
                & .1228 &  .5128  & 1.546 % Prp m 4
                & .1197 &  .3023  & \underline{1.679} % Opt m 4
           \\ \hline
\end{tabular}

}\vspace{0.4em}
\caption*{(a) Results for $m=2\times2$}
\end{subtable}

\vspace{0.6em}

\begin{subtable}{\textwidth}
\centering
{
\begin{tabular}{|p{.45cm}|p{0.75cm}|p{.7cm}p{.7cm}p{.85cm}|p{.45cm}|p{.73cm}p{.73cm}p{.85cm}|p{.73cm}p{.73cm}p{.85cm}|p{.73cm}p{.73cm}p{.85cm}|}\hline
   $f$ &   $I$& \multicolumn{3}{c|}{ \small{$\hat{Y}_{1000}^{\rm obs}$}}& $R$
    &\multicolumn{3}{c|}{\small{$\hat{Y}^{\F, \rm CMC}_R$}}&
        \multicolumn{3}{c|}{\small{$\hat{Y}^{\F, \rm  prop, M2}_{R,m=4\times4}$}}&
  \multicolumn{3}{c|}{\small{$\hat{Y}^{\F, \rm  opt, M2}_{R,m=4\times4}$}}\\
    \hline
           & &
              \texttt{E} & \texttt{SD}${}^*$ & \texttt{AC} &
                   &
             \texttt{E} & \texttt{SD}${}^*$ & \texttt{AC}&
              \texttt{E} & \texttt{SD}${}^*$ &\texttt{AC}&
              \texttt{E} &\texttt{SD}${}^*$& \texttt{AC}
    \\
    \hline
                 \multirow{2}{*}{$j^+_{0.5}$} & \multirow{2}{*}{.3149} & \multirow{2}{*}{.3130} & \multirow{2}{*}{1.466}  & \multirow{2}{*}{2.216} &   $2^{12}$
                & .3037 &  .7184  & 1.683 % CMC
                & .3080 &  .5567  & \textbf{1.877} % Prp m 16
                & .3074 &  .4569  & \underline{1.717} % Opt m 16
           \\
           &   &   &   &   &   $2^{15}$
                & .3070 &  .2548  & 1.633 % CMC
                & .3071 &  .1970  & 1.622 % Prp m 16
                & .3078 &  .1618  & \textbf{1.662} % Opt m 16
           \\ \hline
               \multirow{2}{*}{$j^+_{1.2}$} & \multirow{2}{*}{.0324} & \multirow{2}{*}{.0270} & \multirow{2}{*}{.5126}  & \multirow{2}{*}{.7757} &   $2^{12}$
                & .0314 &  .2722  & 1.247 % CMC
                & .0319 &  .2568  & \underline{1.477} % Prp m 16
                & .0322 &  .1308  & \textbf{1.547} % Opt m 16
           \\
           &   &   &   &   &   $2^{15}$
                & .0323 &  .0976  & 1.809 % CMC
                & .0321 &  .0912  & 1.745 % Prp m 16
                & .0322 &  .0467  & 1.880 % Opt m 16
           \\ \hline

            \multirow{2}{*}{$j^+_{2.0}$} & \multirow{2}{*}{.0008} & \multirow{2}{*}{.0000} & \multirow{2}{*}{.0000}  & \multirow{2}{*}{.0000} &   $2^{12}$
                & .0008 &  .0443  & .6127 % CMC
                & .0012 &  .0522  & .4157 % Prp m 16
                & .0010 &  .0123  & \textbf{.9785} % Opt m 16
           \\
           &   &   &   &   &   $2^{15}$
                & .0010 &  .0178  & .6684 % CMC
                & .0010 &  .0176  & \underline{.7266} % Prp m 16
                & .0010 &  .0049  & .6534 % Opt m 16
           \\ \hline

      \multirow{2}{*}{$h_1$} & \multirow{2}{*}{.0333} & \multirow{2}{*}{.0277} & \multirow{2}{*}{.6144}  & \multirow{2}{*}{.7722} &   $2^{12}$
                & .0314 &  .3102  & \underline{1.408} % CMC
                & .0307 &  .2826  & 1.235 % Prp m 16
                & .0322 &  .1665  & \textbf{1.524} % Opt m 16
           \\
           &   &   &   &   &   $2^{15}$
                & .0320 &  .1109  & 1.507 % CMC
                & .0316 &  .1006  & 1.355 % Prp m 16
                & .0321 &  .0585  & \underline{1.579} % Opt m 16
           \\ \hline
                \multirow{2}{*}{$h_2$} & \multirow{2}{*}{.0318} & \multirow{2}{*}{.0275} & \multirow{2}{*}{.3722}  & \multirow{2}{*}{.8679} &   $2^{12}$
                & .0303 &  .1937  & 1.349 % CMC
                & .0301 &  .1758  & \underline{1.507} % Prp m 16
                & .0311 &  .0993  & \textbf{1.572} % Opt m 16
           \\
           &   &   &   &   &   $2^{15}$
                & .0307 &  .0692  & \underline{1.504} % CMC
                & .0306 &  .0628  & 1.483 % Prp m 16
                & .0308 &  .0352  & 1.505 % Opt m 16
           \\ \hline
            \multirow{2}{*}{$h_3$} & \multirow{2}{*}{.1230} & \multirow{2}{*}{.1014} & \multirow{2}{*}{1.718}  & \multirow{2}{*}{.7548} &   $2^{12}$
                & .1165 &  1.072  & 1.100 % CMC
                & .1181 &  1.102  & \underline{1.418} % Prp m 16
                & .1181 &  .4772  & \textbf{1.620} % Opt m 16
           \\
           &   &   &   &   &   $2^{15}$
                & .1222 &  .6208  & 1.537 % CMC
                & .1184 &  .4030  & 1.403 % Prp m 16
                & .1204 &  .2099  & \textbf{1.739} % Opt m 16
           \\ \hline
\end{tabular}
}\vspace{0.4em}
\caption*{(b) Results for $m=4\times4$}
\end{subtable}
\end{table*}

\medskip\par
 To compare only stratification methods -- values of \texttt{AC} from Tables \ref{tab:Example2d_NO_cart} and \ref{tab:Example2d_NO_spher} (only for $R=2^{12}$) are collected in Table \ref{tab:Example2d_NO_spher_cart_R_12}.
 Note that there is no clear ``winner''. Out of 24 pairs (cartesian vs spherical) cartesian
 stratifiaction is better in 13 cases. However, considering only the best accuracy in each of
 rows (i.e., for each function), spherical stratification is better in 4 cases.
    \setlength\tabcolsep{6pt} % default value: 6pt
    \begin{table}[H]
    \vspace{0.2cm}
\caption{Example  \examplesecond{}: Comparison of accuracy \texttt{AC} of cartesian (M1) and spherical (M2)
stratification for $R=2^{12}$. Data from Tables \ref{tab:Example2d_NO_cart} and \ref{tab:Example2d_NO_spher}.
Best accuracy in each pair (M1, M2) is \underline{underlined}, best accuracy in a whole row is \textbf{bolded}.}
\label{tab:Example2d_NO_spher_cart_R_12}
\begin{center}
{
 \begin{tabular}{|p{0.34cm}|c|c|c|c|c|c|c|c|c|c|c|c|c|c|c|c|c|c|c|c|c}\hline
  $t$     & \multicolumn{2}{c|}{\small{$\hat{Y}^{\F, \rm  prop }_{R,m=2\times2}$}}
  & \multicolumn{2}{c|}{\small{$\hat{Y}^{\F, \rm  opt }_{R,m=2\times2}$}}
       & \multicolumn{2}{c|}{\small{$\hat{Y}^{\F, \rm  prop }_{R,m=4\times4}$}}
  & \multicolumn{2}{c|}{\small{$\hat{Y}^{\F, \rm  opt }_{R,m=4\times4}$}}\\
    \hline
     & M1 & M2 &M1 & M2 &M1 & M2 &M1 & M2 \\
     & {\scriptsize(cart)} & {\scriptsize(spher)}  &{\scriptsize(cart)}  & {\scriptsize(spher)}  &{\scriptsize(cart)}  & {\scriptsize(spher)} &{\scriptsize(cart)}  & {\scriptsize(spher)}  \\ [3pt]
     \hline
  $j^+_{0.5}$ &   1.576 &  \underline{\textbf{1.654}}   &  1.537    &   \underline{1.619}  &   \underline{1.651} &    1.622  &  1.568  &  \underline{1.662}    \\[3pt]
  $j^+_{1.2}$ &  1.837 & \underline{1.889}   &  \underline{\textbf{2.144}}  & 1.992   &  \underline{1.846}  &  1.745  &  \underline{2.216}  &  1.880   \\[3pt]
  $j^+_{2.0}$ &  .6904   &  \underline{\textbf{.9075}}  &  \underline{.7675}  &  .5958  & \underline{.9000}   & .7266   & .6933   &  .6534   \\ [3pt]

  $h_1$ &  1.330   &   \underline{\textbf{1.580}}  &  1.456  & \underline{1.531}   & \underline{{1.548}}   &  1.355   & 1.463 & \underline{1.579}   \\ [3pt]

  $h_2$ & 1.488   &  \underline{\textbf{1.601}}  &   \underline{1.597}   &  1.437  & \underline{1.549}   & 1.483   & 1.471   &  \underline{1.505}  \\ [3pt]

  $h_3$ &    \underline{1.613}  &  1.546  &  1.622  & \underline{1.679}   & \underline{1.656}    &
  1.403 & \underline{\textbf{1.829}} & 1.739 \\ [3pt]

  \hline
 \end{tabular}
  }
  \end{center}
\end{table}

\begin{figure*}
    \begin{minipage}{\textwidth}
 \begin{center}
  \begin{tabular}{ccccccccccc}

   &   {\small{$\hat{Y}^{\F, \rm CMC}_R$}}&
     {\small{$\hat{Y}^{\F, \rm  prop, M1}_{R,m=4}$}}&
  {\small{$\hat{Y}^{\F, \rm  opt, M1}_{R,m=4}$}}&
   {\small{$\hat{Y}^{\F, \rm  prop, M1}_{R,m=16}$}}&
  {\small{$\hat{Y}^{\F, \rm  opt, M1}_{R,m=16}$}}\\
  & \includegraphics[width=0.15\textwidth]{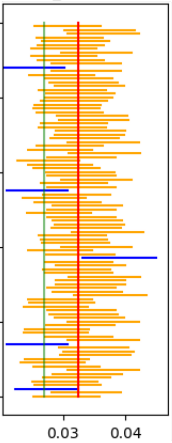}
  &
  \includegraphics[width=0.15\textwidth]{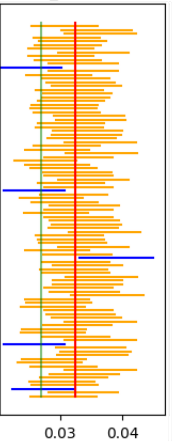}
  & \includegraphics[width=0.15\textwidth]{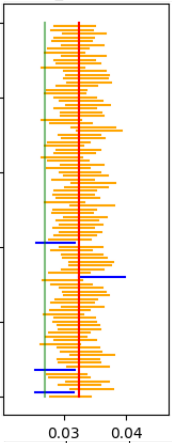}
  & \includegraphics[width=0.15\textwidth]{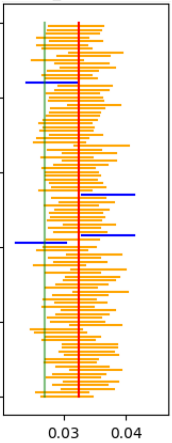}
  & \includegraphics[width=0.15\textwidth]{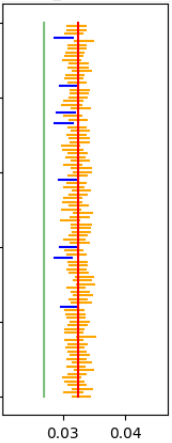}
  \\[5pt]
  \small{$A$:}& $5\%$ & $7\%$ & $4\%$ & $4\%$ & $8\%$ \\[5pt]
   \small{$B$:} & .0108 & .0104 & .0067 &  0.0084 & 0.0029
  \end{tabular}
  \end{center}
\caption{Example \examplesecond{}: 100   estimations of $I=\Prob(X_1>1.2, X_2>1.2)$, each from $R=2^{12}$  simulations.
95$\%$ confidence intervals depicted. \textcolor{red}{Red} line -- true $I$,
\textcolor{green}{green} line -- $\hat{Y}_n^{\rm obs}$ , \textcolor{orange}{orange} lines -- intervals containing $I$; \textcolor{blue}{blue} lines: those not containing $I$. $A$: percentage of intervals not containing $I$; $B$: average confidence interval length}\label{fig:ExNO2d_cartesian_lines}.
    \end{minipage}
    \hfill
    \begin{minipage}{\textwidth}
\begin{center}
  \begin{tabular}{ccccccccccc}

   &   {\small{$\hat{Y}^{\F, \rm CMC}_R$}}&
     {\small{$\hat{Y}^{\F, \rm  prop, M2}_{R,m=4}$}}&
  {\small{$\hat{Y}^{\F, \rm  opt, M2}_{R,m=4}$}}&
   {\small{$\hat{Y}^{\F, \rm  prop, M2}_{R,m=16}$}}&
  {\small{$\hat{Y}^{\F, \rm  opt, M2}_{R,m=16}$}}\\
  & \includegraphics[width=0.15\textwidth]{figures/ExampleNO2d_intervals_R4096_gt1_2_sample.png}
  &
  \includegraphics[width=0.15\textwidth]{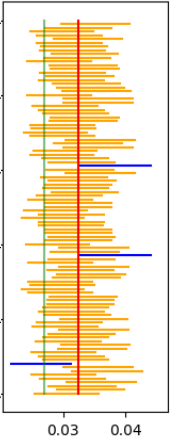}
  & \includegraphics[width=0.15\textwidth]{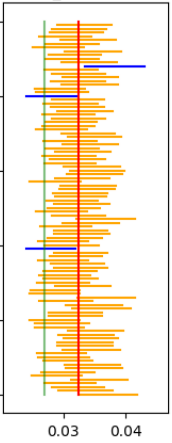}
  & \includegraphics[width=0.15\textwidth]{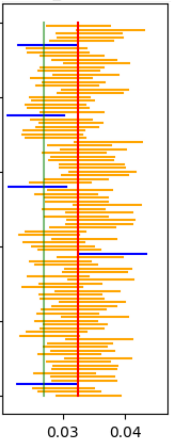}
  & \includegraphics[width=0.15\textwidth]{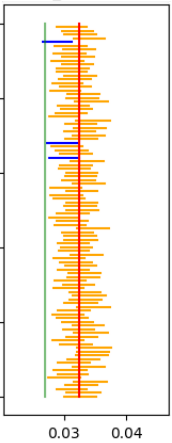}
  \\[5pt]
 \small{$A$:}& $5\%$ & $3\%$ & $3\%$ & $5\%$ & $3\%$ \\[5pt]
  \small{$B$:} & .0108 & .0106 & .0086 &  0.0101 & 0.0051
  \end{tabular}
  \end{center}
\caption{Example \examplesecond{}: $R=2^{12}$, spherical stratification, 100 hundred estimations of $I=\Prob(X_1>1.2, X_2>1.2)$.
Confidence interval (at significance lever $\alpha=5\%$) depicted. \textcolor{red}{Red} line -- true $I$,
\textcolor{green}{green} line -- $\hat{Y}_n^{\rm obs}$, \textcolor{orange}{orange} lines -- intervals containing $I$; \textcolor{blue}{blue} lines: those not containing $I$. $A$: percentage of intervals not containing $I$; $B$: average confidence interval length}\label{fig:ExNO2d_spherical_lines}
    \end{minipage}
\end{figure*}
 In Fig. \ref{fig:ExNO2d_cartesian_lines_short} we presented confidence intervals resulting
 from  100 iterations of the whole procedure (recall, in all the tables average of 10 simulations is reported)
 for estimating $I=\Prob(X_1>1.2, X_2>1.2$) (i.e., $\Exp j^+_{1.2}(\X)$) using  $R=2^{12}$ replications,
 we presented plots only for $\hat{Y}_R^{\mathcal{F}, \rm CMC}$ and
 $\hat{Y}_{R,m=16}^{\mathcal{F}, \rm opt, M1}$ (i.e., cartesian stratification).
  Fig. \ref{fig:ExNO2d_cartesian_lines}  is an extended version of Fig. \ref{fig:ExNO2d_cartesian_lines_short},
 it also contains   $\hat{Y}_{R,m=4}^{\mathcal{F}, \rm prop, M1}$,
 $\hat{Y}_{R,m=4}^{\mathcal{F}, \rm opt, M1}$ and $\hat{Y}_{R,m=16}^{\mathcal{F}, \rm prop, M1}$ estimators.
 Fig. \ref{fig:ExNO2d_spherical_lines} contains similar confidence intervals computed from
 simulations involving spherical stratification (note that $\hat{Y}_R^{\mathcal{F}, \rm CMC}$ used
 the same randomness in both cases).
 To be more precise, for $i=1,\ldots,100$ we perform the sampling procedure, estimate
 the value of $I$ via $\hat{Y}^{(i)}_R$ (one of: CMC or some stratified estimator) and compute $\Var(\hat{Y}^{(i)}_R)$. Afterwards, we draw
  intervals
 $$\left[\hat{Y}^{(i)}_R-z_{1-\alpha/2}\Var(\hat{Y}^{(i)}_R),\ \hat{Y}^{(i)}_R+z_{1-\alpha/2}\Var(\hat{Y}^{(i)}_R)\right].$$
 The theory states that
 $$\Prob\left(I\in\left[\hat{Y}^{(i)}_R-z_{1-\alpha/2}\Var(\hat{Y}^{(i)}_R),\ \hat{Y}^{(i)}_R+z_{1-\alpha/2}\Var(\hat{Y}^{(i)}_R)\right]\right)\approx 1-\alpha.$$
 All the figures are for $\alpha=5\%$, which means that $5\%$ of confidence intervals does not contain (on average)  a true value of $I$. More precisely, number of intervals not containig $I$ has
 bionomial distribution with parameters $n=100$ and $p=0.05$. The probability that we will actualy have
 $\{3,\ldots,8\}$ such intervals is quite likely, it has probability 0.81. In all the figures
 we have between 3 and 8 such intervals. It all shows that the method works,
 thus the best method can be considered the one with shortest confidence interval --
 which is the cartesian stratified estimator with optimal allocation and $m=16$ strata (Fig.
 \ref{fig:ExNO2d_cartesian_lines}, right most) and stratified estimator with
 spherical stratified estimator with also optimal allocation and $m=16$ strata (Fig. \ref{fig:ExNO2d_spherical_lines}, right most) being the second best. Summarizing, using
 stratified sampling results in the smallest uncertainty of estimation.

\subsection{Example \examplesecond{}}\label{sec:app_Example2}
Random variable $X$ is a~mixture of a uniform distribution $\mathcal{U}(0.2,0.4)$,
a   normal $\mathcal{N}(0.6,0.0046)$  distribution and a beta distribution with parameters
$(\alpha,\beta)=(7, 1.1)$ with weights $1/2, 1/4$ and  $1/4$ correspondingly.
We train a flow model on sampled $n=1000$ points for $5000$ epochs, the aim is
to estimate $\Prob(X\leq 0.95), \Prob(X\leq 0.99)$ and $\Prob(X> 0.95)$,
as well as expectations of functions
$\rho_1(x)=\sin(e^x), \rho_2(x)=\log(1+|x|), \rho_3(x)=1/\log(1+|x|).$
We estimate it using
trained flow model by sampling $R=2^{15}$ replications and using
either CMC or stratified estimators (both, proportional and optimal) with $m\in\{4,8\}$ strata.
In Table \ref{tab:Example1d_full} we present results for more models and functions and we report also \texttt{AC}
\begin{table*}[t]
 \caption{ (Example 2)   Numerical results ($R=2^{15}$) for \textbf{cartesian} (method M1) stratification.}\label{tab:Example1d_full}
\begin{subtable}{\textwidth}
\centering
{
\begin{tabular}{|
@{\hspace{2pt}}c@{\hspace{2pt}}|
@{\hspace{2pt}}c@{\hspace{2pt}}|
                p{0.55cm}p{0.55cm}p{0.65cm}|
                p{0.55cm}p{0.55cm}p{0.65cm}|
                p{0.55cm}p{0.55cm}p{0.65cm}|
                p{0.55cm}p{0.55cm}p{0.65cm}|}\hline
     $f$ & $I$& \multicolumn{3}{c|}{ \small{$\hat{Y}_{1000}^{\rm obs}$}}
       &\multicolumn{3}{c|}{\small{$\hat{Y}^{\F, \rm CMC}_R$}}&
      \multicolumn{3}{c|}{\small{$\hat{Y}^{\F, \rm  prop, M1}_{R,m=4}$}}&
   \multicolumn{3}{c|}{\small{$\hat{Y}^{\F, \rm  opt, M1}_{R,m=4}$}} \\
     \hline

         & &\texttt{EST}& \texttt{STD}${}^*$ & \texttt{ACC} &
              \texttt{EST} & \texttt{STD}${}^*$ & \texttt{ACC}&
               \texttt{EST} & \texttt{STD}${}^*$ &\texttt{ACC}&
               \texttt{EST} &\texttt{STD}${}^*$& \texttt{ACC}
     \\
      \hline
      $j^{-}_{0.95}$  & .9351 & .8730 & 1.053  & 1.177
                     & .8788 &  .1803  & 1.221 % CMC
                     & .8793 &  .1380  & 1.224 % Prp m 4
                     & .8797 &  .0690  & \textbf{1.227} % Opt m 4
                \\ \hline
         $j^+_{0.95}$ & .0649 & .1270 & 1.053  & 0.019
                     & .1212 &  .1803  & 0.062 % CMC
                     & .1207 &  .1380  & 0.065 % Prp m 4
                     & .1206 &  .0690  & 0.066 % Opt m 4
                \\ \hline
    $j^{-}_{0.99}$ &   .9875 &  .9640 & .5891  & 1.623
                     & .9729 &  .0897  & 1.831 % CMC
                     & .9731 &  .0856  & \textbf{1.837} % Prp m 4
                     & .9728 &  .0430  & 1.829 % Opt m 4
                 \\ \hline
     $\rho_1$ & .8919 & .8149 & .6144  & 1.063
                     & .8127 &  .1076  & 1.051 % CMC
                     & .8131 &  .0358  & 1.054 % Prp m 4
                     & .8133 &  .0268  & \underline{1.055} % Opt m 4
                \\ \hline

    $\rho_2$ & .4032 & .4405 & .5999  & 1.034
                     & .4403 &  .1054  & 1.036 % CMC
                     & .4401 &  .0211  & 1.038 % Prp m 4
                     & .4399 &  .0178  & \underline{1.040} % Opt m 4
    \\ \hline
     $\rho_3$ & 2.917 & 2.805 & 4.064  & 1.413
                     & 2.810 &  .7114  & 1.438 % CMC
                     & 2.811 &  .1670  & 1.438 % Prp m 4
                     & 2.811 &  .1446  & \underline{1.438} % Opt m 4
    \\ \hline
     \end{tabular}

}\vspace{0.4em}
\caption*{(a) Results for $m=4$.}
\end{subtable}

\vspace{0.6em}

\begin{subtable}{\textwidth}
\centering
{
\begin{tabular}{|
@{\hspace{2pt}}c@{\hspace{2pt}}|
@{\hspace{2pt}}c@{\hspace{2pt}}|
                p{0.55cm}p{0.55cm}p{0.65cm}|
                p{0.55cm}p{0.55cm}p{0.65cm}|
                p{0.55cm}p{0.55cm}p{0.65cm}|
                p{0.55cm}p{0.55cm}p{0.65cm}|}\hline
     $f$ & $I$& \multicolumn{3}{c|}{ \small{$\hat{Y}_{1000}^{\rm obs}$}}
       &\multicolumn{3}{c|}{\small{$\hat{Y}^{\F, \rm CMC}_R$}}&
         \multicolumn{3}{c|}{\small{$\hat{Y}^{\F, \rm  prop, M1}_{R,m=8}$}}&
   \multicolumn{3}{c|}{\small{$\hat{Y}^{\F, \rm  opt, M1}_{R,m=8}$}}\\
     \hline

         & &\texttt{EST}& \texttt{STD}${}^*$ & \texttt{ACC} &
              \texttt{EST} & \texttt{STD}${}^*$ & \texttt{ACC}&
               \texttt{EST} & \texttt{STD}${}^*$ &\texttt{ACC}&
               \texttt{EST} &\texttt{STD}${}^*$& \texttt{ACC}
     \\
      \hline
      $j^{-}_{0.95}$  & .9351 & .8730 & 1.053  & 1.177
                     & .8788 &  .1803  & 1.221 % CMC
                     & .8795 &  .0364  & 1.226 % Prp m 16
                     & .8796 &  .0130  & \underline{1.226} % Opt m 16
                \\ \hline
         $j^+_{0.95}$ & .0649 & .1270 & 1.053  & 0.019
                     & .1212 &  .1803  & 0.062 % CMC
                     & .1205 &  .0364  & \textbf{0.067} % Prp m 16
                     & .1205 &  .0128  & \underline{0.066} % Opt m 16
                \\ \hline
    $j^{-}_{0.99}$ &   .9875 &  .9640 & .5891  & 1.623
                     & .9729 &  .0897  & 1.831 % CMC
                     & .9730 &  .0804  & 1.833 % Prp m 16
                     & .9731 &  .0284  & \underline{1.836} % Opt m 16
                 \\ \hline
     $\rho_1$ & .8919 & .8149 & .6144  & 1.063
                     & .8127 &  .1076  & 1.051 % CMC
                     & .8134 &  .0196  & \textbf{1.055} % Prp m 16
                     & .8133 &  .0140  & 1.054 % Opt m 16
                \\ \hline

    $\rho_2$ & .4032 & .4405 & .5999  & 1.034
                     & .4403 &  .1054  & 1.036 % CMC
                     & .4399 &  .0143  & \textbf{1.040} % Prp m 16
                     & .4400 &  .0098  & 1.039 % Opt m 16
    \\ \hline
     $\rho_3$ & 2.917 & 2.805 & 4.064  & 1.413
                     & 2.810 &  .7114  & 1.438 % CMC
                     & 2.811 &  .1035  & \textbf{1.441} % Prp m 16
                     & 2.811 &  .0799  & 1.438 % Opt m 16
    \\ \hline
     \end{tabular}
}\vspace{0.4em}
\caption*{(b) Results for $m=8$.}
\end{subtable}
\end{table*}

\smallskip\par
Moreover, in Tables \ref{tab:ex1_n500} and \ref{tab:ex1_n2000} we present
results for Example \examplesecond{}, where the flow model was trained on $500$ and $2000$ points respectively.
The conclusions are very similar to those for $n=1000$ (see Table \ref{tab:Example1d_full}),
showing that even $n=500$ points are enough for flow model to learn the distribution correctly. This is also confirmed in Section \ref{sec:sample_perform}, where we conducted a more detailed study on the influence of the sample size $n$ on the quality of the trained flow model.

\begin{table*}[t]
\caption{(Example \examplesecond{})  Numerical results ($R=2^{15}$) for $500$ training observations.}
\label{tab:ex1_n500}
\begin{subtable}{\textwidth}
\centering
{
\begin{tabular}{|
@{\hspace{2pt}}c@{\hspace{2pt}}|
@{\hspace{2pt}}c@{\hspace{2pt}}|
                p{0.55cm}p{0.55cm}p{0.65cm}|
                p{0.55cm}p{0.55cm}p{0.65cm}|
                p{0.55cm}p{0.55cm}p{0.65cm}|
                p{0.55cm}p{0.55cm}p{0.78cm}|}\hline
   $f$ & $I$& \multicolumn{3}{c|}{ \small{$\hat{Y}_{500}^{\rm obs}$}}
    &\multicolumn{3}{c|}{\small{$\hat{Y}^{\F, \rm CMC}_R$}}&
     \multicolumn{3}{c|}{\small{$\hat{Y}^{\F, \rm  prop, M1}_{R,m=4}$}}&
  \multicolumn{3}{c|}{\small{$\hat{Y}^{\F, \rm  opt, M1}_{R,m=4}$}} \\
    \hline
        & &\texttt{E}& \texttt{SD}${}^*$ & \texttt{AC} &
             \texttt{E} & \texttt{SD}${}^*$ & \texttt{AC}&
              \texttt{E} & \texttt{SD}${}^*$ &\texttt{AC}&
              \texttt{E} &\texttt{SD}${}^*$& \texttt{AC}
    \\
    \hline
     $j^{-}_{0.95}$  & .935 & .746 & 1.94  & 0.69
                    & .741 &  .241  & .683 % CMC
                    & .741 &  .049  & .682 % Prp m 4
                    & .741 &  .024  & \textbf{.684} % Opt m 4
               \\ \hline
        $j^+_{0.95}$ & .065 & .254 & 1.946  & \mbox{-0.46}
                    & .258 &  .241  & \mbox{-0.47} % CMC
                    & .258 &  .049  & \mbox{-0.47} % Prp m 4
                    & .258 &  .024  & -\textbf{0.47} % Opt m 4
               \\ \hline
   $j^{-}_{0.99}$ & .987 &  .9280 & 1.156  & 1.22
                    & .929 &  .141  & 1.23 % CMC
                    & .929 &  .124  & \underline{1.23} % Prp m 4
                    & .929 &  .062  & 1.22 % Opt m 4
                \\ \hline
    $\rho_1$ & .892 & .657 & .716  & .580
                    & .657 &  .088  & \textbf{.580} % CMC
                    & .657 &  .025  & .580 % Prp m 4
                    & .657 &  .023  & \underline{.580} % Opt m 4
               \\ \hline
   $\rho_2$ & .4032 & .6222 & .2818 & 0.265
                    & .6222 &  .0348  & \textbf{.265} % CMC
                    & .6222 &  .0164  & .265 % Prp m 4
               & .6222 &  .0118  & \underline{.265} % Opt m 4
   \\ \hline
    $\rho_3$ & 2.917 & 1.62 & .911  & .354
                    & 1.62 &  .121  & .354 % CMC
                    & 1.62 &  .082  & \textbf{.354} % Prp m 4
                    & 1.62 &  .048  & .354 % Opt m 4
   \\ \hline
\end{tabular}

}\vspace{0.4em}
\caption*{(a) Results for $m=4$}
\end{subtable}

\vspace{0.6em}

\begin{subtable}{\textwidth}
\centering
{
\begin{tabular}{|
@{\hspace{2pt}}c@{\hspace{2pt}}|
@{\hspace{2pt}}c@{\hspace{2pt}}|
                p{0.55cm}p{0.55cm}p{0.65cm}|
                p{0.55cm}p{0.55cm}p{0.65cm}|
                p{0.55cm}p{0.55cm}p{0.65cm}|
                p{0.55cm}p{0.55cm}p{0.78cm}|}\hline
   $f$ & $I$& \multicolumn{3}{c|}{ \small{$\hat{Y}_{500}^{\rm obs}$}}
    &\multicolumn{3}{c|}{\small{$\hat{Y}^{\F, \rm CMC}_R$}} &
        \multicolumn{3}{c|}{\small{$\hat{Y}^{\F, \rm  prop, M1}_{R,m=8}$}}&
  \multicolumn{3}{c|}{\small{$\hat{Y}^{\F, \rm  opt, M1}_{R,m=8}$}}\\
    \hline
        & &\texttt{E}& \texttt{SD}${}^*$ & \texttt{AC} &
             \texttt{E} & \texttt{SD}${}^*$ & \texttt{AC}&
              \texttt{E} & \texttt{SD}${}^*$ &\texttt{AC}&
              \texttt{E} &\texttt{SD}${}^*$& \texttt{AC}
    \\
    \hline
     $j^{-}_{0.95}$  & .935 & .746 & 1.94  & 0.69
                    & .741 &  .241  & .683 % CMC
                    & .741 &  .048  & .682 % Prp m 16
                    & .741 &  .017  & \underline{.683} % Opt m 16
               \\ \hline
        $j^+_{0.95}$ & .065 & .254 & 1.946  & \mbox{-0.46}
                    & .258 &  .241  & \mbox{-0.47} % CMC
                    & .258 &  .049  & \mbox{-0.47} % Prp m 16
                    & .258 &  .017  & \underline{-0.47} % Opt m 16
               \\ \hline
   $j^{-}_{0.99}$ & .987 &  .9280 & 1.156  & 1.22
                    & .929 &  .141  & 1.23 % CMC
                    & .929 &  .097  & \textbf{1.23} % Prp m 16
                    & .929 &  .034  & 1.23 % Opt m 16
                \\ \hline
    $\rho_1$ & .892 & .657 & .716  & .580
                    & .657 &  .088  & \textbf{.580} % CMC
                    & .657 &  .012  & .580 % Prp m 16
                    & .657 &  .012  & .580 % Opt m 16
               \\ \hline
   $\rho_2$ & .4032 & .6222 & .2818 & 0.265
                    & .6222 &  .0348  & \textbf{.265} % CMC
                    & .6222 &  .0110  & .265 % Prp m 16
                    & .6222 &  .0065  & .265 % Opt m 16
   \\ \hline
    $\rho_3$ & 2.917 & 1.62 & .911  & .354
                    & 1.62 &  .121  & .354 % CMC
                    & 1.62 &  .074  & \underline{.354} % Prp m 16
                    & 1.62 &  .046  & .354 % Opt m 16
   \\ \hline
\end{tabular}
}\vspace{0.4em}
\caption*{(b) Results for $m=8$}
\end{subtable}
\end{table*}

\begin{table*}[t]
\caption{(Example \examplesecond{})  Numerical results ($R=2^{15}$) for $2000$ training observations.}\label{tab:ex1_n2000}
\begin{subtable}{\textwidth}
\centering
{
\begin{tabular}{|
@{\hspace{2pt}}c@{\hspace{2pt}}|
@{\hspace{2pt}}c@{\hspace{2pt}}|
                p{0.55cm}p{0.55cm}p{0.65cm}|
                p{0.55cm}p{0.55cm}p{0.65cm}|
                p{0.55cm}p{0.55cm}p{0.65cm}|
                p{0.55cm}p{0.55cm}p{0.78cm}|}\hline
   $f$ & $I$& \multicolumn{3}{c|}{ \small{$\hat{Y}_{500}^{\rm obs}$}}
    &\multicolumn{3}{c|}{\small{$\hat{Y}^{\F, \rm CMC}_R$}}&
     \multicolumn{3}{c|}{\small{$\hat{Y}^{\F, \rm  prop, M1}_{R,m=4}$}}&
  \multicolumn{3}{c|}{\small{$\hat{Y}^{\F, \rm  opt, M1}_{R,m=4}$}} \\
    \hline
        & &\texttt{E}& \texttt{SD}${}^*$ & \texttt{AC} &
             \texttt{E} & \texttt{SD}${}^*$ & \texttt{AC}&
              \texttt{E} & \texttt{SD}${}^*$ &\texttt{AC}&
              \texttt{E} &\texttt{SD}${}^*$& \texttt{AC}
    \\
    \hline
      $j^{-}_{0.95}$ & .9351  & .9365  &  .5453   &  2.825
                    & .9349 &  .1363  & 3.066 % CMC
                    & .9359 &  .1206  & 3.282 % Prp m 4
                    & .9349 &  .0606  & 3.578 % Opt m 4
               \\ \hline
        $j^+_{0.95}$ &   .0649 & .0635 & .5453  & 1.667
                    & .0651 &  .1363  & 1.908 % CMC
                    & .0641 &  .1206  & 2.124 % Prp m 4
                    & .0649 &  .0606  & 2.117 % Opt m 4
               \\ \hline
   $j^{-}_{0.99}$ & .9875 & .9820 & .2973  & 2.255
                    & .9838 &  .0697  & 2.434 % CMC
                    & .9838 &  .0679  & \textbf{2.439} % Prp m 4
                    & .9839 &  .0339  & \underline{2.436} % Opt m 4
                \\ \hline
     $\rho_1$ & .8919 & .8905 & .3533  & 2.795
                    & .8908 &  .0872  & \underline{2.897} % CMC
                    & .8908 &  .0417  & \textbf{2.912} % Prp m 4
                    & .8904 &  .0261  & 2.781 % Opt m 4
               \\ \hline
   $\rho_2$ & .4032 & .4036 & .3589  & 3.015
                    & .4034 &  .0886  & 3.031 % CMC
                    & .4034 &  .0189  & \textbf{3.198} % Prp m 4
                    & .4036 &  .0179  & \underline{3.034} % Opt m 4
   \\ \hline
    $\rho_3$ & 2.917 & 2.921 & 2.624  & 2.903
                    & 2.921 &  .6473  & \textbf{3.118} % CMC
                    & 2.921 &  .1603  & 2.835 % Prp m 4
                    & 2.921 &  .1429  & \underline{2.950} % Opt m 4
   \\ \hline
\end{tabular}

}\vspace{0.4em}
\caption*{(a) Results for $m=4$}
\end{subtable}

\vspace{0.6em}

\begin{subtable}{\textwidth}
\centering
{
\begin{tabular}{|
@{\hspace{2pt}}c@{\hspace{2pt}}|
@{\hspace{2pt}}c@{\hspace{2pt}}|
                p{0.55cm}p{0.55cm}p{0.65cm}|
                p{0.55cm}p{0.55cm}p{0.65cm}|
                p{0.55cm}p{0.55cm}p{0.65cm}|
                p{0.55cm}p{0.55cm}p{0.78cm}|}\hline
   $f$ & $I$& \multicolumn{3}{c|}{ \small{$\hat{Y}_{2000}^{\rm obs}$}}
    &\multicolumn{3}{c|}{\small{$\hat{Y}^{\F, \rm CMC}_R$}} &
        \multicolumn{3}{c|}{\small{$\hat{Y}^{\F, \rm  prop, M1}_{R,m=8}$}}&
  \multicolumn{3}{c|}{\small{$\hat{Y}^{\F, \rm  opt, M1}_{R,m=8}$}}\\
    \hline
        & &\texttt{E}& \texttt{SD}${}^*$ & \texttt{AC} &
             \texttt{E} & \texttt{SD}${}^*$ & \texttt{AC}&
              \texttt{E} & \texttt{SD}${}^*$ &\texttt{AC}&
              \texttt{E} &\texttt{SD}${}^*$& \texttt{AC}

    \\
    \hline
      $j^{-}_{0.95}$ & .9351  & .9365  &  .5453   &  2.825
                    & .9349 &  .1363  & 3.066 % CMC
                    & .9350 &  .0269  & \underline{3.591} % Prp m 16
                    & .9350 &  .0068  & \textbf{3.995} % Opt m 16
               \\ \hline
        $j^+_{0.95}$ &   .0649 & .0635 & .5453  & 1.667
                    & .0651 &  .1363  & 1.908 % CMC
                    & .0650 &  .0269  & \underline{2.433} % Prp m 16
                    & .0650 &  .0068  & \textbf{2.807} % Opt m 16
               \\ \hline
   $j^{-}_{0.99}$ & .9875 & .9820 & .2973  & 2.255
                    & .9838 &  .0697  & 2.434 % CMC
                    & .9837 &  .0607  & 2.417 % Prp m 16
                    & .9838 &  .0151  & 2.424 % Opt m 16
                \\ \hline
     $\rho_1$ & .8919 & .8905 & .3533  & 2.795
                    & .8908 &  .0872  & \underline{2.897} % CMC
                    & .8905 &  .0102  & 2.808 % Prp m 16
                    & .8905 &  .0064  & 2.803 % Opt m 16
               \\ \hline
   $\rho_2$ & .4032 & .4036 & .3589  & 3.015
                    & .4034 &  .0886  & 3.031 % CMC
                    & .4036 &  .0057  & 3.015 % Prp m 16
                    & .4036 &  .0049  & 3.008 % Opt m 16
   \\ \hline
    $\rho_3$ & 2.917 & 2.921 & 2.624  & 2.903
                    & 2.921 &  .6473  & \textbf{3.118} % CMC
                    & 2.920 &  .0455  & 2.940 % Prp m 16
                    & 2.921 &  .0384  & 2.935 % Opt m 16
   \\ \hline
\end{tabular}
}\vspace{0.4em}
\caption*{(b) Results for $m=8$}
\end{subtable}
\end{table*}

\subsection{Example \examplemultit{}}\label{sec:app_Example3}
In Table \ref{tab:Example_stud_d4} we provided results for Example \examplemultit{} in case $d=4$. Here, in Table \ref{tab:Example_stud_d3},
we provide results for case $d=3$. Conclusions are similar.
 \begin{table*}[t]
 \caption{(Example \examplemultit{}, $d=3$) Numerical results for \textbf{cartesian} (method M1) and \textbf{spherical} (method M2)
}
\label{tab:Example_stud_d3}
\begin{subtable}{\textwidth}
\centering
%
%
% \begin{tabular}{|
% @{\hspace{2pt}}c@{\hspace{2pt}}|
% @{\hspace{2pt}}c@{\hspace{2pt}}|
%                 p{0.55cm}p{0.55cm}p{0.65cm}|
%                 p{0.55cm}p{0.55cm}p{0.65cm}|
%                 p{0.55cm}p{0.55cm}p{0.65cm}|
%                 p{0.55cm}p{0.55cm}p{0.78cm}|}\hline
%
{
\begin{tabular}{|
 @{\hspace{2pt}}c@{\hspace{2pt}}|
 @{\hspace{2pt}}c@{\hspace{2pt}}|
 p{0.3cm}p{0.3cm}p{0.4cm}|p{0.65cm}|
 p{0.45cm}p{0.45cm}p{0.5cm}|p{0.45cm}p{0.45cm}p{0.55cm}|p{0.45cm}p{0.45cm}p{0.55cm}|} \hline
     %   &  & & & & & & & &  \multicolumn{6}{c|}{Cartesian} &\multicolumn{6}{c|}{Spherical} \\\hline
        $f$ &     $I^*$& \multicolumn{3}{c|}{ \small{$\hat{Y}_{20k}^{\rm obs}$}}& $R$
            &\multicolumn{3}{c|}{\small{$\hat{Y}^{\F, \rm CMC}_R$}}&
             \multicolumn{3}{c|}{\small{$\hat{Y}^{\F, \rm  prop, M1}_{R, m=10\!\times\! 10\!\times\! 10}$}}&
          \multicolumn{3}{c|}{\small{$\hat{Y}^{\F, \rm  opt, M1}_{R, m=10\!\times\! 10\!\times\! 10}$}}\\
            \hline
            & &
            \texttt{E}$^*$ & \texttt{SD}${}^*$ & \texttt{AC} &
               &
           \texttt{E}$^*$ & \texttt{SD}${}^*$ & \texttt{AC}&
            \texttt{E}$^*$ & \texttt{SD}${}^*$ &\texttt{AC}&
            \texttt{E}$^*$ &\texttt{SD}${}^*$& \texttt{AC}
  \\
  \hline
  &  &  &  &  & 50k & \mbox{-.025} & .044 & \mbox{-.121} & \mbox{-.029} & .042 & \mbox{-.130} & \mbox{-.017} & \mbox{.008} & \mbox{-.015}  \\
  $h_1$ & \mbox{-.03} & .070 & .066 & \mbox{-.52} & 100k & \mbox{-.039} & .031 & .174 & \mbox{-.029} & .030 & .310 & \mbox{-.036} & .006 & \textbf{.475}  \\
  &  &  &  &  & 500k & \mbox{-.019} & .014 & .560 & \mbox{-.029} & .013 & .669 & \mbox{-.037} & .003 & \underline{1.35}  \\ \hline
  &  &  &  &  & 50k & .383 & .015 & \textbf{1.61} & .386 & .010 & \underline{1.56} & .390 & .002 & 1.53   \\
  $h_2$ & .400 & .376 & .024 & 1.23 & 100k & .380 & .011 & 1.37 & .387 & .007 & \textbf{1.53} & .381 & .001 & 1.35   \\
  &  &  &  &  & 500k & .384 & .005 & 1.41 & .386 & .003 & 1.47 & .386 & .001 & 1.45  \\ \hline
  &  &  &  &  & 50k & 1.30 & .051 & 1.45 & 1.31 & .037 & \textbf{1.61} & 1.36 & .005 & 1.12  \\
  $h_3$ & 1.36 & 1.27 & .080 & 1.20 & 100k & 1.28 & .036 & 1.35 & 1.30 & .025 & 1.40 & 1.31 & .004 & 1.34   \\
  &  &  &  &  & 500k & 1.30 & .017 & 1.41 & 1.31 & .012 & 1.44 & 1.31 & .002 & \textbf{1.55}  \\            \hline
        \end{tabular}

}\vspace{0.4em}
\caption*{(a) Results for method M1}
\end{subtable}

\vspace{0.6em}

\begin{subtable}{\textwidth}
\centering
{
\begin{tabular}{|
 @{\hspace{2pt}}c@{\hspace{2pt}}|
 @{\hspace{2pt}}c@{\hspace{2pt}}|
 p{0.3cm}p{0.3cm}p{0.4cm}|p{0.65cm}|
 p{0.45cm}p{0.45cm}p{0.5cm}|p{0.45cm}p{0.45cm}p{0.55cm}|p{0.45cm}p{0.45cm}p{0.55cm}|} \hline
       $f$ &     $I^*$& \multicolumn{3}{c|}{ \small{$\hat{Y}_{20k}^{\rm obs}$}}& $R$
            &\multicolumn{3}{c|}{\small{$\hat{Y}^{\F, \rm CMC}_R$}}&

                \multicolumn{3}{c|}{\small{$\hat{Y}^{\F, \rm  prop, M2}_{R, m=10\!\times\! 10\!\times\! 10}$}}&
          \multicolumn{3}{c|}{\small{$\hat{Y}^{\F, \rm  opt, M2}_{R, m=10\!\times\! 10\!\times\! 10}$}}\\
            \hline
            & &
            \texttt{E}$^*$ & \texttt{SD}${}^*$ & \texttt{AC} &
               &
           \texttt{E}$^*$ & \texttt{SD}${}^*$ & \texttt{AC}&
            \texttt{E}$^*$ & \texttt{SD}${}^*$ &\texttt{AC}&
            \texttt{E}$^*$ &\texttt{SD}${}^*$& \texttt{AC}
  \\
  \hline
  &  &  &  &  & 50k & \mbox{-.025} & .044 & \mbox{-.121} & \mbox{-.035} & .040 & \textbf{.247} & .004 & .009 & \underline{.095} \\
  $h_1$ & \mbox{-.03} & .070 & .066 & \mbox{-.52} & 100k & \mbox{-.039} & .031 & .174  &\mbox{-.027} & .029 & .182 & \mbox{-.021} & .006 & \underline{.347} \\
  &  &  &  &  & 500k & \mbox{-.019} & .014 & .560   & \mbox{-.023} & .013 & .762 & \mbox{-.030} & .003 & \textbf{1.90}\\ \hline
  &  &  &  &  & 50k & .383 & .015 & \textbf{1.61} & .382 & .011 & 1.39 & .394 & .002 & 1.25 \\
  $h_2$ & .400 & .376 & .024 & 1.23 & 100k & .380 & .011 & 1.37& .384 & .008 & \underline{1.43} & .382 & .001 & 1.27 \\
  &  &  &  &  & 500k & .384 & .005 & 1.41   & .387 & .003 & \textbf{1.51} & .386 & .001 & \underline{1.47} \\ \hline
  &  &  &  &  & 50k & 1.30 & .051 & 1.45   & 1.31 & .043 & \underline{1.55} & 1.34 & .006 & 1.10 \\
  $h_3$ & 1.36 & 1.27 & .080 & 1.20 & 100k & 1.28 & .036 & 1.35 & 1.31 & .027 & \textbf{1.52} & 1.29 & .004 & \underline{1.44} \\
  &  &  &  &  & 500k & 1.30 & .017 & 1.41 &  1.31 & .012 & \underline{1.47} & 1.31 & .002 & 1.42 \\            \hline
        \end{tabular}
}\vspace{0.4em}
\caption*{(b) Results for method M2.}
\end{subtable}
\end{table*}

 \subsection{Example \examplesdf{}: Signed Distance Function (SDF)}\label{sec:app_ExampleSDF}

In this experiment, the random vector $\X=(X_1,X_2,X_3)^T$ has a density induced by
a 3D mesh represented through a signed distance function (SDF)~\cite{sdf_git}.
Figure~\ref{fig:sdf} illustrates the mesh (left), a point cloud of $1000$ mesh points (center),
and a point cloud of $10000$ points generated by the flow trained on $1000$ points (right).

\begin{wraptable}{r}{0.6\textwidth}
\centering
\includegraphics[width=0.7\linewidth]{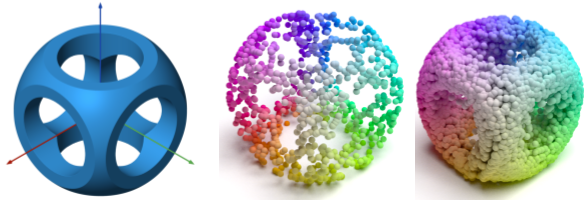}
\caption{Example \examplesdf{}. Sample mesh (SDF).}
\label{fig:sdf}
\end{wraptable}
We consider the function $j^{+}_t$ for $t\in\{0.30,0.50,0.55\}$.
The reference value~$I$ is computed from a mesh of $2^{22}$ points.
The model is trained on $1000$ examples and the estimation of~$I$ is performed using
$R\in\{2^{17},2^{18}\}$ points.
For optimal allocation, we use $R^{\prime}=R$ pilot simulations.

Table~\ref{tab:Example_SDF} presents the results for cartesian stratification.
Table~\ref{tab:Example_SDF_spher} contains the results for spherical stratification.
In both cases, stratified sampling provides lower empirical standard deviations compared
to crude Monte Carlo, with optimal allocation yielding the smallest variability.
Accuracy is highest for CMC and proportional allocation, while optimal allocation
provides the most significant variance reduction.
\begin{table*}%[htbp]
\caption{(Example \examplesdf{}) Numerical results for \textbf{cartesian} (M1) stratification.}
\label{tab:Example_SDF}
{\fontsize{8pt}{9.2pt}\selectfont
\setlength{\tabcolsep}{1.9pt}
% \begin{tabular}{|
%  @{\hspace{2pt}}c@{\hspace{2pt}}|
%  @{\hspace{2pt}}c@{\hspace{2pt}}|
%  p{0.3cm}p{0.3cm}p{0.4cm}|p{0.65cm}|
%  p{0.45cm}p{0.45cm}p{0.5cm}|p{0.45cm}p{0.45cm}p{0.55cm}|p{0.45cm}p{0.45cm}p{0.55cm}|} \hline
%
\begin{tabular}{|@{\hspace{2pt}}c@{\hspace{2pt}}|p{0.45cm}p{0.45cm}p{0.45cm}|p{0.45cm}|
p{0.5cm}p{0.5cm}p{0.45cm}|
p{0.55cm}p{0.55cm}p{0.45cm}|
p{0.55cm}p{0.55cm}p{0.45cm}|
p{0.55cm}p{0.55cm}p{0.45cm}|
p{0.55cm}p{0.55cm}p{0.45cm}|}
\hline
  $f$ &   \multicolumn{3}{c|}{ {$\hat{Y}_{1000}^{\rm obs}$}}& $R$    &\multicolumn{3}{c|}{{$\hat{Y}^{\F, \rm CMC}_R$}}&
     \multicolumn{3}{c|}{{$\hat{Y}^{\F, \rm  prop, M1}_{R,m=8}$}}&
  \multicolumn{3}{c|}{{$\hat{Y}^{\F, \rm  opt, M1}_{R,m=8}$}}&
        \multicolumn{3}{c|}{{$\hat{Y}^{\F, \rm  prop, M1}_{R,m=64}$}}&
  \multicolumn{3}{c|}{{$\hat{Y}^{\F, \rm  opt, M1}_{R,m=64}$}}\\
    \hline

          &
              \texttt{E} & \texttt{SD}${}^*$ & \texttt{AC} &
                   &
             \texttt{E} & \texttt{SD}${}^*$ & \texttt{AC}&
              \texttt{E} & \texttt{SD}${}^*$ &\texttt{AC}&
              \texttt{E} &\texttt{SD}${}^*$& \texttt{AC}&
              \texttt{E} & \texttt{SD}${}^*$& \texttt{AC}&
            \texttt{E} & \texttt{SD}${}^*$ & \texttt{AC}
    \\

    \hline

   \multirow{2}{*}{$j^{+}_{0.3}$} & \multirow{2}{*}{.053} & \multirow{2}{*}{.708} &\multirow{2}{*}{1.1}  &     $2^{17}$
   & .050&  .060 & 1.8 % CMC
   & .048 & .055  & 1.9  % Prp m 8
   & .049 & .033  & 2.1 % Opt m 8
   & .049& .051 & \textbf{2.5} % PRop m 64
   & .049& .025  & \underline{2.4} % Opt  m 64
   \\
   &   &   &   &     $2^{18}$
  & .049 &   .042 & 2.1 % CMC
   & .050  & .040  & 1.6  % Prp m 8
   & .049 & .023  & \textbf{3.2} % Opt m 8
   & .050 & .036 & 1.8 % PRop m 64
   & .049 & .017  & \underline{2.4} % Opt  m 64
   \\ \hline

     \multirow{2}{*}{$j^{+}_{0.5}$} & \multirow{2}{*}{.004}& \multirow{2}{*}{.199} & \multirow{2}{*}{.89} &   $2^{17}$
   & .004&   .019 & \underline{1.3}% CMC
   & .005  & .019 & .86 % Prp m 8
   & .005 &  .009  & .99 % Opt m 8
   & .004&   .018  & \textbf{1.6} % PRop m 64
   & .004  & .005 & 1.2 % Opt  m 64
   \\
   &   &   &   &    $2^{18}$
   & .004&  .013 & \textbf{1.4} % CMC
   & .005 & .013 & .99  % Prp m 8
   & .005 & .007  & 1.1 % Opt m 8
   & .004&  .013  & \underline{1.3} % PRop m 64
   & .005&  .004 & 1.1 % Opt  m 64
   \\ \hline

        \multirow{2}{*}{$j^{+}_{0.55}$} & \multirow{2}{*}{.000} & \multirow{2}{*}{.000} & \multirow{2}{*}{.00} &   $2^{17}$
   & .001& .008  & .23 % CMC
   & .001& .008 &   .33  % Prp m 8
   & .001 &.004  & \textbf{.46} % Opt m 8
   & .001& .009  & .14% PRop m 64
   & .001& .002 &  \underline{.34} % Opt  m 64
   \\
   &   &   &   &   $2^{18}$
   & .001&  .006& .30 % CMC
   & .001&  .006& .34 %  Prp m 8
   & .001 & .003 &.38  %Opt m 8
   & .001&  .005& \underline{.44}% PRop m 64
   & .001 & .001& \textbf{.45} % Opt  m 64
   \\ \hline

\end{tabular}
}
\end{table*}

\begin{table*}[t]
\caption{Example \examplesdf{}: numerical results for \textbf{spherical} (method M2) stratification.  }
\label{tab:Example_SDF_spher}
\begin{subtable}{\textwidth}
\centering
{
\fontsize{8pt}{9.2pt}\selectfont
\setlength{\tabcolsep}{1.9pt}
\begin{tabular}{|
@{\hspace{2pt}}c@{\hspace{2pt}}|
@{\hspace{2pt}}c@{\hspace{2pt}}|
p{0.65cm}p{0.65cm}p{0.65cm}|
p{0.5cm}|
rrr|rrr|rrr|rrr|rrr|rrr|rrr|rrr|rrr|rrr|rrr|}\hline
  $f$ &   $I$& \multicolumn{3}{c|}{ \small{$\hat{Y}_{1000}^{\rm obs}$}}& $R$

    &\multicolumn{3}{c|}{\small{$\hat{Y}^{\F, \rm CMC}_R$}}&
     \multicolumn{3}{c|}{\small{$\hat{Y}^{\F, \rm  prop, M2}_{R,m=8}$}}&
  \multicolumn{3}{c|}{\small{$\hat{Y}^{\F, \rm  opt, M2}_{R,m=8}$}} \\
    \hline

          & &
              \texttt{E} & \texttt{SD}${}^*$ & \texttt{AC} &
                   &
             \texttt{E} & \texttt{SD}${}^*$ & \texttt{AC}&
              \texttt{E} & \texttt{SD}${}^*$ &\texttt{AC}&
              \texttt{E} &\texttt{SD}${}^*$& \texttt{AC}
    \\

    \hline

   \multirow{2}{*}{$j^+_{0.3}$} & \multirow{2}{*}{.0494} & \multirow{2}{*}{.0532} & \multirow{2}{*}{.7085} &\multirow{2}{*}{1.146}  &     $2^{17}$
   & .0501&   .0603  & \textbf{1.848} % CMC
   & .0550& .0610 &  .9479  % Prp m 4
   & .0563&  .0515  &.8565 % Opt m 4

   \\
   &   &   &   &   &   $2^{18}$
  & .0491 &   .0422  & \textbf{2.126} % CMC
   & .0552&  .0432 & .9339  % Prp m 4
   & .0552 & .0383  & .9330 % Opt m 4

   \\ \hline

     \multirow{2}{*}{$j^+_{0.5}$} & \multirow{2}{*}{.0046} & \multirow{2}{*}{.0041}& \multirow{2}{*}{.1996} & \multirow{2}{*}{.8927} &   $2^{17}$
   & .0048&     .0191 & \underline{1.277}% CMC
   & .0034  & .0162 & .6229 % Prp m 4
   & .0034 & .0104 & .5888 % Opt m 4

   \\
   &   &   &   &   &   $2^{18}$
   & .0047&    .0134  & \textbf{1.412} % CMC
&   .0032&   .0109 & .5148 % Prp m 4
   & .0032 & .0069 & .5090 % Opt m 4

   \\ \hline

        \multirow{2}{*}{$j^+_{0.55}$} & \multirow{2}{*}{.0006} & \multirow{2}{*}{.0000} & \multirow{2}{*}{.0000} & \multirow{2}{*}{.0000} &   $2^{17}$
   & .0009&     .0083 & .2921 % CMC
      & .0004&.0061& .7934 %  Prp m 4
   & .0005 & .0038 & .8513  %Opt m 4

   \\
   &   &   &   &   &   $2^{18}$
   & .0008&   .0058 & .3032 % CMC
&    .0006 & .00468& \textbf{1.595} % Prp m 4
   & .0005& .0026&   .7001  % Opt m 4

   \\ \hline

\end{tabular}

}\vspace{0.4em}
\caption*{(a) Results for $m=8$.}
\end{subtable}

\vspace{0.6em}

\begin{subtable}{\textwidth}
\centering
{
\fontsize{8pt}{9.2pt}\selectfont
\setlength{\tabcolsep}{1.9pt}
\begin{tabular}{|p{0.5cm}|p{0.6cm}|p{0.65cm}p{0.65cm}p{0.65cm}|p{0.5cm}|
rrr|rrr|rrr|rrr|rrr|rrr|rrr|rrr|rrr|rrr|rrr|}\hline
  $f$ &   $I$& \multicolumn{3}{c|}{ \small{$\hat{Y}_{1000}^{\rm obs}$}}& $R$

    &\multicolumn{3}{c|}{\small{$\hat{Y}^{\F, \rm CMC}_R$}}&
        \multicolumn{3}{c|}{\small{$\hat{Y}^{\F, \rm  prop, M2}_{R,m=64}$}}&
  \multicolumn{3}{c|}{\small{$\hat{Y}^{\F, \rm  opt, M2}_{R,m=64}$}}\\
    \hline

          & &
              \texttt{E} & \texttt{SD}${}^*$ & \texttt{AC} &
                   &
             \texttt{E} & \texttt{SD}${}^*$ & \texttt{AC}&
              \texttt{E} & \texttt{SD}${}^*$ &\texttt{AC}&
              \texttt{E} &\texttt{SD}${}^*$& \texttt{AC}
    \\

    \hline

   \multirow{2}{*}{$j^+_{0.3}$} & \multirow{2}{*}{.0494} & \multirow{2}{*}{.0532} & \multirow{2}{*}{.7085} &\multirow{2}{*}{1.146}  &     $2^{17}$
   & .0501&   .0603  & \textbf{1.848} % CMC

   & .0559&  .0565 & .8819 % PRop m 16
   & .0547 & .0306  & \underline{.9767} % Opt  m 16
   \\
   &   &   &   &   &   $2^{18}$
  & .0491 &   .0422  & \textbf{2.126} % CMC

   & .0553 & .0398 & .9268 % PRop m 16
   & .0549 & .0216  & \underline{.9562} % Opt  m 16
   \\ \hline

     \multirow{2}{*}{$j^+_{0.5}$} & \multirow{2}{*}{.0046} & \multirow{2}{*}{.0041}& \multirow{2}{*}{.1996} & \multirow{2}{*}{.8927} &   $2^{17}$
   & .0048&     .0191 & \underline{1.277}% CMC

   & .0005&  .0046  &\textbf{1.595} % PRop m 16
   & .0004 & .0026 & 1.226 % Opt  m 16
   \\
   &   &   &   &   &   $2^{18}$
   & .0047&    .0134  & \textbf{1.412} % CMC

   & .0032&  .0109 & .5231 % PRop m 16
   & .0034&  .0042& \underline{.5810} % Opt  m 16
   \\ \hline

        \multirow{2}{*}{$j^+_{0.55}$} & \multirow{2}{*}{.0006} & \multirow{2}{*}{.0000} & \multirow{2}{*}{.0000} & \multirow{2}{*}{.0000} &   $2^{17}$
   & .0009&     .0083 & .2921 % CMC

   & .0005&  .0065& \textbf{1.239}% PRop m 16
   & .0005&  .0019& \underline{.9648} % Opt  m 16
   \\
   &   &   &   &   &   $2^{18}$
   & .0008&   .0058 & .3032 % CMC

   & .0005& .0043 & \underline{.8491}% PRop m 16
   & .0005& .0013&  .8269 % Opt  m 16
   \\ \hline

\end{tabular}
}\vspace{0.4em}
\caption*{(b) Results for $m=64.$}
\end{subtable}
\end{table*}

\subsection{Example \examplereal{} (real-world AmeriGEOSS data)}\label{sec:app_ExampleWIND}

We considered estimating expectations of functions $g_1, g_2, g_3$ and
provided results (Table \ref{tab:Example2d_WIND_spher}) of spherical stratification only.
Here we additionally consider functions
$$g_4(x_1,x_2)=|\cos(e^{x_1x_2})|, \quad  \quad g_5(x_1,x_2)=\log(x_1+x_2), \quad g_6(x_1,x_2)=|\log|x_1+x_2||^{-1}$$
and provide also results for cartesian stratification.
Table \ref{tab:Example2d_WIND_spher_full} contains result for
spherical stratification (thus, this is an extended version of Table \ref{tab:Example2d_WIND_spher}),
whereas Table  \ref{tab:Example2d_WIND_cart_full} contains results for cartesian stratification.
Moreover, in Fig. \ref{fig:AmeriGEOSS_orig_flow} we present 3000 (left) training data
as well as 3000 points sampled from flow model trained on these points (right).

\setlength\tabcolsep{3pt}
  \begin{table*}[h!]
  \caption{(Example \examplereal{}) Numerical results for \textbf{cartesian} (method M1) stratification. }
\label{tab:Example2d_WIND_cart_full}
\begin{center}
{
\fontsize{8pt}{9.2pt}\selectfont
\setlength{\tabcolsep}{1.9pt}
 \begin{tabular}{|p{0.34cm}|p{0.7cm}p{0.7cm}|p{0.4cm}|p{0.8cm}p{0.8cm}|p{0.8cm}p{0.8cm}|p{0.8cm}p{0.8cm}|p{0.8cm}p{0.8cm}|p{0.8cm}p{0.85cm}|}\hline
   $F$ &  \multicolumn{2}{c|}{ \small{$\hat{Y}_{3000}^{\rm obs}$}}& $R$
     &\multicolumn{2}{c|}{\small{$\hat{Y}^{\F, \rm CMC}_R$}}&
      \multicolumn{2}{c|}{\small{$\hat{Y}^{\F, \rm  prop, M1}_{R,m=2\times2}$}}&
   \multicolumn{2}{c|}{\small{$\hat{Y}^{\F, \rm  opt, M1}_{R,m=2\times2}$}}&
         \multicolumn{2}{c|}{\small{$\hat{Y}^{\F, \rm  prop, M1}_{R,m=4\times4}$}}&
   \multicolumn{2}{c|}{\small{$\hat{Y}^{\F, \rm  opt, M1}_{R,m=4\times4}$}}
  \\
     \hline

           &
               \texttt{E} & \texttt{SD}${}^*$ &
                    &
              \texttt{E} & \texttt{SD}${}^*$ &
               \texttt{E} & \texttt{SD}${}^*$ &
               \texttt{E} &\texttt{SD}${}^*$&
               \texttt{E} & \texttt{SD}${}^*$&
             \texttt{E} & \texttt{SD}${}^*$
      \\
     \hline
         \multirow{2}{*}{$g_1$} &  \multirow{2}{*}{.1923} & \multirow{2}{*}{.7196}   &   $2^{12}$
                & .1944 &  .6182   % CMC
                & .1920 &  .3654  % Prp m 4
                & .1935 &  \underline{.2368}   % Opt m 4
                & .1930 &  .2915   % Prp m 16
                & .1934 &  \textbf{.1389}   % Opt m 16
           \\
           &   &   &       $2^{15}$
                & .1943 &  .2186   % CMC
                & .1935 &  .1287    % Prp m 4
                & .1935 &  \underline{.0833}   % Opt m 4
                & .1933 &  .1037   % Prp m 16
                & .1932 &  \textbf{.0495}   % Opt m 16
           \\ \hline
           \multirow{2}{*}{$g_2$} &   \multirow{2}{*}{.0009} & \multirow{2}{*}{.0877}  &   $2^{12}$
                & .0010 &  .0749   % CMC
                & .0007 &  .0457   % Prp m 4
                & .0012 &  .0459   % Opt m 4
                & .0009 &  \underline{.0276}   % Prp m 16
                & .0009 &  \textbf{.0256}    % Opt m 16
           \\
              &   &   &   $2^{15}$
                & .0010 &  .0265    % CMC
                & .0009 &  .0162  % Prp m 4
                & .0008 &  .0162   % Opt m 4
                & .0009 & \underline{.0098}   % Prp m 16
                & .0009 &  \textbf{.0091}    % Opt m 16
           \\ \hline
           \multirow{2}{*}{$g_4$} &  \multirow{2}{*}{.0046} & \multirow{2}{*}{.0097}  &     $2^{12}$
                & .0046 &  .0082      % CMC
                & .0046 &  .0082      % Prp m 4
                & .0046 &  .0083      % Opt m 4
                & .0046 &  \underline{.0053}    % Prp m 16
                & .0046 &  \textbf{.0045}     % Opt m 16
           \\
           &   &   &      $2^{15}$
                & .0046 &  .0029    % CMC
                & .0046 &  .0029   % Prp m 4
                & .0046 &  .0029  % Opt m 4
                & .0046 &  \underline{.0019}    % Prp m 16
                & .0046 &  \textbf{.0016}    % Opt m 16
           \\ \hline
               \multirow{2}{*}{$g_5$} &   \multirow{2}{*}{.5402} & \multirow{2}{*}{.0109}  &    $2^{12}$
                & .5402 &  .0092    % CMC
                & .5402 &  .0070   % Prp m 4
                & .5402 &  .0070   % Opt m 4
                & .5402 &  \underline{.0046}   % Prp m 16
                & .5402 &  \textbf{.0039}    % Opt m 16
           \\
           &   &   &      $2^{15}$
                & .5402 &  .0033   % CMC
                & .5402 &  .0025   % Prp m 4
                & .5402 &  .0025    % Opt m 4
                & .5402 &  .0016    % Prp m 16
                & .5402 &  \textbf{.0014}    % Opt m 16
           \\ \hline
            \multirow{2}{*}{$g_6$} & \multirow{2}{*}{-2.847} & \multirow{2}{*}{2.839}  &   $2^{12}$
                & -2.825 &  2.171   % CMC
                & -2.837 &  1.988    % Prp m 4
                & -2.839 &  1.953   % Opt m 4
                & -2.829 &  \underline{1.266}   % Prp m 16
                & -2.837 &  \textbf{1.141}    % Opt m 16
           \\
           &   &   &      $2^{15}$
                & -2.832 &  .7722   % CMC
                & -2.835 &  .7075    % Prp m 4
                & -2.833 &  .6889   % Opt m 4
                & -2.832 &  \underline{.4531}    % Prp m 16
                & -2.835 &  \textbf{.4017}   % Opt m 16
           \\ \hline
                   \multirow{2}{*}{$g_7$} & \multirow{2}{*}{.4370} & \multirow{2}{*}{.3574}  &    $2^{12}$
                & .4373 &  .3033    % CMC
                & .4357 &  .2752    % Prp m 4
                & .4380 &  .2746   % Opt m 4
                & .4374 &  \underline{.1690}   % Prp m 16
                & .4368 &  \textbf{.1611}    % Opt m 16
           \\
           &   &   &       $2^{15}$
                & .4371 &  .1078   % CMC
                & .4366 &  .0976   % Prp m 4
                & .4369 &  .0970   % Opt m 4
                & .4371 &  \underline{.0599}    % Prp m 16
                & .4365 &  \textbf{.0567}   % Opt m 16
           \\ \hline
\end{tabular}
}
\end{center}

\end{table*}

\begin{table*}[h!]
\vspace{.4cm}
\caption{(Example \examplereal{}) Numerical results for \textbf{spherical} (method M2) stratification.
Full version of Table \ref{tab:Example2d_WIND_spher}.}
\label{tab:Example2d_WIND_spher_full}
\begin{center}
{
\fontsize{8pt}{9.2pt}\selectfont
\setlength{\tabcolsep}{1.9pt}
  \begin{tabular}{|p{0.34cm}|p{0.7cm}p{0.7cm}|p{0.6cm}|p{0.8cm}p{0.8cm}|p{0.8cm}p{0.8cm}|p{0.8cm}p{0.8cm}|p{0.8cm}p{0.8cm}|p{0.8cm}p{0.85cm}|}\hline
   $F$ &  \multicolumn{2}{c|}{ \small{$\hat{Y}_{3000}^{\rm obs}$}}& $R$
     &\multicolumn{2}{c|}{\small{$\hat{Y}^{\F, \rm CMC}_R$}}&
      \multicolumn{2}{c|}{\small{$\hat{Y}^{\F, \rm  prop, M2}_{R,m=2\times2}$}}&
   \multicolumn{2}{c|}{\small{$\hat{Y}^{\F, \rm  opt, M2}_{R,m=2\times2}$}}&
         \multicolumn{2}{c|}{\small{$\hat{Y}^{\F, \rm  prop, M2}_{R,m=4\times4}$}}&
   \multicolumn{2}{c|}{\small{$\hat{Y}^{\F, \rm  opt, M2}_{R,m=4\times4}$}}
  \\
     \hline

           &
               \texttt{E} & \texttt{SD}${}^*$ &
                    &
              \texttt{E} & \texttt{SD}${}^*$ &
               \texttt{E} & \texttt{SD}${}^*$ &
               \texttt{E} &\texttt{SD}${}^*$&
               \texttt{E} & \texttt{SD}${}^*$&
             \texttt{E} & \texttt{SD}${}^*$
      \\
     \hline
                \multirow{2}{*}{$g_1$} &  \multirow{2}{*}{.1923} & \multirow{2}{*}{.7196}  &    $2^{12}$
                & .1944 &  .6182   % CMC
                & .1938 &  .5270    % Prp m 4
                & .1916 &  .3715   % Opt m 4
                & .1944 &  \underline{.2883}    % Prp m 16
                & .1930 &  \textbf{.1758}   % Opt m 16
           \\
           &   &   &      $2^{15}$
                & .1943 &  .2186    % CMC
                & .1926 &  .1862    % Prp m 4
                & .1929 &  .1315    % Opt m 4
                & .1937 &  \underline{.1031}   % Prp m 16
                & .1931 &  \textbf{.0640}   % Opt m 16
           \\ \hline
            \multirow{2}{*}{$g_2$}  & \multirow{2}{*}{.0009} & \multirow{2}{*}{.0877}  &   $2^{12}$
                & .0010 &  .0749   % CMC
                & .0008 &  .0372   % Prp m 4
                & .0007 &  .0349   % Opt m 4
                & .0007 &  \underline{.0339}   % Prp m 16
                & .0010 &  \textbf{.0303}    % Opt m 16
           \\
           &   &     &   $2^{15}$
                & .0010 &  .0265    % CMC
                & .0008 &  .0131   % Prp m 4
                & .0009 &  .0123    % Opt m 4
                & .0009 &  \underline{.0120}   % Prp m 16
                & .0009 &  \textbf{.0107}   % Opt m 16
           \\ \hline
            \multirow{2}{*}{$g_4$} &  \multirow{2}{*}{.0046} & \multirow{2}{*}{.0097}  &    $2^{12}$
                & .0046 &  .0082   % CMC
                & .0046 &  .0062    % Prp m 4
                & .0046 &  .0052    % Opt m 4
                & .0046 &  \underline{.0051}    % Prp m 16
                & .0046 &  \textbf{.0039}   % Opt m 16
           \\
           &   &      &   $2^{15}$
                & .0046 &  .0029  % CMC
                & .0046 &  .0022    % Prp m 4
                & .0046 &  \underline{.0018}    % Opt m 4
                & .0046 &  .0018   % Prp m 16
                & .0046 &  \textbf{.0014}  % Opt m 16
           \\ \hline
            \multirow{2}{*}{$g_5$} &  \multirow{2}{*}{.5402} & \multirow{2}{*}{.0109}   &   $2^{12}$
                & .5402 &  .0092    % CMC
                & .5402 &  .0092    % Prp m 4
                & .5402 &  .0075   % Opt m 4
                & .5402 &  .0046    % Prp m 16
                & .5402 &  \textbf{.0035}   % Opt m 16
           \\
           &   &   &      $2^{15}$
                & .5402 &  .0033    % CMC
                & .5402 &  .0032   % Prp m 4
                & .5402 &  .0026  % Opt m 4
                & .5402 &  .0017   % Prp m 16
                & .5402 &  \textbf{.0012}   % Opt m 16
           \\ \hline
           \multirow{2}{*}{$g_6$} &   \multirow{2}{*}{-2.847} & \multirow{2}{*}{2.839}  &     $2^{12}$
                & -2.825 &  2.171    % CMC
                & -2.836 &  1.914  % Prp m 4
                & -2.836 &  1.884    % Opt m 4
                & -2.833 &  \underline{1.525}  % Prp m 16
                & -2.834 &  \textbf{1.326}   % Opt m 16
           \\
           &   &   &     $2^{15}$
                & -2.832 &  .7722   % CMC
                & -2.834 &  .6742    % Prp m 4
                & -2.841 &  .6747    % Opt m 4
                & -2.833 &  \underline{.5311}    % Prp m 16
                & -2.837 &  \textbf{.4685}   % Opt m 16
           \\ \hline
           \multirow{2}{*}{$g_7$} &  \multirow{2}{*}{.4370} & \multirow{2}{*}{.3574}  &     $2^{12}$
                & .4373 &  .3033   % CMC
                & .4371 &  .2480    % Prp m 4
                & .4365 &  .2391    % Opt m 4
                & .4375 &  \underline{.1733}   % Prp m 16
                & .4361 &  \textbf{.1592}   % Opt m 16
           \\
           &   &   &       $2^{15}$
                & .4371 &  .1078  % CMC
                & .4366 &  .0872    % Prp m 4
                & .4367 &  .0840    % Opt m 4
                & .4370 &  \underline{.0613}    % Prp m 16
                & .4365 &  \textbf{.0563}    % Opt m 16
           \\ \hline

\end{tabular}
}
\end{center}
\end{table*}

 \begin{figure}[h!]
  \begin{center}
  \begin{tabular}{llll}
  \includegraphics[width=0.46\textwidth]{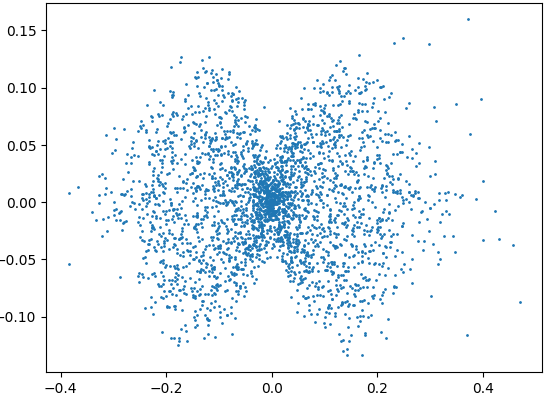} & $\qquad$
  \includegraphics[width=0.46\textwidth]{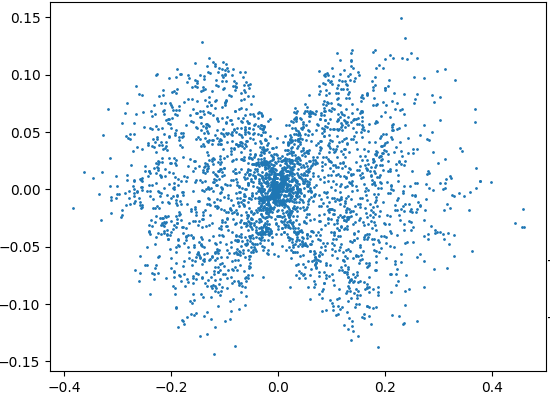}
  \end{tabular}
  \end{center}
\caption{Example \examplereal{}: 3000 training points (left) and 3000 replications from  trained flow model (right)}
\label{fig:AmeriGEOSS_orig_flow}
%\end{wrapfigure}
\end{figure}

\subsection{Example \examplelarged{} (30D example)}\label{sec:app_Example30d}
In Table~\ref{tab:Example_30d_rad_part1} we presented
estimators
$\hat{Y}^{\F, \rm  prop, \mathrm{M}\texttt{rad}}_{R,m=3}$,
$\hat{Y}^{\F, \rm  opt, \mathrm{M}\texttt{rad}}_{R,m=3}$,
and
$\hat{Y}^{\F, \rm  prop, \mathrm{M}\texttt{rad}}_{R,m=7}$. In Table~\ref{tab:Example_30d_rad_part2}
estimators
$\hat{Y}^{\F, \rm  opt, \mathrm{M}\texttt{rad}}_{R,m=7}$,
$\hat{Y}^{\F, \rm  prop, M\texttt{Rand3}}_{{R},m=3\times 3\times 3}$, and
$\hat{Y}^{\F, \rm  prop, M\texttt{High3}}_{{R},m=3\times 3\times 3}$
are provided.
\begin{table*}[h]
\caption{(Example  \examplelarged{}) Results for $\mathrm{M}\texttt{rad}$,
$\mathrm{M}\texttt{High3}$ and $\mathrm{M}\texttt{Rand3}$ methods. Part 2.  }
\label{tab:Example_30d_rad_part2}
\fontsize{8pt}{9.2pt}\selectfont
\setlength{\tabcolsep}{1.9pt}
\begin{tabular}{|C{0.75cm}|p{0.6cm}|p{0.6cm}p{0.6cm}p{0.7cm}|C{0.45cm}|p{0.6cm}p{0.6cm}p{0.6cm}|p{0.6cm}p{0.6cm}p{0.6cm}|p{0.6cm}p{0.6cm}p{0.6cm}|p{0.6cm}p{0.6cm}p{0.6cm}|}\hline
    \multicolumn{1}{|c|}{$f$} &   $I^*$& \multicolumn{3}{c|}{ \small{$\hat{Y}_{n=500}^{\rm obs}$}}& $R$
     &\multicolumn{3}{c|}{{$\hat{Y}^{\F, \rm CMC}_R$}}&
   \multicolumn{3}{c|}{{$\hat{Y}^{\F, \rm  opt, \mathrm{M}\texttt{rad}}_{R,m=7}$}} &
     \multicolumn{3}{c|}{{$\hat{Y}^{\F, \rm  prop, M\texttt{Rand3}}_{{R},m=3\times 3\times 3}$}}&
         \multicolumn{3}{c|}{{$\hat{Y}^{\F, \rm  prop, M\texttt{High3}}_{{R},m=3\times 3\times 3}$}}\\
     \hline
           & &
               {\texttt{E}${}^*$}&{\texttt{SD}${}^*$}&{\texttt{AC}}&    &
               {\texttt{E}${}^*$}&{\texttt{SD}${}^*$}&{\texttt{AC}}&
               {\texttt{E}${}^*$}&{\texttt{SD}${}^*$}&{\texttt{AC}}&
               {\texttt{E}${}^*$}&{\texttt{SD}${}^*$}&{\texttt{AC}}&
               {\texttt{E}${}^*$}&{\texttt{SD}${}^*$}&{\texttt{AC}}
     \\
    \hline
   \multirow{2}{*}{$\!\!j^+_{-1.0}$} & \multirow{2}{*}{\!1.37} &\multirow{2}{*}{1.80}& \multirow{2}{*}{.595} &\multirow{2}{*}{.504}  &     $\!\!2^{12}$
   & 1.67 & .200 & .733% CMC
   & 2.60 & .141 &	.815% Opt  m 16
   & 1.82 & .208 & .507	 % Prp m 4
   & 1.86 & .210 & .465%
   \\
   &   &   &   &   &   $\!\!2^{13}$
 & 1.67 &   .141  & 	\underline{.703}% CMC
   & 1.55&	.132	& \textbf{.918} % Opt  m 16
    & 1.74 &   .144 & .605 % Prp m 4
   & 1.68	&.141&	.697% Opt m 4
   \\ \hline
     \multirow{2}{*}{$\!\!j^+_{-0.8}$} & \multirow{2}{*}{\!.560} &\multirow{2}{*}{.400}& \multirow{2}{*}{.282} & \multirow{2}{*}{.543} &   $\!\!2^{12}$
   & .735	 & .133 & .635% CM
   & .235 & .043 & .333 % Opt  m 16
      & .699& .129 & .642 % Prp m 4
   & .679& .127 & .841% Opt m 4
   \\
   &   &   &   &   &   $\!\!2^{13}$
   & .711	& .092	& .589 % CMC
   & .444 & .057 & .918 % Opt  m 16
   &   .669	& .089	& \textbf{1.07} % Prp m 4
   & .689 & .091 & .719% Opt m 4
   \\ \hline
        \multirow{2}{*}{$\!\!j^+_{-0.6}$} & \multirow{2}{*}{\!.194} &\multirow{2}{*}{.000} & \multirow{2}{*}{.000} & \multirow{2}{*}{.000} &   $\!\!2^{12}$
        & 	.236 &	.075	& \textbf{1.00} % CMC
   & .151& .047& .453 % Opt  m 16
    & .191 & .067 & .821  %Opt m 4
   & .191	&.067 & \underline{.972}% PRop m 16
   \\
   &   &   &   &   &   $\!\!2^{13}$
   & 	.267	 & .056 & .506 % CMC
   & .070 & .016 & .264 % Opt  m 16
     &  .190  & .048 & \textbf{1.06} %  Prp m 4
   &.189	&.047 & \underline{.924}
   \\ \hline
\end{tabular}
%}\vspace{0.4em}
%\caption*{(b) Results: part 2.}
%\end{subtable}
\end{table*}

We consider synthetic $d=30$ dimensional example with coordinates from different distributions.
The distribution of $\mathbf{X}=(X_1,\ldots,X_{30})$ can be described as follows:
\begin{itemize}[leftmargin=0.5cm,align=left, itemsep=0.2cm]
    \item $X_6, X_7, X_8, X_9, X_{21}, X_{22}, X_{23}, X_{24}$ are independent having standard normal  $N(0, \sigma^2)$
    distributions with standard deviations $1,\ 1.2,\ 1.4,\ 1.6,\ 1,\ 1.2,\ 1.4$ and $1.6$.
    \item $X_{10}, X_{11}, X_{12}, X_{13}, X_{25}, X_{26}, X_{27}, X_{28}$  are independent  exponential Exp($\lambda$) random variables
    with scales $ 0.4,\ 0.5,\ 0.6,\ 0.7,\ 0.4,\ 0.5,\ 0.6$ and $ 0.7$.
    \item $X_{14}, X_{15}, X_{29}, X_{30}$  are independet  gamma random variables  $\Gamma(2, \beta)$ with
    parameters $\beta$ being $1.6,\ 1.8,\ 1.6$ and $ 1.8$.
    \item $(X_1, ..., X_5)$  and $(X_{16},.., X_{20})$
    are two five-dimensional normal random vectors $\mathcal{N}(\boldsymbol{\mu}_1,\boldsymbol{\Sigma}_1)$
    and $\mathcal{N}(\boldsymbol{\mu}_2,\boldsymbol{\Sigma}_2)$ respectively, where
    $$
    \begin{array}{lllllllll}
  \boldsymbol{\mu}_1  & = & \begin{bmatrix}
    -1, &1, & 2, & 1, & 1
\end{bmatrix}^{T}
& \qquad & \boldsymbol{\mu}_2 &=&
    \begin{bmatrix}
    1.0,& -0.5,& 0.0,& 1.2,& -0.8
\end{bmatrix}^{T}
\\[10pt]
\boldsymbol{\Sigma}_1 &=&
    \begin{bmatrix}
55&  8& 17& 19& 23\\
8&  8&  8&  9& 11\\
17&  8& 13& 15& 20\\
19&  9& 15& 23& 24\\
23& 11& 20& 24& 32
\end{bmatrix} & \qquad &  \boldsymbol{\Sigma}_2 &=&
    \begin{bmatrix}
    1.0& 0.8& 0.6& 0.4& -0.3\\
    0.8& 2.0& 1.0& 0.7& -0.5\\
    0.6& 1.0& 1.5& 0.9& -0.4\\
    0.4& 0.7& 0.9& 1.2& -0.2\\
    -0.3& -0.5& -0.4& -0.2& 1.0
\end{bmatrix}
    \end{array}
    $$

\end{itemize}

 \begin{figure}[h!]
  \begin{center}
  \end{center}
\caption{Example \examplelarged{}: Pairplot 1000 points from $d=30$ dimensional distribution.}
\vspace{1cm}
\label{fig:Ex6Pairplot}
\end{figure}

1000 points sampled from this 30D distribution are shown in Fig. \ref{fig:Ex6Pairplot}.
The figure contains a pairplot: $i$-th row and $j$-th column (for $i\neq j$) is a
scatter plot between $X_i$ and $X_j$, whereas histograms are presented on the diagonal.
We can see that most of the dimensions are independent,
the dependent part is  clearly visible  (upper right corner and   near the center of the plot).
This synthetic example was created for reflect real-life scenarios, where we have data with multiple dimensions and that some of the variables are correlated.
In our experiments we estimate $I = P(X_1 \geq t, X_2 \geq t, ..., X_{30} \geq t)$ (i.e., $\Exp j^+_t(\X)$) for various values of $t$.
For $\textrm{M}\texttt{High3}$ method we first perform $R_0=2^{10}$ simulations per dimension to estimate $\sigma^{\rm dim}$.

\subsection{Example \examplelargestd{}: 128-dimensional experiment}
\label{sec:app_Example128d}

In this experiment, we consider the $128$-dimensional latent space of the
PointFlow model~\cite{yang2019pointflow} trained on $4500$ examples from the
\emph{chair} subset of the ShapeNet dataset~\cite{chang2015shapenet}.
The distribution in the latent space is modeled using a normalizing flow,
and we used the publicly available trained model~\cite{pointflow_git} to compute
estimates of $\Exp j^{+}_{t}(\X)$ for $t\in\{-0.45,-0.50,-0.55\}$,
where $\X\in\mathbb{R}^{128}$.

We applied three stratification approaches:
(i)~the $\textrm{M}\texttt{rad}$ method, stratifying the radius into
$m\in\{3,5,7,9\}$ strata,
(ii)~the $\textrm{M}\texttt{High3}$ method, and
(iii)~the $\textrm{M}\texttt{Rand3}$ method,
both stratifying three selected coordinates into
$m_0\in\{2,4\}$ strata (yielding $8$ or $64$ strata in total).

Table~\ref{tab:Example_128d} reports the results for $n=1000$ samples.
Since the true value $I$ is unknown, we present only the estimates $\texttt{E}$
and the empirical standard deviations $\texttt{SD}$.
In all but one case, a stratified method achieves the lowest $\texttt{SD}$.
However, the differences between methods are relatively small.

\begin{table*}
\caption{(Example 7) numerical results for $\mathrm{M}\texttt{rad}$,
$\mathrm{M}\texttt{High3}$ and $\mathrm{M}\texttt{Rand3}$ methods.}
\label{tab:Example_128d}
\begin{subtable}{\textwidth}
\centering
{
\fontsize{8pt}{9.2pt}\selectfont
\setlength{\tabcolsep}{1.9pt}
\begin{tabular}{|C{0.85cm}|p{0.75cm}p{0.7cm}|C{0.5cm}|p{0.7cm}p{0.75cm}|p{0.7cm}p{0.75cm}|p{0.7cm}p{0.75cm}|p{0.7cm}p{0.75cm}|p{.7cm}p{0.8cm}|}\hline
         \multicolumn{1}{|c|}{$f$} &
         \multicolumn{2}{c|}{ \small{$\hat{Y}_{4500}^{\rm obs}$}}& $R$&
         \multicolumn{2}{c|}{\small{$\hat{Y}^{\F, \rm CMC}_R$}}&
         \multicolumn{2}{c|}{\small{$\hat{Y}^{\F, \textrm{M}\texttt{rad}}_{R,m=3}$}}&
         \multicolumn{2}{c|}{\small{$\hat{Y}^{\F, \textrm{M}\texttt{rad}}_{R,m=5}$}}&
         \multicolumn{2}{c|}{\small{$\hat{Y}^{\F, \textrm{M}\texttt{rad}}_{R,m=7}$}}&
         \multicolumn{2}{c|}{\small{$\hat{Y}^{\F, \textrm{M}\texttt{rad}}_{R,m=9}$}}
  \\
     \hline
    & \texttt{EST} & \texttt{STD}${}^*$ &   &
       \texttt{EST} & \texttt{STD}${}^*$ &
       \texttt{EST} & \texttt{STD}${}^*$ &
       \texttt{EST} &\texttt{STD}${}^*$&
       \texttt{EST} & \texttt{STD}${}^*$&
       \texttt{EST} & \texttt{STD}${}^*$
      \\
     \hline
                \multirow{2}{*}{$\!\!j^{+}_{-0.45}$} &  \multirow{2}{*}{.029} & \multirow{2}{*}{.2491}  &    $2^{17}$
                & .039  &       .0532   % CMC
                & .038  &     \textbf{.0523}      %    MR_3_EST
                & .038  &       .0530               %    MR_5_EST
                & .038  &       .0530               %    MR_7_EST
                & .038  &       .0529               %    MR_9_EST
           \\
           &   &   &      $2^{18}$
               &    .038  &       .0374    % CMC
                &   .038  &       .0373       %    MR_3_EST
                &   .038 &        \underline{.0373}              %    MR_5_EST
                &   .038   &      .0373               %    MR_7_EST
                &   .038   &     \textbf{.0373}             %    MR_9_EST
           \\ \hline
            \multirow{2}{*}{$\!\!j^{+}_{-0.50}$}  & \multirow{2}{*}{.070} & \multirow{2}{*}{.3747}  &   $2^{17}$
                &   .080   &      \textbf{.0745}  % CMC
                &   .080  &       .0746     %    MR_3_EST
                &   .080&         .0747             %    MR_5_EST
                &   .080    &     .0747               %    MR_7_EST
                &   .080  &       \underline{.0746}               %    MR_9_EST
           \\
           &   &     &   $2^{18}$
                  & .080&         .0530 % CMC
                &   .080   &      .0528     %    MR_3_EST
                &   .080 &        .0528          %    MR_5_EST
                &   .080  &       \underline{.0528}          %    MR_7_EST
                &   .080  &       .0529       %    MR_9_EST
           \\ \hline

            \multirow{2}{*}{$\!\!j^{+}_{-0.55}$} &  \multirow{2}{*}{.128} & \multirow{2}{*}{.4921}  &    $2^{17}$
             &      .138    &  .0952 % CMC
                &   .137 &    \textbf{.0946}     %    MR_3_EST
                &   .138  &    .0947            %    MR_5_EST
                &   .137  &    .0946            %    MR_7_EST
                &   .137  &    \underline{.0946}          %    MR_9_EST
           \\
           &   &      &   $2^{18}$
            &       .138   &    .0673  % CMC
                &   .138  &     .0670        %    MR_3_EST
                &   .138 &      .0670                %    MR_5_EST
                &   .137 &     \underline{.0669}        %    MR_7_EST
                &   .137  &   \textbf{.0669}      %    MR_9_EST
           \\ \hline
\end{tabular}
}\vspace{0.4em}
\caption*{(a) Results for $m\in\{3, 5, 7 , 9\}$.}
\end{subtable}

\vspace{0.6em}

\begin{subtable}{\textwidth}
\centering
{
\fontsize{8pt}{9.2pt}\selectfont
\setlength{\tabcolsep}{1.9pt}
\begin{tabular}{|C{0.85cm}|p{0.75cm}p{0.7cm}|C{0.5cm}|p{0.7cm}p{0.75cm}|p{0.7cm}p{0.75cm}|p{0.7cm}p{0.75cm}|p{0.7cm}p{0.75cm}|p{.7cm}p{0.8cm}|}\hline
         \multicolumn{1}{|c|}{$f$} &
         \multicolumn{2}{c|}{ \small{$\hat{Y}_{4500}^{\rm obs}$}}& $R$&
         \multicolumn{2}{c|}{\small{$\hat{Y}^{\F, \rm CMC}_R$}}&
         \multicolumn{2}{c|}{\small{$\hat{Y}^{\F, \textrm{M}\texttt{High}3}_{R,m=2^3}$}}&
         \multicolumn{2}{c|}{\small{$\hat{Y}^{\F, \textrm{M}\texttt{Rand}3}_{R,m=2^3}$}}&
         \multicolumn{2}{c|}{\small{$\hat{Y}^{\F, \textrm{M}\texttt{High}3}_{R,m=4^3}$}}&
         \multicolumn{2}{c|}{\small{$\hat{Y}^{\F, \textrm{M}\texttt{Rand}3}_{R,m=4^3}$}}
  \\
     \hline
    & \texttt{EST} & \texttt{STD}${}^*$ &   &
       \texttt{EST} & \texttt{STD}${}^*$ &
       \texttt{EST} & \texttt{STD}${}^*$ &
       \texttt{EST} &\texttt{STD}${}^*$&
       \texttt{EST} & \texttt{STD}${}^*$&
       \texttt{EST} & \texttt{STD}${}^*$
      \\
     \hline
                \multirow{2}{*}{$\!\!j^{+}_{-0.45}$} &  \multirow{2}{*}{.029} & \multirow{2}{*}{.2491}  &    $2^{17}$
                & .039  &       .0532   % CMC
                & .039  &     \underline{.0531}   %    DIM3_best3_2_EST
                & .038  &       .0523               %    DIM3_rand3_2_EST
                & .039  &       .0531               %    DIM3_best3_4_EST
                & .038  &       .0530               %    DIM3_rand3_4_EST
           \\
           &   &   &      $2^{18}$
               &    .038  &       .0374    % CMC
                &   .038 &        .0374  %    DIM3_best3_2_EST
                &   .039    &     .0375            %    DIM3_rand3_2_EST
                &   .038    &     .0373              %    DIM3_best3_4_EST
                &   .038  &       .0375            %    DIM3_rand3_4_EST
           \\ \hline
            \multirow{2}{*}{$\!\!j^{+}_{-0.50}$}  & \multirow{2}{*}{.070} & \multirow{2}{*}{.3747}  &   $2^{17}$
                &   .080   &      \textbf{.0745}  % CMC
                &   .080&         .0749 %    DIM3_best3_2_EST
                &   .080  &       .0749               %    DIM3_rand3_2_EST
                &   .080   &      .0748              %    DIM3_best3_4_EST
                &   .080  &       .0749              %    DIM3_rand3_4_EST
           \\
           &   &     &   $2^{18}$
                  & .080&         .0530 % CMC
                &   .080&         \textbf{.0528}      %    DIM3_best3_2_EST
                &   .080 &        .0529      %    DIM3_rand3_2_EST
                &   .080  &       .0529      %    DIM3_best3_4_EST
                &   .080  &       .0529               %    DIM3_rand3_4_EST
           \\ \hline

            \multirow{2}{*}{$\!\!j^{+}_{-0.55}$} &  \multirow{2}{*}{.128} & \multirow{2}{*}{.4921}  &    $2^{17}$
             &      .138    &  .0952 % CMC
                &   .138 &     .0951  %    DIM3_best3_2_EST
                &   .137   &   .0950             %    DIM3_rand3_2_EST
                &   .137  &    .0950           %    DIM3_best3_4_EST
                &   .138  &    .0950              %    DIM3_rand3_4_EST
           \\
           &   &      &   $2^{18}$
            &       .138   &    .0673  % CMC
                &   .137 &      .0672  %    DIM3_best3_2_EST
                &   .137   &    .0672             %    DIM3_rand3_2_EST
                &   .137   &    .0671              %    DIM3_best3_4_EST
                &   .137     &  .0671        %    DIM3_rand3_4_EST
           \\ \hline
\end{tabular}
}\vspace{0.4em}
\caption*{(b) Results for $m\in\{2^3, 4^3\}$.}
\end{subtable}
\end{table*}

\section{Training time}.\label{sec:app_training_time}
Training (on a single GeForce RTX 2080) demanding Example \examplelarged{} (30D) with CNF's parameter \texttt{dims}=128 (see A.3) takes $\sim$3h (\texttt{dims}=32 takes $\sim$1h which gives slightly worse performance -- not reported). Training Examples \examplefirst{} and \examplesecond{} takes $\sim$1h. Note that in case of estimating many functions $\Exp f_i(X)$, the flow model must be trained only once.

\section{More details on choosing strata}\label{sec:app_choosing_strata}
\begin{minipage}{1.0\textwidth}
\begin{figure}[H]
  \begin{center}
   \begin{tabular}{cccc}
  \includegraphics[width=0.31\linewidth]{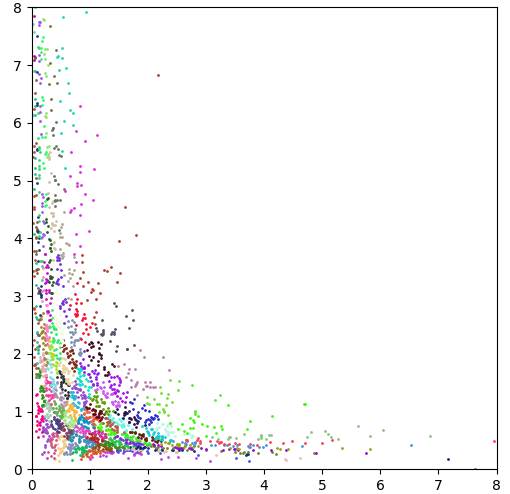}
  &\includegraphics[width=0.31\linewidth]{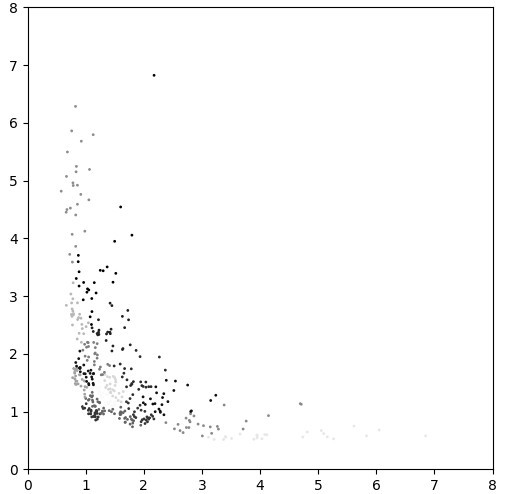}
  &
  \includegraphics[width=0.31\linewidth]{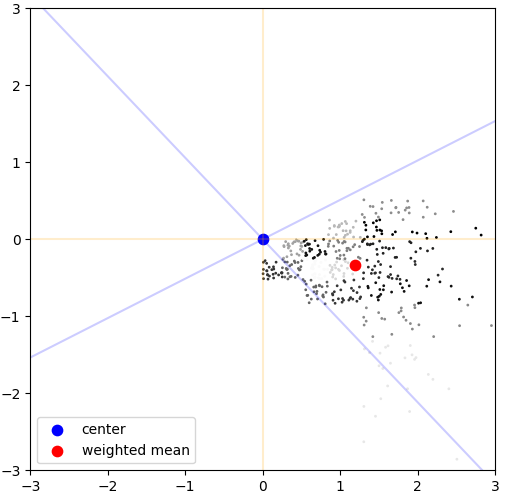}
  \\
  (a) & (b) & (c)
  \end{tabular}
  \end{center}
 \caption{Choosing strata in Example \examplesecond{}:  (a)  $R=2^{13}$ points from a base bivariate normal distribution sampled in 100 strata, transformed via flow model; (b) standard deviations (weights, the darker the larger)
 of a function $h_1$ of points from given strata (all points in the same strata are assign the same weight); (c) points from (b) presented in base distribution, together with it weighted mean (\textcolor{red}{red} ball),
 rotation indicated by \textcolor{blue}{blue} axis (rotated by $29.33^{\circ}$).
  }  \label{fig:Example2d_norm_flow_100str}
\end{figure}
\end{minipage}

\vspace{1cm}
\begin{minipage}{1.0\textwidth}
 \begin{figure}[H]
  \begin{center}
   \begin{tabular}{cccc}
  \includegraphics[width=0.32\linewidth]{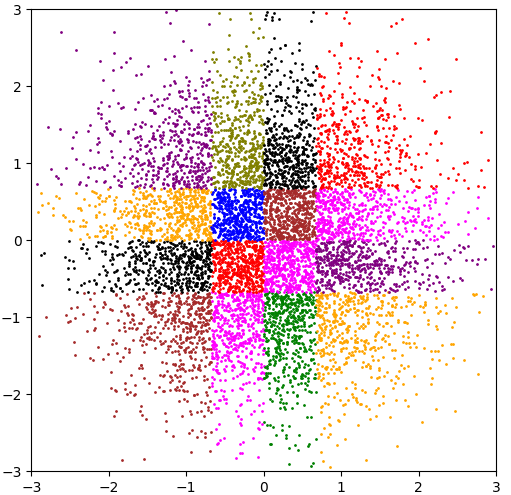}
  &
    \includegraphics[width=0.32\linewidth]{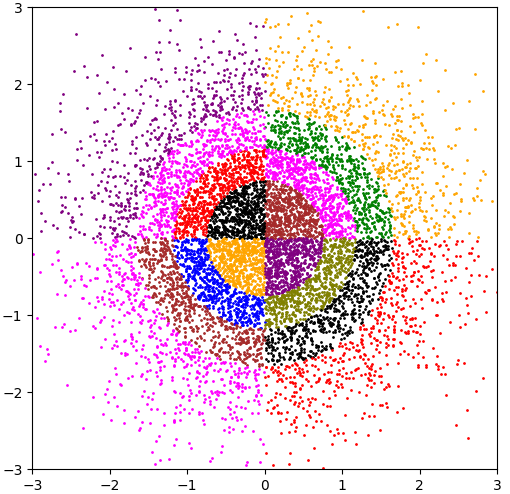}
    &
    \includegraphics[width=0.32\linewidth]{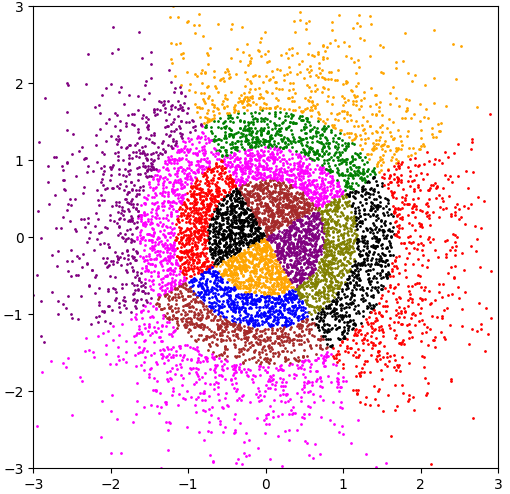}
    \\
   \includegraphics[width=0.32\linewidth]{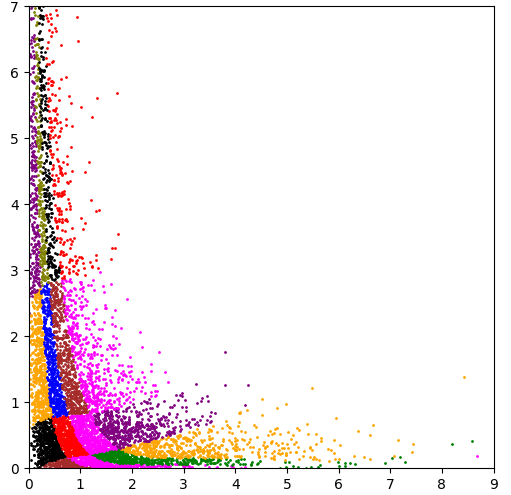}
  &\includegraphics[width=0.32\linewidth]{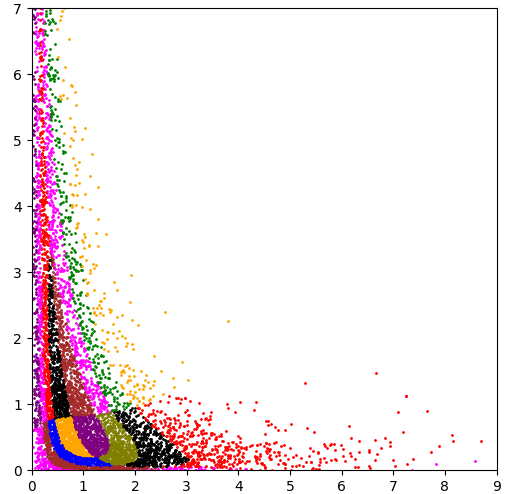}
  &\includegraphics[width=0.32\linewidth]{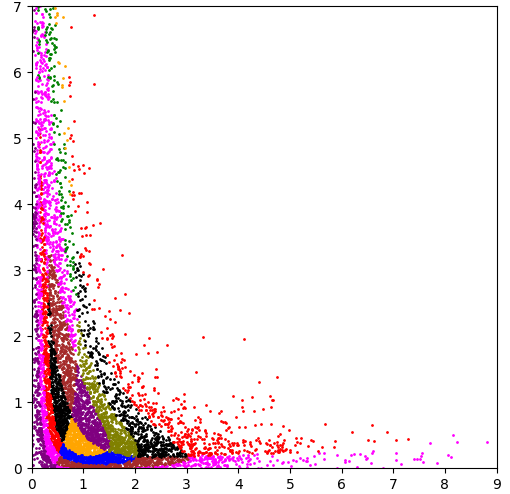}
  \end{tabular}
  \end{center}
 \caption{Example \examplesecond{}: $R=2^{13}$ points and $m=16$ strata.
 Cartesian ($1^{\rm st}$ column) and spherical ($2^{\rm nd}$ and $3^{\rm rd}$ columns,
 in the latter strata are rotated by $29.33^{\circ}$). First two columns were depicted in Fig. \ref{fig:Example2d_norm_flow_16str}.
  }  \label{fig:Example2d_norm_flow_16str_v2}
\end{figure}
\end{minipage}
\smallskip\par
Note that even if we fix number of strata, we decide to have them equally probable, we may choose them
in a variety of ways. E.g., in case of spherical method we may ``rotate'' strata
(in all previous examples spherical strata started at angles 0, i.e., unrotated).
As mentioned, a general design for choosing strata is left for a future work, here we shortly describe some reasonable approach
for a 2-dimensional case. We will continue Example~\examplesecond{} with estimation of $\Exp h_1(\mathbf{X})$. The main idea is that the angle will depend on a function we want to estimate,
and we will try to accommodate many points in a small number of strata.
Simple method could be (for 2D) as follows.
 We start with some relatively large number of cartesian strata, say $100$ (i.e., 10 in each dimension).
For each stratum in a latent space we compute the standard deviation of estimated function.
Then, each point is assigned a weight -- the computed standard deviation.
In Fig. \ref{fig:Example2d_norm_flow_100str}: left -- points sampled from 100 strata are presented;
center -- the darker the point the larger the standard deviation (most points invisible -- std from corresponding
strata either 0 or close to 0); right --  the points in base (2D normal) distribution with their weights depicted.
Afterwards we compute the weighted (with weights being the standard deviation) mean of the point (\textcolor{red}{red} ball in Fig. \ref{fig:Example2d_norm_flow_100str}). The ball is at angle $-15.67^{\circ}$. Finally, we wish to have this ball in the middle of the strata -- we rotate it $-15.67^{\circ} +45^{\circ} = 29.33^{\circ}$. This way most of the points will be in
this case in one stratum (disregarding radii). The rotated strata are depicted in the right-most column  of Fig.  \ref{fig:Example2d_norm_flow_16str_v2}.

Initial simulations show that it improves the accuracy: with $R=2^{15}$ the accuracy
of a stratified estimator with optimal allocation yields accuracy 1.2271 in case of unrotated strata and  1.6056 in case of rotated strata.

We want to \textsl{emphasize} that this is just some heuristic approach,
but it proves that choosing a strata based on a function we wish to estimate may results in better performance.
A more in-depth study is left for a future work.

\end{document}